\documentclass{article}

\usepackage{microtype}
\usepackage{graphicx}
\usepackage{booktabs} %
\usepackage{hyperref}

\usepackage{tikz}

\newif\ifdraft
\drafttrue
\usepackage{rotating}
\usepackage{float}
\usepackage{setspace}
\usepackage[bottom]{footmisc}
\usepackage{amsthm}
\usepackage{amssymb}
\usepackage{amsmath}
\usepackage{paralist} %
\usepackage{pifont}
\usepackage{mathtools}
\usepackage{multirow}
\usepackage{booktabs}
\usepackage{multirow}
\usepackage{color}
\usepackage{style_ltx}
\usepackage{xspace}
\usepackage{algorithm}
\usepackage{algorithmic}
\usepackage{fullpage}
\usepackage{array}
\usepackage{wrapfig}
\usepackage[labelformat=simple]{subcaption}

\usepackage[utf8]{inputenc} %
\usepackage{hyperref}       %
\usepackage{url}            %
\usepackage{amsfonts}       %
\usepackage{nicefrac}       %
\usepackage[inline]{enumitem}
\usepackage[compact]{titlesec}
\usepackage{natbib}
\usepackage{authblk}

\title{Robustness via Deep Low-Rank Representations}

\author[1,2]{Amartya Sanyal\thanks{amartya.sanyal@cs.ox.ac.uk}}
\author[1,2]{Varun Kanade}
\author[3]{Philip HS Torr}
\author[3]{Puneet K Dokania}
\affil[1]{Department of Computer Science, University of Oxford}
\affil[3]{Department of Engineering Science, University of Oxford}
\affil[2]{The Alan Turing Institute, London.}

\date{}

\newcommand{\fsgm}{FSGM\xspace}
\newcommand{\ifsgm}{Iter-FSGM\xspace}
\newcommand{\deepfool}{DeepFool\xspace}

\newcommand{\ill}{Iter-LL-FSGM\xspace}
\newcommand{\lr}{$\mathsf{LR}$-$\mathrm{layer}$\xspace}
\newcommand{\lrs}{$\mathsf{LR}$-$\mathrm{layers}$\xspace}

\newcommand{\prnk}[2]{\mathsf{\Pi}_{#1}^{\mathrm{rank}}(#2)}
\newcommand{\clip}[1]{\mathrm{clip}_{\vx, \epsilon}(#1)}
\renewcommand{\th}{{\it th}}
\definecolor{DarkGreen}{rgb}{0, 0.4, 0}

\graphicspath{{./figs/},{./images_adv/}}

\usetikzlibrary{arrows,decorations.pathmorphing,backgrounds,fit,positioning,shapes.symbols,chains}
\tikzset{ node distance = 1cm, auto,font=\footnotesize,
tensors/.style={circle, rounded corners, draw=black, fill=black!10, inner sep=0.5pt, text width=1cm, text badly centered, minimum height=1.2cm,, font=\bfseries\footnotesize\sffamily} ,
temp_tensors/.style={circle, rounded corners, dashed, draw=black, fill=black!5, inner sep=0.5pt, text width=1cm, text badly centered, minimum height=1.2cm,, font=\bfseries\footnotesize\sffamily} ,
parameters/.style={align=center, text width=2cm, font=\bfseries\footnotesize\sffamily}}

\newcolumntype{L}[1]{>{\raggedright\let\newline\\\arraybackslash\hspace{0pt}}m{#1}}
\newcolumntype{C}[1]{>{\centering\let\newline\\\arraybackslash\hspace{0pt}}m{#1}}
\newcolumntype{R}[1]{>{\raggedleft\let\newline\\\arraybackslash\hspace{0pt}}m{#1}}

\begin{document}
\maketitle

\begin{abstract}
We investigate the effect of the dimensionality of the representations learned in Deep
Neural Networks (DNNs) on their robustness to input perturbations, both adversarial and
random. To achieve low dimensionality of learned representations, we propose an easy-to-
use, end-to-end trainable, low-rank regularizer (LR) that can be applied to any
intermediate layer representation of a DNN. This regularizer forces the feature
representations to (mostly) lie in a low-dimensional linear subspace. We perform a wide
range of experiments that demonstrate that the LR indeed induces low rank on the
representations, while providing modest improvements to accuracy as an added benefit.
Furthermore, the learned features make the trained model significantly more robust to
input perturbations such as Gaussian and adversarial noise (even without adversarial
training). Lastly, the low-dimensionality means that the learned features are highly
compressible; thus discriminative features of the data can be stored using very
little memory. Our experiments indicate that models trained using the LR learn robust
classifiers by discovering subspaces that avoid non-robust features. Algorithmically, the LR
is scalable, generic, and straightforward to implement into existing deep learning
frameworks.
\end{abstract}

\section{Introduction}
\label{sec:introduction}
Dimensionality reduction methods are some of the oldest techniques in machine learning that extract a small number of factors from a dataset that explain the most of its variance; these factors contain most of the {\em discriminative} power useful in classification or regression tasks, and are known to increase {\em robustness}, i.e. these methods typically have a denoising effect. 
Perhaps, the most popular and widely used among them are PCA (see e.g.~\citep{PCA2002a}) and CCA~\citep{hotelling1935canonical}. %

In recent years, deep neural networks~(DNNs) have proved to be the state-of-the art models for a wide
range of tasks. An intriguing aspect has been their ability to
generate representations directly from raw data that are useful in several
tasks, including ones for which they were not specifically trained, usually known as ~\textit{representation
learning}~\citep{zeiler2014,Sermanet2014,donahue2013}.
Essentially, for most models trained in a supervised fashion, the vector of
activations in the penultimate layer is a \emph{learned} representation of the raw input. The remarkable success of DNNs is primarily attributed to the discriminative quality of this learned representation space. However, despite their impressive performance, DNNs are known to be brittle to input perturbations~\citep{szegedy2013intriguing,goodfellow2014explaining}. This raises concerns regarding the robustness of the factors captured by the learned representation space of DNNs. As mentioned earlier, the factors captured by dimensionality reduction techniques, while being discriminative, are robust to input perturbations. This motivates the thesis behind this work---if we enforce DNNs to learn representations that lie in a low-dimensional subspace (for the entire dataset), we would obtain more robust classifiers while preserving their discriminative power.

Ideally, to encourage learning low-dimensional representations, we would like to
insert a \emph{dimensionality reduction} ``module'' in DNNs and develop an end-to-end
training method that simultaneously does supervised training and dimensionality
reduction. At first glance, using SVD to project representations into lower dimensional subspace seems viable. However, this approach encounters challenges because of the large number of training examples, and, also due to the fact that the representations themselves change after every parameter update (discussed later in detail). A workaround could be to design architectures with bottlenecks similar to auto-encoders~\citep{Hinton2006}. The fact that most of the state-of-the-art networks do not have such bottlenecks limit their usability.

Our work provides the benefits of dimensionality reduction by inserting a
\emph{virtual layer} (not used at prediction time) and augmenting the
loss function to induce low-rank representations. Precisely, we propose a
low-rank Regularizor (LR) that (1) does not put any restriction on the network
architecture,%
\footnote{It puts no \emph{direct} restriction, though of course any extra regularizion will produce an inductive bias.}
(2) is end-to-end trainable, and (3) is efficient in that it allows
mini-batch training. LR explicitly {\em enforces} representations to lie in a
linear subspace with low {\em intrinsic} dimension and is guaranteed to provide
low-rank representations for the entire dataset even when trained 
using mini batches. As LR is a virtual layer, it can
be applied to any intermediate representations of DNNs. It is sensible to do so
as, DNNs, actually learn hierarchical representations, one after another.
These intermediate representations are known to capture interesting properties
such as semantic meaning or concepts needed to improve the discriminative power
of the penultimate layer representation. %

Apart from successfully reducing the dimensionality of learned representations,
DNNs trained with LR turn out to be significantly more robust to input
perturbations, both adversarial and random, while providing modest improvements
over the unperturbed test accuracy. This is of particular
interest as it suggests that adding well-thought priors over factors
influencing the representation space (e.g. low-rank prior over 
representations) might further improve the robustness of DNNs,
before {\em actually} reaching the limit beyond which robustness comes at a
cost, be it computational~\citep{goodfellow2014explaining,madry2018towards},
statistical~\citep{schmidt2018adversarially} or a loss in
accuracy~\citep{tsipras2018robustness}.

Lastly, because of the low-dimensionality, we are able to compress
representations by a significant factor without losing its discriminative
power. Thus, discriminative features of the data can be stored using very
little memory. For example, we show in one of our experiments that, even with a
$5$-dimensional embedding~(400x compression), the model with LR looses only $6\%$
in accuracy.

\section{Deep Low-Rank Representations}
\label{sec:low_rank}
\label{sec:what-low-rank}
Consider $f: \mathbb{R}^p \mapsto \mathbb{R}^k$ to be a feed-forward
multilayer NN that maps $p$ dimensional input $\vx$ to a $k$
dimensional output $\vy$. We can decompose this into two
sub-networks, one consisting of the layers before the $\ell^{\it th}$
layer and one after i.e.  $f(\vx) = f^{+}_\ell\br{f^{-}_\ell(\vx;
  \phi) ; \theta}$, where $f^{-}_\ell (.;\phi)$, parameterized by
$\phi$, represents the part of the network up to layer
$\ell$ and, $f^{+}_\ell(.;\theta)$ represents the part of the
network thereafter. With this notation, the $m$ dimensional
representation (or the activations) of any layer $\ell$ can simply be
written as $\va = f^{-}_\ell(\vx; \phi) \in \mathbb{R}^m$. In
what follows, we first formalize the low-rank representation problem,
then provide insights on its difficulty, and finally propose our
approach to solve it approximately and efficiently.

\noindent\textbf{Problem Formulation:} Let $\vec{X} = \{{\bf x}_i\}_{i=1}^n$
and $\vec{Y} = \{{\bf y}_i\}_{i=1}^n$ be the set of inputs and outputs
of a given training dataset. By slight abuse of notation, we
define $\vec{A}_\ell = f^{-}_\ell(\vec{X}; \phi) =\bs{\vec{a}_1,\cdots,\vec{a}_n}^\top\in \mathbb{R}^{n
  \times m}$ to be the activation matrix of the entire dataset, so
that $\vec{a}_i$ is the activation vector of the $i$-th sample. Note
that for most practical purposes $n\gg m$. In this setting,
the problem of learning low-rank representations can be
formulated as a constrained optimization problem as follows:
\begin{align}
  \label{eq:opt_prob}
	\min_{\theta, \phi}\mathcal{L}(\vec{X}, \vec{Y}; \theta, 
  \phi),~\text{s.t.}~~&\rank{\vec{A}_\ell} = r,\end{align} 
where $\mathcal{L}(.)$ is the loss function and $r < m$ is the
desired rank of the representations at layer $\ell$. The rank $r$ is a
hyperparameter (though empirically not a sensitive one as observed
in our experiments). Throughout this section, we consider imposing
low-rank constraints over only one intermediate layer, however, the
methods provided here can be easily extended to any number of
layers. Note that both the loss and the constraint set of the above
objective function are non-convex. One approach to optimize this would
be to perform alternate minimization, first over the loss (gradient
descent) and then projecting onto the non-convex set to satisfy the
rank constraint.

Since $n \gg m$, ensuring $\rank{\vec{A}_\ell} = r$
would  be practically infeasible as it would require performing SVD at
every iteration (at cost $\mathcal{O}(n^2m)$). A feasible, but incorrect, approach
would be do this on mini-batches, instead of the entire dataset. However,
projecting each mini-batch onto the space of rank $r$ matrices does
not guarantee that the activation matrix of the entire dataset will be of
rank $r$, as each of these mini-batches can lie in very different subspaces.
Computational issues aside, another crucial problem stems
from the fact that the activation matrix $\vec{A}_\ell = f^{-}_\ell(.;
\phi)$  is itself parameterized by $\phi$ and thus $\phi$ needs to be
updated in a way such that the generated $\vec{A}_\ell$ is low
rank and it is not immediately clear how to use the low-rank projection
of $\vec{A}_\ell$ to achieve this.
One might suggest to first fully train the network and then obtain low
rank projections of the activations. However, as our experiments show,
this procedure does not provide the two main benefits: compression and robustness. 

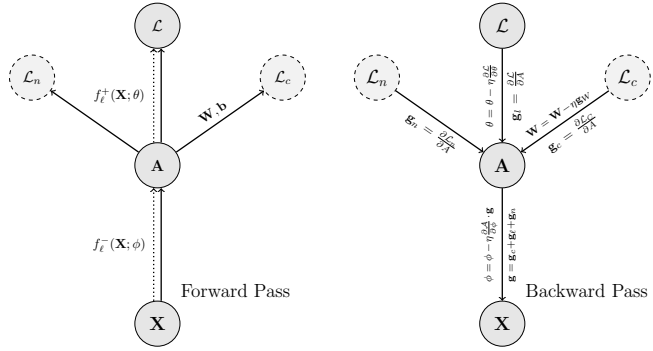
\begin{figure}\centering%
\scalebox{0.5}{\begin{tikzpicture}
\node [tensors] (data) { \Large$\mathbf{X}$};
\node [tensors, above=3cm of data] (activations) {
  \large$\mathbf{A}$};
\node [rectangle, rounded corners,  above right=0.8cm of data,
yshift=-0.4cm, xshift=-0.5cm]
(fwd_lbl) { \Large Forward Pass};
\node [tensors, above=2.5cm  of activations] (output) { \large$\mathcal{L}$};
\node [temp_tensors,  above right =3.5cm of activations, yshift=-1cm] (recons) { \large$\mathcal{L}_c$};
\node [temp_tensors, above left =3.5cm of activations, yshift=-1cm] (normed) { \large$\mathcal{L}_n$};
\draw[->, thick, color=black, line width=1pt] ([xshift=1 * 0.1 cm]data.north)--   node[midway, parameters] {\normalsize$f_{\ell}^{-}(\mathbf{X};\phi)$}  ([xshift=1 * 0.1 cm]activations.south);
\draw[->, dotted, color=black, line width=1pt] ([xshift=-1 * 0.1 cm]data.north)--  ([xshift=-1 * 0.1 cm]activations.south);

\draw[->, thick, color=black, line width=1pt] ([xshift=1 * 0.1
cm]activations.north)  --  node[parameters, align=center, midway] {\normalsize$f_{\ell}^{+}(\mathbf{X};\theta)$}  ([xshift=1 * 0.1 cm]output.south);
\draw[->, dotted, color=black, line width=1pt] ([xshift=-1 * 0.1 cm]activations.north)  -- ([xshift=-1 * 0.1 cm]output.south);

\draw[->, thick, color=black, line width=1pt] (activations) -- node[parameters, midway,above, sloped] {\normalsize$\mathbf{W},\mathbf{b}$} (recons);
\draw[->, thick, color=black, line width=1pt] (activations) -- (normed);
\end{tikzpicture}}\hspace{15pt}
\scalebox{0.5}{\begin{tikzpicture}
\node [tensors] (data) { \Large$\mathbf{X}$};
\node [tensors, above=3cm of data] (activations) { \Large$\mathbf{A}$};
\node [tensors, above=2.5cm  of activations] (output) { \Large$\mathcal{L}$};
\node [temp_tensors, above right=3.5cm of activations, yshift=-1cm] (recons) { \Large$\mathcal{L}_c$};
\node [temp_tensors, above left=3.5cm of activations, yshift=-1cm] (normed) { \Large$\mathcal{L}_n$};
\node [ rounded corners, dashed, above right=0.8cm of data,
xshift=-0.5cm,
yshift=-0.4cm]
(fwd_lbl) {\Large  Backward Pass};%
\draw[->, thick, color=black, line width=1pt] ([yshift=0 * 0.1
cm]activations.south) --   node[midway, parameters, below, sloped, rotate=180] {$\mathbf{g}
  =\mathbf{g}_c + \mathbf{g}_\ell +\mathbf{g} _n$}  node[parameters,
above, sloped, rotate=180]{$\phi=\phi-\eta\frac{\partial \mathcal{A}}{\partial \phi}\cdot \mathbf{g}$} ([yshift=0 * 0.1 cm]data.north);

\draw[->, thick, color=black, line width=1pt] (output)  --
node[parameters, below, sloped, rotate=180]
{\normalsize$\mathbf{g}_l=\frac{\partial \mathcal{L}}{\partial A}$}
node[parameters, above, sloped, rotate=180, allow upside down]{$\theta=\theta-\eta\frac{\partial \mathcal{L}}{\partial \theta}$}(activations) ;

\draw[->, thick, color=black, line width=1pt] (recons)  --
node[parameters, midway,above, sloped,
align=center]{$\mathbf{W}=\mathbf{W}-\eta\mathbf{g}_W$ }
node[parameters, midway,below, sloped, align=center, sloped]{\normalsize $\mathbf{g}_c = \frac{\partial \mathcal{L}_C}{\partial A}$ }  (activations);

\draw[->, thick, color=black, line width=1pt] (normed) --
node[parameters, midway,sloped,below,align=center]{\normalsize $\mathbf{g}_n = \frac{\partial \mathcal{L}_n}{\partial A}$ }  (activations);

\end{tikzpicture} }
\caption{{\bf The LR layer}. Left figure shows the {\em forward pass}, {\em solid
  edges} show the flow of data during training, {\em dashed edges}
  show the flow of data during inference, and {\em dashed nodes} indicate the
  {\em virtual layer}. Right figure shows the {\em backward pass}.}
  \label{fig:LRtikz}
\end{figure}

\textbf{Low-Rank Regularizer:} We now describe our regularizer that
encourages learning low-rank activations, and, if optimized properly, guarantees that the rank of the activation matrix (of any
size) will be  bounded by $r$. 
We do this by introducing an auxiliary %
parameter~$\vec{W}\in\reals^{m\times m}$, augmenting the loss function
and shifting the low-rank constraint from the activation
matrix $\vec{A}_\ell$~(as in \eqref{eq:opt_prob}) to this auxiliary parameter
$\vec{W}$. Switching the rank constraint to $\vec{W}$ has two
advantages: The rank constraint is put (a) on a matrix that is independent
of the batch/dataset size, and (b) on a parameter as opposed to a
data-dependant intermediate tensor~(like activations), and can thus be updated
directly at each iteration. Combining
these ideas, our
final augmented objective function, with the regularizer, is as follows:
\begin{align}
  \label{eq:aug_opt_2}
  &\min_{\theta, \phi, \vec{W}, \vb} \mathcal{L}(\vec{X}, \vec{Y}; \theta, \phi) + \mathcal{L}_c(\vec{A}_\ell; \vec{W},\vb) + \mathcal{L}_n(\vec{A}_\ell) \\
  &\text{s.t.,} \vec{W}\in \mathbb{R}^{m\times m}, \rank{\vec{W}} = r,~ \vb\in \mathbb{R}^m,\vec{A}=f^{-}_\ell(\vec{X}; \phi),\nonumber
\end{align}
where,
\begin{align}
  \label{eq:aug_opt_3}
  \mathcal{L}_c(\vec{A}; \vec{W}, \vb) &= \frac{1}{n}\sum_{i=1}^{n}\norm{\vec{W}^\top(\va_i+\vb) - (\va_i+\vb)}_2^2, \nonumber \\
   \text{and} \; \; \mathcal{L}_n(\vec{A}) &= \frac{1}{n}\sum_{i=1}^n \Big|1 - \norm{\va_i}\Big|.\nonumber
\end{align}

The intuition behind our approach is that we add a virtual (doesn't modify the
main network) branch at layer $\ell$, the goal of which is to learn a low-rank
identity projection for $\vec{A}_\ell$. It is well known that for a low-rank
matrix, there exists a low-rank projection that projects the matrix onto
itself\footnote{It is the PCA problem}. Thus, if such a low-rank identity map
does not exist for $\vec{A}_\ell$ i.e. $\vec{A}_\ell$ is not low-rank, then the
goal of our regularizer is to jointly penalize the activations to make them
low-rank and learn that low-rank identity map.

\noindent Specifically, minimizing the projection loss $\mathcal{L}_c$
ensures that the affine low-rank mappings~($\vec{A}\vec{W}$) of the activations
are close to the original ones i.e. $\vec{AW} \approx \vec{A}$. As $\vec{W}$ is
low-rank, $\vec{A}\vec{W}$ is also low-rank and thus implicitly~(due to
$\vec{A}\vec{W}\approx\vec{A}$) it forces the original activations $\vec{A}$ to
be low-rank. The bias $\vb$ allows for the activations to be translated before
projection.%
\footnote{We use the term \emph{projection} loosely as we do not strictly
constrain $\vec{W}$ to be a projection matrix.}
However note that setting $\vec{A}$ and $\vb$ close to zero trivially minimizes $\mathcal{L}_c$, especially when the activation dimension is large. We observed this to happen in practice as it is easier for the network to learn $\phi$ such that the activations and the bias are very small in order to minimize $\mathcal{L}_c$. To prevent this, we use $\mathcal{L}_n$ that acts as a norm constraint on the activation vector to keep the activations sufficiently large.
Lastly, as the rank constraint is now over $\vec{W}$ and $\vec{W}$ is a {\em global} parameter independent of the
dimension~$n$~(i.e. size of minibatch/dataset) we can use
mini-batches to optimize~\eqref{eq:aug_opt_2}. Since $\rank{\vec{AW}}
\leq r$ for any $\vec{A}$, optimizing over mini-batches 
still ensures that the entire activation matrix is
low-rank. Intuitively, this is due to the fact that the basis vectors
of the low-rank affine subspace are now captured by the low-rank parameter
$\vec{W}$. Thus, as long as $\cL_c$ is minimized for all the mini-batches, $\vec{A} \approx \vec{A}\vec{W}$ holds for the entire dataset, leading to the low-dimensional support. 
\begin{algorithm}[h!]  \centering
   \caption{Low-Rank (LR) Regularizer}
   \label{alg:lr_layer_main}
   \begin{algorithmic}[1]
		\INPUT Activation Matrix $\vec{A}_l$, gradient input ${\bf g}_l$
     \STATE{${\bf Z}  \gets (\vec{A}_l+\vb)^\top \vec{W}$}\footnotemark \COMMENT {forward propagation towards the virtual LR layer}
     \STATE {$\mathcal{L}_c \gets \frac{1}{b} \norm{{\bf Z} - (\vec{A}_l + \vb)}_2^2$}
     \COMMENT{the reconstruction loss}
     \STATE {$\mathcal{L}_n \gets \frac{1}{b}\sum_{i=1}^{b}\big\vert\mathbf{1} - \norm{\va_i} \big\vert$}
     \COMMENT{norm constraint loss}
     \STATE {${\bf g}_W \gets \frac{\partial \mathcal{L}_c}{\partial \vec{W}},\enskip {\bf g}\gets {\bf g}_l + \frac{1}{b}\sum_{i=1}^{b}\frac{\partial (\mathcal{L}_c + \mathcal{L}_n)}{\partial \va_i}$}
	  \STATE{$\vec{W} \gets \vec{W} - \lambda {\bf g}_W$}%
	  \STATE{$\vec{W} \gets \prnk{k}{\vec{W}}$ \label{alg:hard_thresh_step}}\COMMENT{hard thresholds the rank of $\vec{W}$}
     \OUTPUT ${\bf g}$
     \COMMENT{the gradient to be passed to the layer before}
\end{algorithmic}
\end{algorithm} \footnotetext{$\vec{A} + \vb$ is computed by adding $\vb$ to
every row in $\vec{A}$}

\noindent\textbf{Implementing the Low-Rank Regularizer}:
Algorithm~\ref{alg:lr_layer_main} (further details in Appendix~\ref{sec:alg-lr})
describes the forward and the backward operations of the low-rank virtual layer
for a mini-batch of size $b$. We present a flow diagram for the same in
Figure~\ref{fig:LRtikz}. This layer is virtual in the sense that it only
includes the parameters $\vec{W}$ and $\vb$ that are not used  in the NN model
itself to make predictions, but nonetheless the corresponding loss term
$\mathcal{L}_c$ does affect the model parameters through gradient updates. 
Algorithm~\ref{alg:lr_layer_main} alternately minimizes the augmented loss
function~(Line 1-5) and projects the auxiliary parameter $\vec{W}$ to the space
of low-rank matrices. The algorithm is reminiscent of the Singular Value
Projection~(SVP) Algorithm~\citep{jain2010guaranteed}. However, to make the
algorithm practical for high dimensional representation spaces, we use
ensembled Nystr\"om SVD~\citep{williams2001using, halko2011finding,
kumar2009ensemble} for the projection in Step 6~(details in
Appendix~\ref{sec:alg_details_app}).

\section{Experiments}
\label{sec:experiments}
We perform a wide range of experiments to show the effectiveness of imposing low-rank constraints on the representations of a dataset using our proposed LR. Briefly,
\begin{compactitem}
\item We show that LR indeed reduces the rank of the representation space and improves robustness for both, {\em adversarial} and {\em random} noise input perturbations, while providing modest improvements over accuracy.
\item We show results on both, white-box and black-box adversarial attacks.
\item In addition, we compare LR with various different approaches such as Stable Rank Normalization~\citep{sanyal2020stable}, Pruning~\citep{lee2018snip}, and Spectral Normalization~\citep{miyato2018spectral} where the parameter space, instead of the activations, is being either compressed or encouraged to have low {\em effective} rank.
\item Lastly, we provide analyses to show that the representations learned using LR is extremely discriminative and, because they lie in a linear space with low {\em intrinsic} dimension, can be compressed significantly. 
\end{compactitem}

\paragraph{Architectures and Datasets}
We use the standard ResNet~\citep{HZRS:2016} architecture with four residual blocks. To
capture the effect of network depth, we use ResNet-50~(R50) and
ResNet-18~(R18). Since LR can be applied to any representation layer in the
network, we investigate the following two configurations:
\begin{compactitem}
\item {\bf 1-LR}, where the LR layer is located just before the last fully-connected (FC) layer that contains 512 and 2048 units in ResNet-18 and ResNet-50, respectively.
\item {\bf 2-LR}, where there are two LR layers, the first one positioned before the fourth ResNet block with $16,384$ incoming units, and the second one just before the FC layer as in ResNet 1-LR.
\item {\bf N-LR}, without any LR layer (standard setting).
\end{compactitem}
Other model formulations that we consider such as {\em bottle-LR} and
{\em hybrid max-margin} models will be discussed when introduced. 
We use CIFAR10 and CIFAR100 datasets and show results using the coarse labels (20 classes) of CIFAR100. Further experimental details and additional experiments on other models~(VGG19) and datasets~(CIFAR100
with fine labels and SVHN) are reported in
Appendix~\ref{sec:experimental-details}. Experimentally we observed
that the target rank is not a sensitive hyper-parameter, as the
training enforces a much lower rank than what it is set to. For our
experiments, we set a target rank of $100$ for the layer before the
last FC layer and $500$ for the layer before the fourth ResNet block. 

\paragraph{Effective Rank} Before we discuss our primary findings,
here we empirically show the effect of LR on the effective rank of
activations. We use the standard  \emph{variance ratio}, defined as
$\sum_{i=1}^r\sigma_i^2/\sum_{i=1}^p\sigma_i^2$, where $\sigma_i$'s
are the ordered singular values of the given activation matrix $A$,
$p$ is the rank of the matrix, and $r\leq p$. Given $r$, a
higher value of  variance ratio indicates that a larger fraction of
the total variance in the data is captured in the $r$ dimensional
subspace.

Fig.~\ref{fig:var_1} shows the variance ratio for the activations
before the last FC layer. Note that even for NLR, the effective rank
is as low as 10. Similar low-rank structure was also observed empirically by~\citet{Oyallon_2017_CVPR}. 
However, the LR-models have almost
negligible variance leakage.

Fig.~\ref{fig:var_2} shows the variance ratio for the
activations before the 4\th ~ResNet block. The activation vector is $16,384$-dimensional and the use of the Nystr\"om method ensures computational feasibility. Note, ResNet 2-LR is the
only model that has an LR-layer in that position and these figures show that \emph{2-LR is
the only model that shows a (reasonably) low-rank structure on that
layer}. More experiments provided in the Appendix~
(Fig~\ref{fig:var_ratio_plots_vgg} and Figs.~\ref{fig:var_3},
~\ref{fig:var_4}). 

\begin{figure}[t]
  \begin{subfigure}[t]{0.45\linewidth}
    \centering
    \def\svgwidth{0.95\columnwidth}
    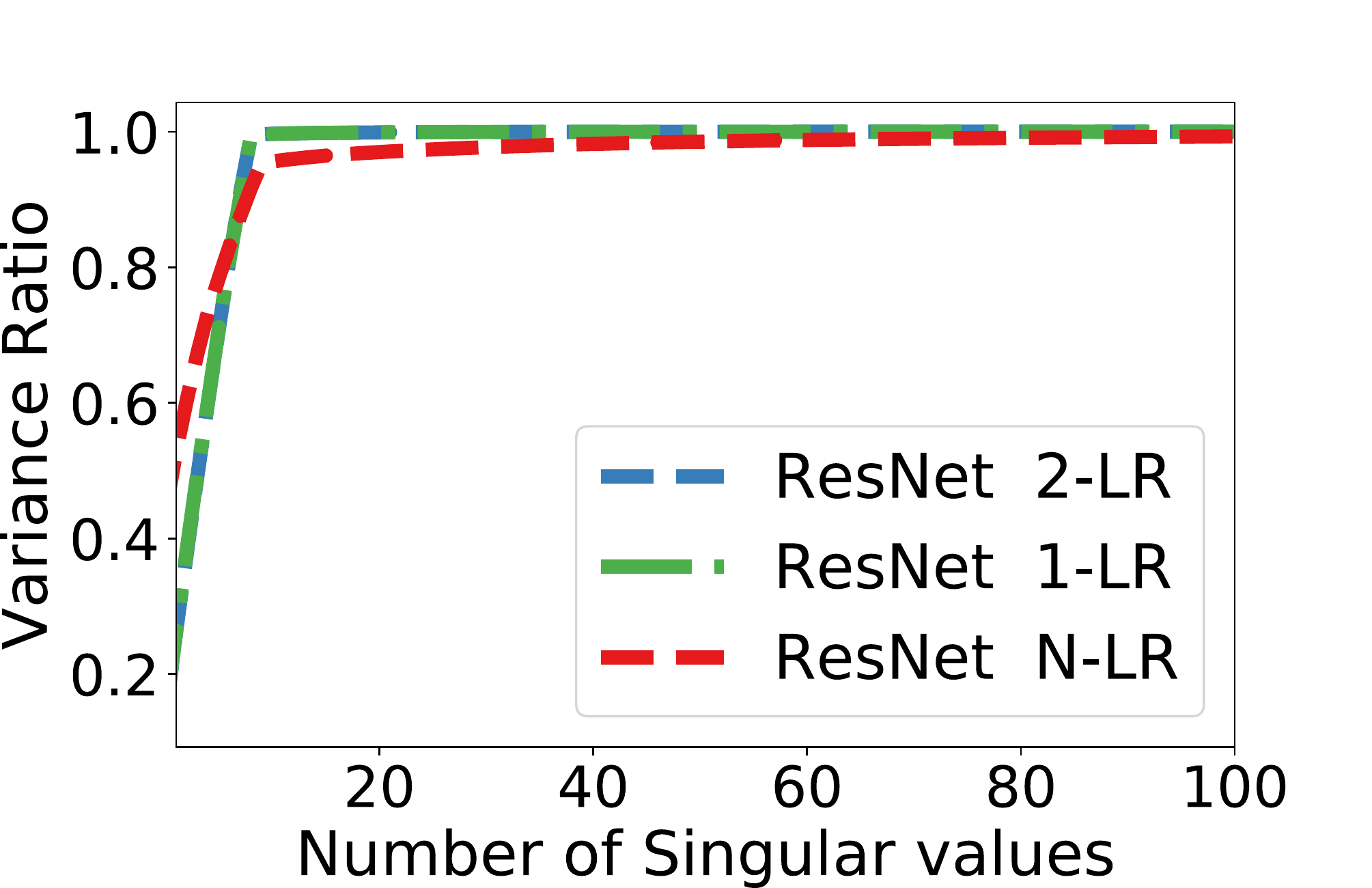
    \caption{CIFAR10: Activations after last ResNet block.}
    \label{fig:var_1}
  \end{subfigure}\hfill
  \begin{subfigure}[t]{0.45\linewidth}
    \centering
    \def\svgwidth{0.95\columnwidth}
    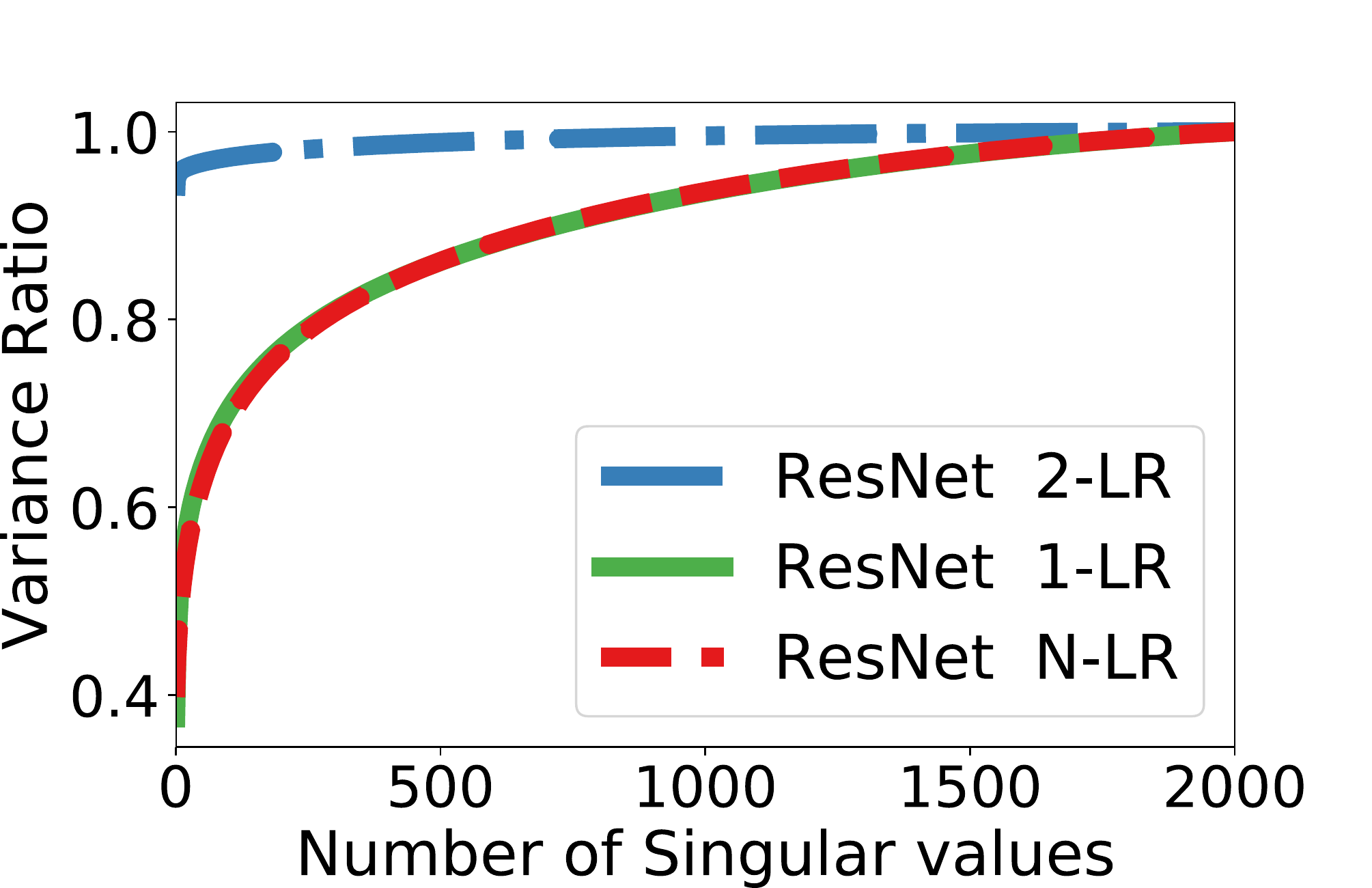
    \caption{CIFAR10: Activations before last ResNet block}
    \label{fig:var_2}
  \end{subfigure}
  \caption{Variance Ratio  captured by varying number of Singular Values}
  \label{fig:var_ratio_plots}\vspace{-1em}
\end{figure}

\begin{table*}[!htb]\centering\footnotesize
\begin{tabular}{llllllllllllc}
\toprule
&&&&\multicolumn{8}{c}{Adversarial Test Accuracy($\%)$}           &
                                                                     Clean
                                                                    Test Accuracy ($\%$)    \\ \midrule
\multicolumn{4}{c}{$L_\infty$ radius}               &   \multicolumn{2}{c}{$8/255$}      &   \multicolumn{2}{c}{$10/255$}               &   \multicolumn{2}{c}{$16/255$}   &\multicolumn{2}{c}{$20/255$}&       \\
\multicolumn{4}{c}{Attack iterations}            & $7$     & $20$       &$7$    &$20$           &$7$    &$20$      &$7$    &$20$   &       \\\toprule
\multirow{10}{*}{\rotatebox[origin=c]{90}{White Box}} &
                                                         \multirow{5}{*}{\rotatebox[origin=c]{90}{\footnotesize C10}}&  \multirow{2}{*}{R50}       &N-LR   & 43.1 & 31.0  & 38.5 & 21.8 & 31.2 & 7.8   & 28.9 & 4.5  & $\mathbf{95.4}$ \\
&&&1-LR & $\mathbf{79.1}$ & $\mathbf{78.5}$   & $\mathbf{78.6}$ & $\mathbf{78.1}$          & $\mathbf{77.9}$ & $\mathbf{77.0}$  & $\mathbf{77.1}$ & $\mathbf{76.6}$ & $\mathbf{95.4}$  \\\addlinespace
&&  \multirow{3}{*}{R18} &N-LR  & 40.9 & 26.7   & 35.1 & 16.6         & 26.7 & 4.4   &24.3&2.3 & 94.6 \\
&&&1-LR & 48    & 31.3 & 44.4 & 25.4 & 39.6 & 17.9   & 38.2 & 15.7 & $\mathbf{94.9}$ \\
                      &                   &                &2-LR & $\mathbf{54.7}$ & $\mathbf{37.6}$  & $\mathbf{52.4}$ & $\mathbf{33.1}$   & $\mathbf{48.7}$ & $\mathbf{25.7}$  & $\mathbf{48.0}$ & $\mathbf{23.6}$ & 94.5 \\\cmidrule{4-13}
& \multirow{5}{*}{\rotatebox[origin=c]{90}{\footnotesize C100}}&  \multirow{2}{*}{R50}       &N-LR  & 37.2 & 29.9  & 34.1 & 24.6& 29.8 & 15.9 &  34.1 & 13.3  & $\mathbf{85.8}$ \\
                      &                   &                &1-LR&   $\mathbf{45.3}$  & $\mathbf{38.7}$   & $\mathbf{43.7}$ & $\mathbf{35.8}$  & $\mathbf{40.9}$ & $\mathbf{31.5}$  & $\mathbf{40.0}$  & $\mathbf{29.8}$ &$\mathbf{85.8}$ \\\addlinespace
                      &                   &  \multirow{2}{*}{R18}       &N-LR&   30.6  & 23.2 & 26.4 & 16.9  & 20.5 & 7.42 & 18.4  & 5.1 & 84.1 \\
                      &                   &                &1-LR& $\mathbf{34.5}$\ & $\mathbf{25.4}$   & $\mathbf{31.3}$ & $\mathbf{20.2}$ & $\mathbf{27.3}$ & $\mathbf{13.1}$  & $\mathbf{25.7}$ & $\mathbf{10.8}$ & $\mathbf{84.2}$ \\
                      &                   &                &2-LR& 33.82&24.37 & 30.9 & 19.1& 26.8 & 11.83  & 25.41& 9.9 & 84    \\\midrule
\multirow{6}{*}{\rotatebox[origin=c-10]{90}{ Black Box}} &\multirow{4}{*}{\rotatebox[origin=c]{90}{\footnotesize C10}} &  \multirow{1}{*}{R50}       &1-LR&  64.7 & 56.8  & 59.0 & 47.5  & 51.2 & 28.0 &  48.3 & 20.6  &   95.4   \\\addlinespace
                      &            &  \multirow{2}{*}{R18}       &1-LR& 66.6 & 60.8  & 61.1 & 51.0       & 52.2 & 31.52   & 49.8 & 23.6 &  94.9     \\
                      &            &       &2-LR & 68.0 & 62.5   & 62.3 & 53.4  & 53.8 & 33.5 & 50.8 & 25.8 &   94.5    \\\cmidrule{4-13}
                      &\multirow{3}{*}{\rotatebox[origin=c]{90}{\footnotesize C100}} & R50&1-LR&  52.4 & 46.2  & 48.1 & 38.8  & 42.0 & 25.4   & 48.1 & 20.9 &  85.8     \\\addlinespace
                      &     &  \multirow{2}{*}{R18} &1-LR& 53.0 & 48.6  &47.9  & 41.0   & 41.1 & 26.3 &  38.7 & 20.4  &   84.2    \\
                      &     &        &2-LR     &  51.3     &  47.2     & 46.9       & 39.9      & 39.5    & 24.3      &   37.2  &  19.2   &   84    \\\bottomrule
\end{tabular}\caption{Adversarial Test Accuracy against a
  $\ell_\infty$ constrained PGD adversary with the $\ell_\infty$
  radius bounded by $\epsilon$ and the number of attack steps bounded
  by $\tau$. R50 and R18 denotes ResNet50 and ResNet18
  respectively. C10 and C100 refer to CIFAR10 and CIFAR100~(Coarse
  labels) respectively.}\label{tab:adv-robust-cifar}
\end{table*}

\subsection{Adversarial Robustness}
\label{sec:adversarial-attacks-1}
We now begin our analysis on the impact of low-rank representations on
adversarial robustness. We would like to highlight that in all our
experiments, all the models are trained using clean dataset. This is
important as it shows whether training on clean dataset, with
well-thought priors or regularizers, can actually improve adversarial
robustness without compromising with the clean data test accuracy. 
\renewcommand{\arraystretch}{1.2}
\begin{table}[!htb]\centering\small
\begin{tabular}{l@{\hspace{4pt}}l@{\hspace{2.5pt}}c@{\hspace{4.5pt}}l@{\hspace{2.5pt}}c@{\hspace{4.5pt}}l@{\hspace{2.5pt}}c@{\hspace{4.5pt}}l@{\hspace{2.5pt}}c@{}}
\toprule
&\multicolumn{6}{c}{\parbox[c]{10em}{Adversarial Test
           Acc.($\%$)}}\\\midrule%
$L_\infty$ radius &   \multicolumn{2}{c}{$\nicefrac{8}{255}$}      &   \multicolumn{2}{c}{$\nicefrac{10}{255}$}               &   \multicolumn{2}{c}{$\nicefrac{16}{255}$}   &\multicolumn{2}{c}{$\nicefrac{20}{255}$}      \\
Att. iter  & $7$     & $20$       &$7$    &$20$           &$7$
                                                                                                                                         &$20$      &$7$    &$20$          \\\toprule
  N-LR   & 43.1 & 31.0  & 38.5 & 21.8 & 31.2 & 7.8   & 28.9 & 4.5   \\
  SNIP          &   29.4 & 14.5  & 25.0 & 8.0 & 18.5 & 1.3  & 16.2 & 0.4 \\
  SRN    & 47.8& 37.6& 44.4& 31.4& 39.8&  21.3& 37.5& 18.4\\
  SN     & 54.2& 43.8& 50.8& 36.4 & 45.0 & 22.6 & 42.8 & 18.1 \\
  LR~(Ours) & $\mathbf{79.1}$ & $\mathbf{78.5}$   & $\mathbf{78.6}$ & $\mathbf{78.1}$          & $\mathbf{77.9}$ & $\mathbf{77.0}$  & $\mathbf{77.1}$ & $\mathbf{76.6}$ \\\bottomrule
\end{tabular}\caption{Robustness of other regularization/compression
  techniques to Adversarial Perturbations.
}\label{tab:adv-pert-compre}
\end{table}
We recall that adversarial perturbations are well crafted~(almost
imperceptible) input perturbations that, when added to a clean input, flips the prediction
of the model on the input to an incorrect
one~\citet{szegedy2013intriguing}. Various
methods~\citep{szegedy2013intriguing, goodfellow2014explaining,
  kurakin2016adversarial,mosaavi2016} have been proposed in recent
years for constructing adversarial perturbations. 

Here we use the following three {\bf white-box} 
adversarial attacks to perform our experiments:
\begin{compactenum}
\item Iterative Fast Sign Gradient Method (\ifsgm or IFSGM)~\citep{kurakin2016,madry2018towards},
\item Iterative Least Likely Class Method~(\ill or
  ILL)~\citep{kurakin2016adversarial}, and
\item \deepfool~(DFL)~\citep{mosaavi2016}.
\end{compactenum} 
The reader may refer to Appendix~\ref{sec:types-attacks} for further details on the
attacks. Iter-FSGM is essentially equivalent to the Projected Gradient Descent (PGD) with
$\ell_\infty$ projections on the negative loss function~\citep{madry2018towards}.

We also consider {\bf black-box} version of each of the aforementioned adversarial attacks where the noise is constructed using N-LR. This is to avoid situations where LR might be at an advantage due to the low-rank structure that might enforce a form of gradient masking~\citep{tramer2018ensemble}.

\paragraph{Robustness to Adversarial Attacks}
In Table~\ref{tab:adv-robust-cifar}, we measure the adversarial test
accuracy of ResNet18 and ResNet50 models trained on CIFAR10 and
CIFAR100 respectively. Adversarial test accuracy measures the  accuracy of the model,
subjected to an adversarial attack with a fixed perturbation budget,
on a test set. The adversary used here is an $L_\infty$ PGD
adversary~(or IFSGM) that has two main constraints --- the
$\ell_\infty$ radius and the number of attack steps the PGD algorithm
can take. The $\ell_\infty$ radius is chosen from $\bc{\nicefrac{8}{255},\nicefrac{16}{255},\nicefrac{32}{255},\nicefrac{64}{255}}$
with either 7 or 20 attack steps of PGD. This represents a wide variety of
severity in the attack model and our LR model performs much better
than the N-LR model in all the settings including the black-box settings (refer Table~\ref{tab:adv-robust-cifar}). For example, in the case of white-box attack, for C10, R50, $\ell_\infty = \frac{16}{255}$, and attack iteration of $20$, LR is nearly {\bf 10 times more accurate} than N-LR. 

Above results clearly indicate that LR provides low-rank representations that
are robust to adversarial perturbations and also provide modest
improvements on the test accuracy. It is also interesting to note that
in terms of model complexity, a more complex model~(ResNet50) is
significantly more adversarially robust than ResNet18. This was
observed in~\citet{madry2018towards} and the ordering is true for LR
models as well. In fact, the difference in adversarial robustness
between ResNet18 and ResNet50 is much greater for LR models than for N-LR.%

In Table~\ref{tab:adv-pert-compre}, we also compare our LR
with other methods that reduce some form of intrinsic dimensions of
the parameter space. SNIP~\citep{lee2018snip}, a pruning technique, increases the sparsity of the
parameter, SRN~\citep{sanyal2020stable} reduces the~(stable) rank and
spectral norm of the parameters, and SN~\citep{miyato2018spectral}
reduces the spectral norm of the parameters.  We use the best hyper-parameter
settings suggested for these approaches in their manuscripts. For a
description of these methods and other related approaches please refer to
Appendix~\ref{sec:alt_algs}. Our method performs much better than all
three of these methods
indicating that reducing the dimensionality of the representation
space is much more effective than doing so for the parameter space
when it comes to robustness of the network. 

\begin{table}[t]\centering\small
  \begin{tabular}[h!]{c@{\quad}c@{\enskip}c@{\enskip}c@{\enskip}c}\toprule
   &R18&$\rho$ [DFL]&$\rho$ [ILL]&$\rho$ [IFSGM]\\\hline
    \multirow{3}{0.7cm}{White Box}&\textsf{2-LR}&$\mathbf{1.8\times 10^{-1}}$&$9.8\times 10^{-2}$&$\mathbf{7.6\times 10^{-2}}$\\
           &\textsf{1-LR}&$1.7\times 10^{-1}$&$\mathbf{1.1\times 10^{-1}}$&$6.0\times 10^{-2}$\\
           &\textsf{N-LR}&$1.6\times10^{-2}$&$2.4\times10^{-2}$&$2.1\times10^{-2}$\\\hline
    \multirow{2}{0.7cm}{Black Box}&\textsf{2-LR}&$\mathbf{5.5\times10^{-2}}$&$\mathbf{2.0\times 10^{-1}}$&$\mathbf{7.5\times 10^{-2}}$\\
           &\textsf{1-LR}&$4.7\times10^{-2}$&$1.8\times 10^{-1}$&$5.6\times10^{-2}$\\\hline
  \end{tabular}
  \caption{\small Minimum perturbation required for $99\%$ Adversarial
    Misclassification by ResNet18 models on
    CIFAR10.~Table~\ref{tab:adv_rob_pert_eps} in Appendix shows the $\ell_\infty$ perturbations
used.}
\label{tab:adv_rob_pert}
\end{table}
\renewcommand{\arraystretch}{1}

\paragraph{Accuracy vs the amount of adversarial noise}
Next, we compare the change in accuracy
of adversarial classification with respect to the actual amount of
noise added (as opposed to the perturbation budget as in Table~\ref{tab:adv-robust-cifar}). The amount of noise added can be measured using the
normalized $L_2$ dissimilarity measure ($\rho$), defined as:
$$ \rho = \bE\left[\norm{\vx_a -\vx_d}_2/\norm{\vx_d}_2\right],$$
where $\vx_d$ and $\vx_a$ are the clean and adversarially perturbed
samples, respectively. It  measures the magnitude of the
noise~\citep{mosaavi2016} in the input corresponding to a 
certain adversarial misclassification rate\footnote{Similar to the setting in~\citet{kurakin2016adversarial}, the
noise is added for a pre-determined number of
steps.}.

Here we also consider another model, for comparison, we call {\em Bottle-LR} model. It contains an {\em explicit bottleneck} low-rank layer, rather than \lr, which is essentially a fully connected layer without any non-linear activation where the weight matrix $\vec{\bar{W}}\in\reals^{q\times q}$ is parameterized by $\vec{\bar{W}}_l\in{\reals^{q\times r}}$, where $r\leq q$, so that $\vec{\bar{W}}=\vec{\bar{W}}_l\vec{\bar{W}}_l^\top$. Note, by design, it can not have rank greater than $r$. 

Figure~\ref{fig:adv_pert} shows that as the noise increases, the accuracy of N-LR models decreases much
faster than the LR models.  Specifically, to reach an adversarial
mis-classification rate of $50\%$, our models require about twice the noise as
the N-LR or Bottle-LR.
~\emph{Therefore, for \emph{all} kinds of attacks we considered, LR models consistently
outperform N-LR and Bottle-LR.} More experiments provided in Appendix~(Fig~\ref{fig:adv_pert_svhn_vgg}).
\begin{wrapfigure}{r}{.6\textwidth} \vspace{-0.2em}
\parbox[b][22pt][t]{0.17\linewidth}{Input}
\def\svgwidth{0.18\linewidth}
\hspace{-2pt}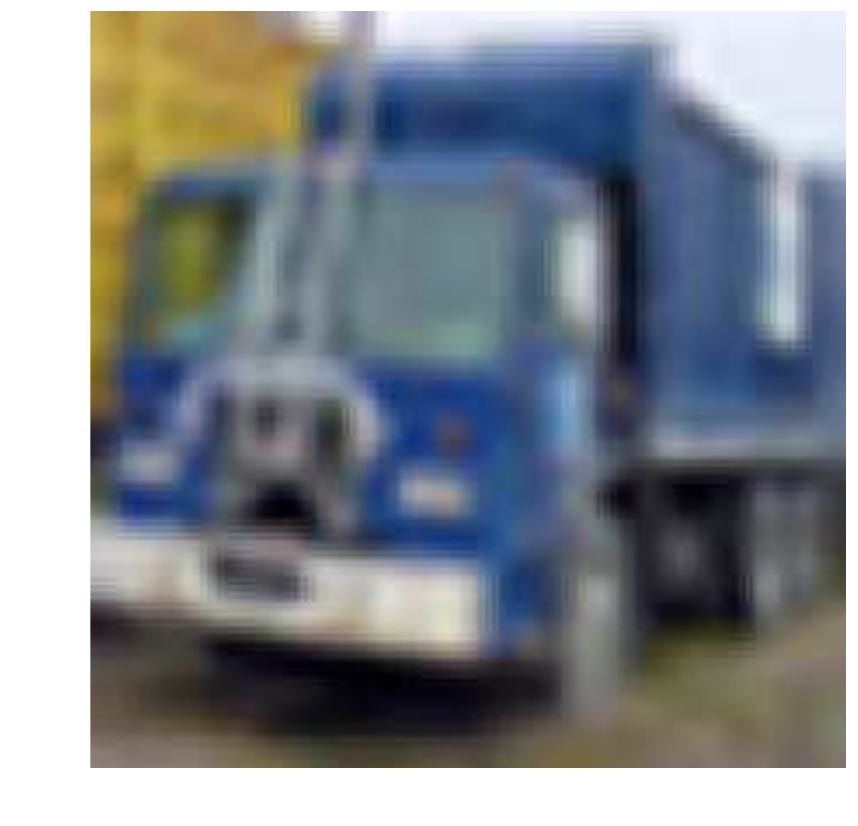\hspace{0pt}
\def\svgwidth{0.18\linewidth}
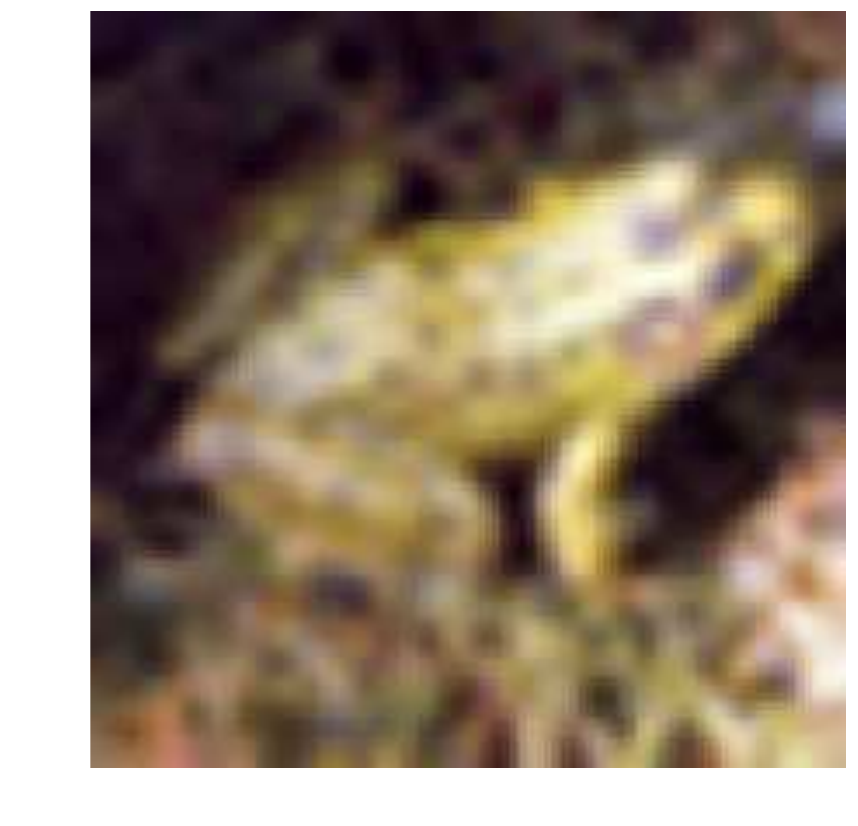%
\hspace{0pt}
\def\svgwidth{0.18\linewidth} 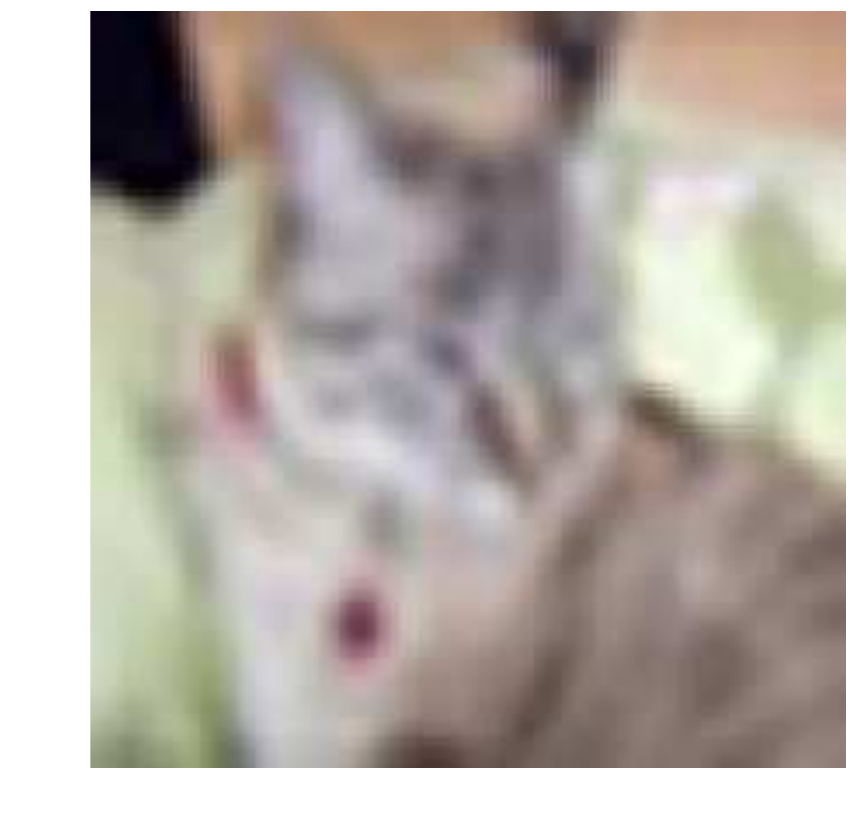\\
\parbox[b][22pt][t]{0.17\linewidth}{{2-LR}}\def\svgwidth{0.18\linewidth}
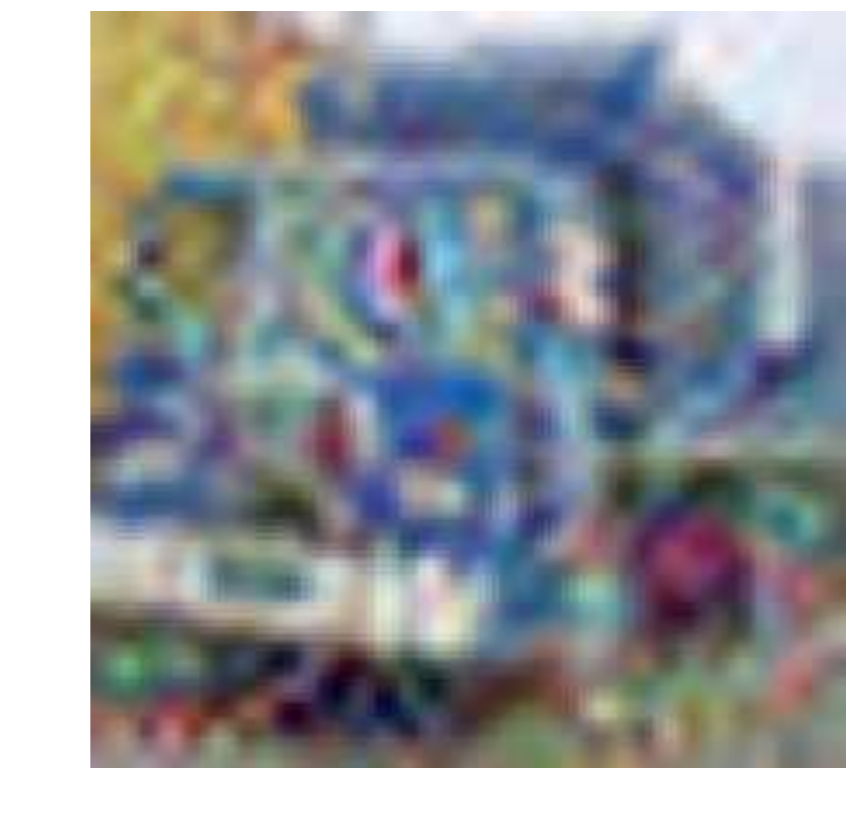
\def\svgwidth{0.18\linewidth} 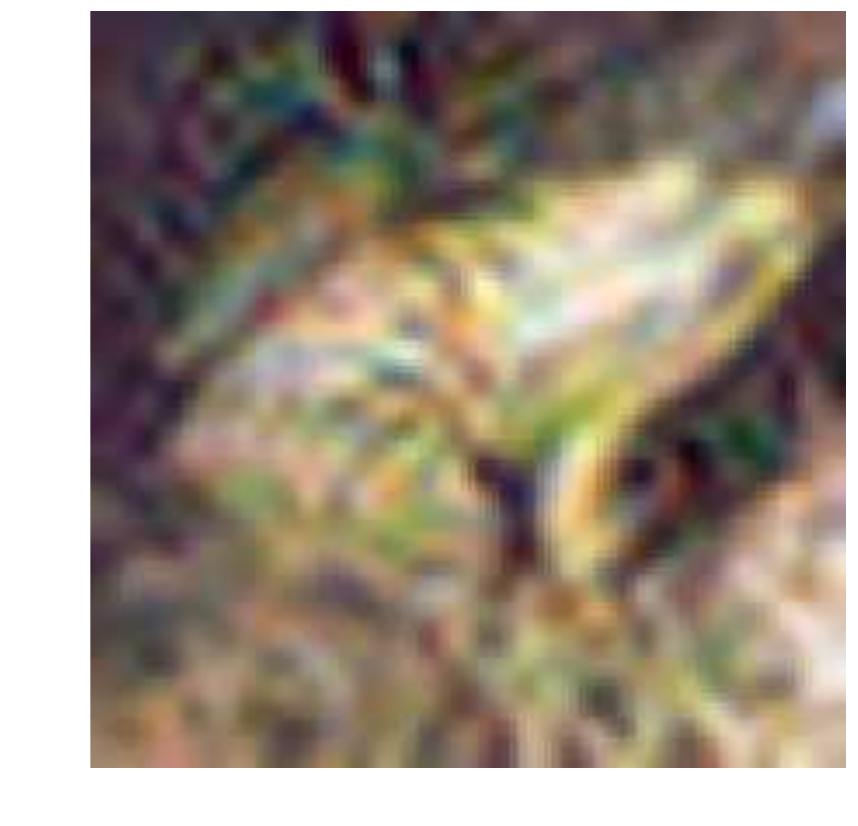
\def\svgwidth{0.18\linewidth} 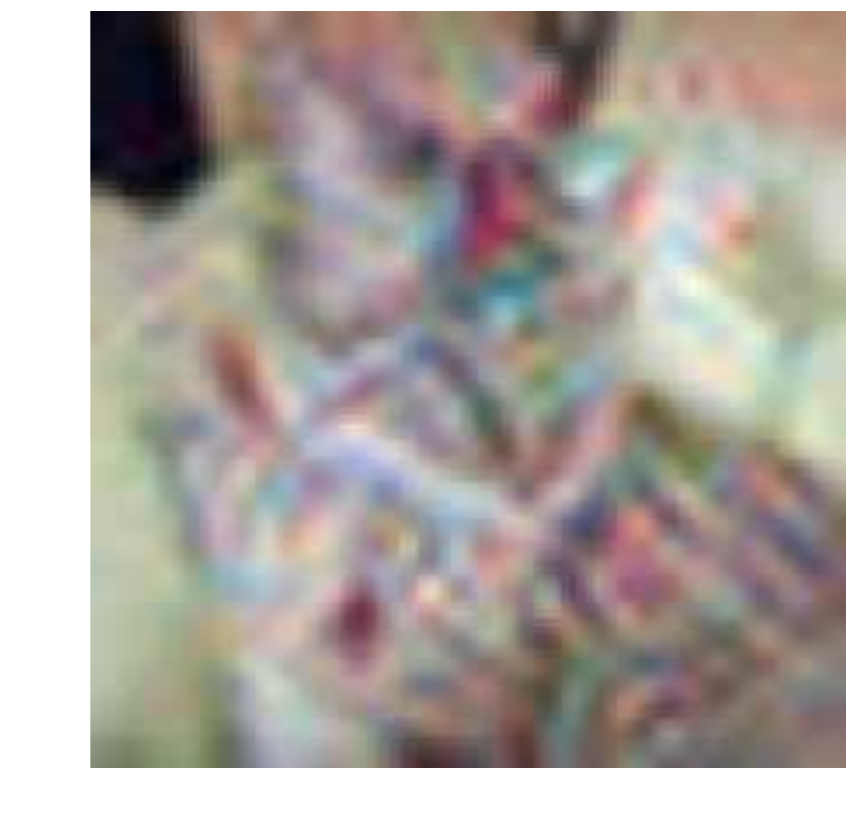\\
\parbox[b][22pt][t]{0.17\linewidth}{{1-LR}}\def\svgwidth{0.18\linewidth}
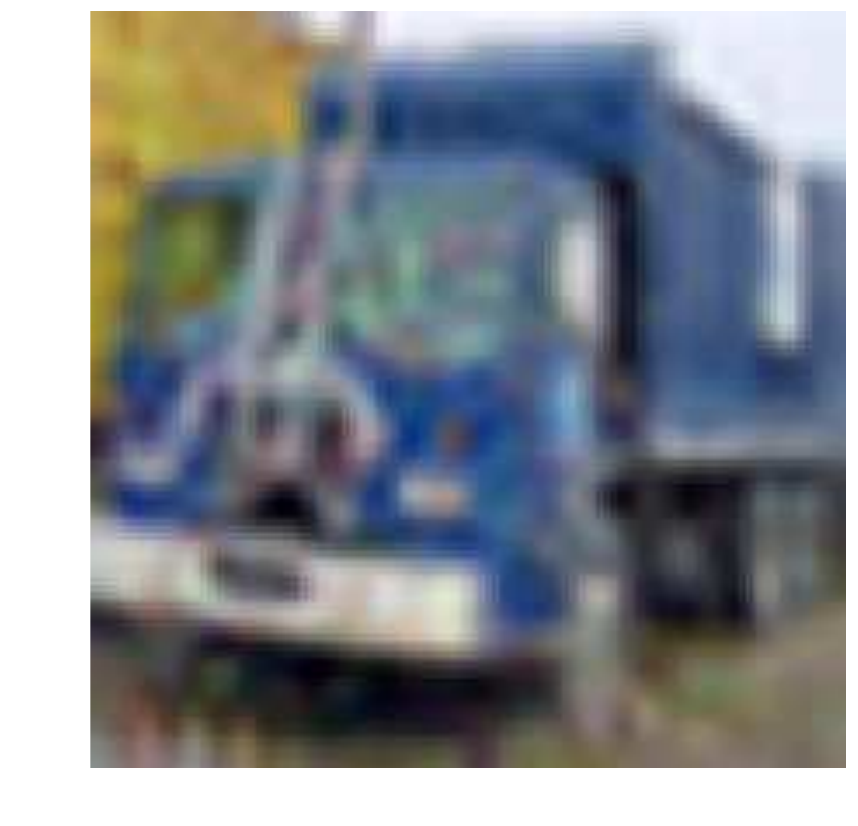
\def\svgwidth{0.18\linewidth} 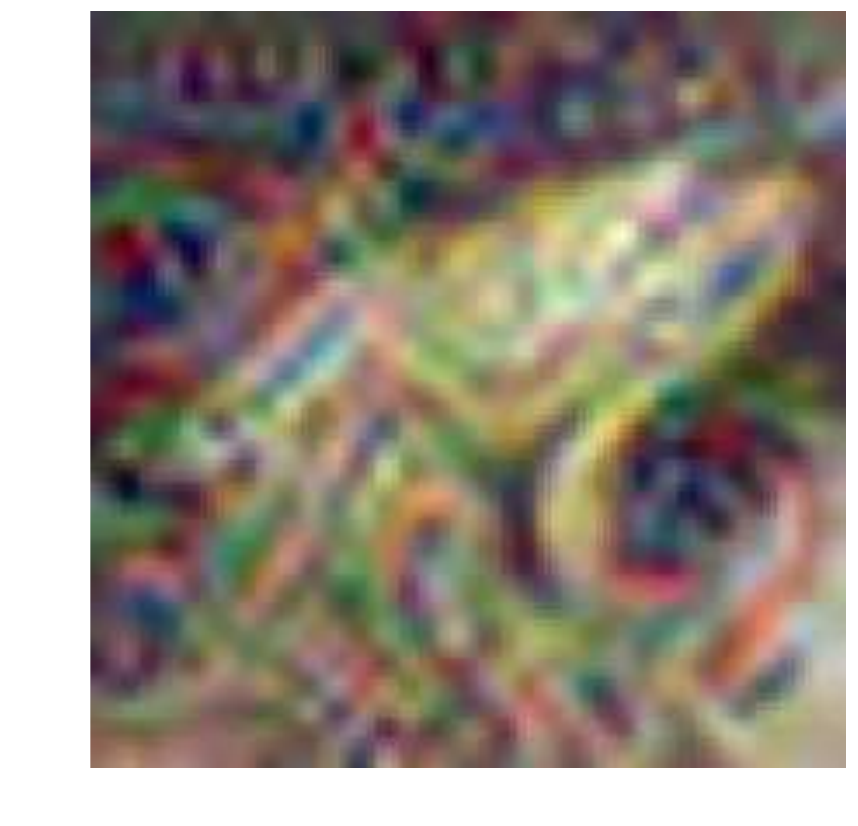
\def\svgwidth{0.18\linewidth} 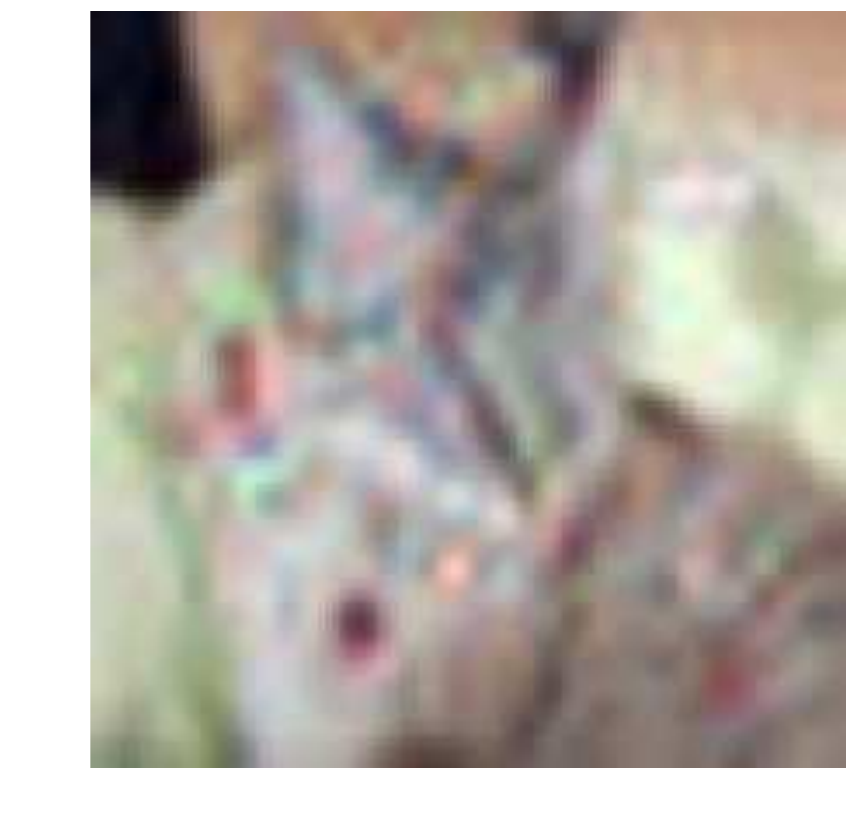\\
\parbox[b][22pt][c]{0.17\linewidth}{ResNet}\def\svgwidth{0.18\linewidth}
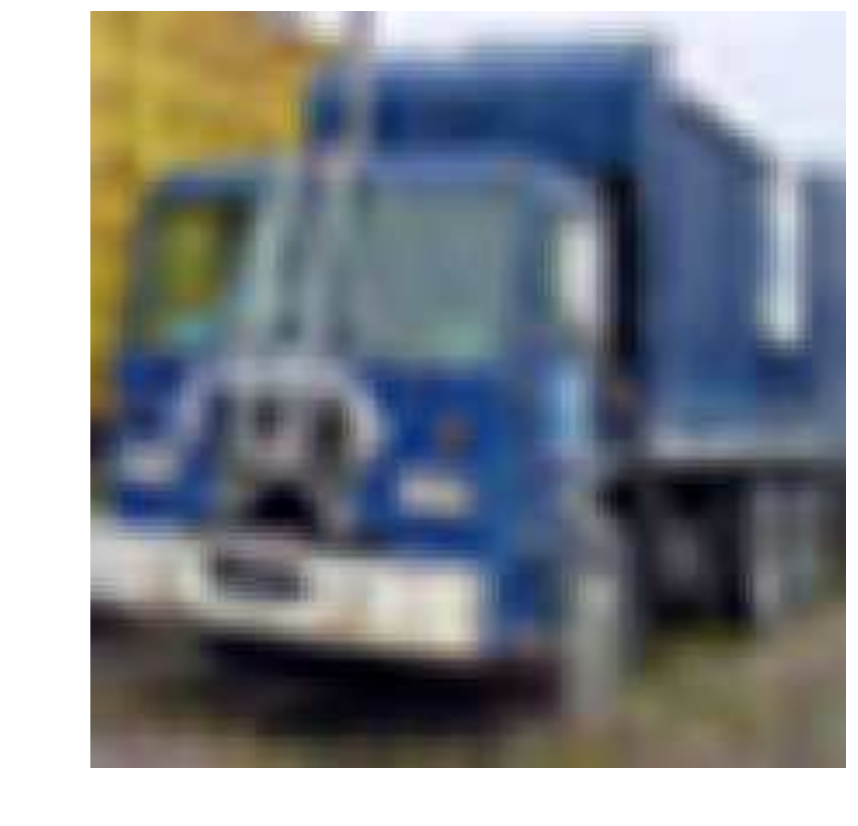
\def\svgwidth{0.18\linewidth} 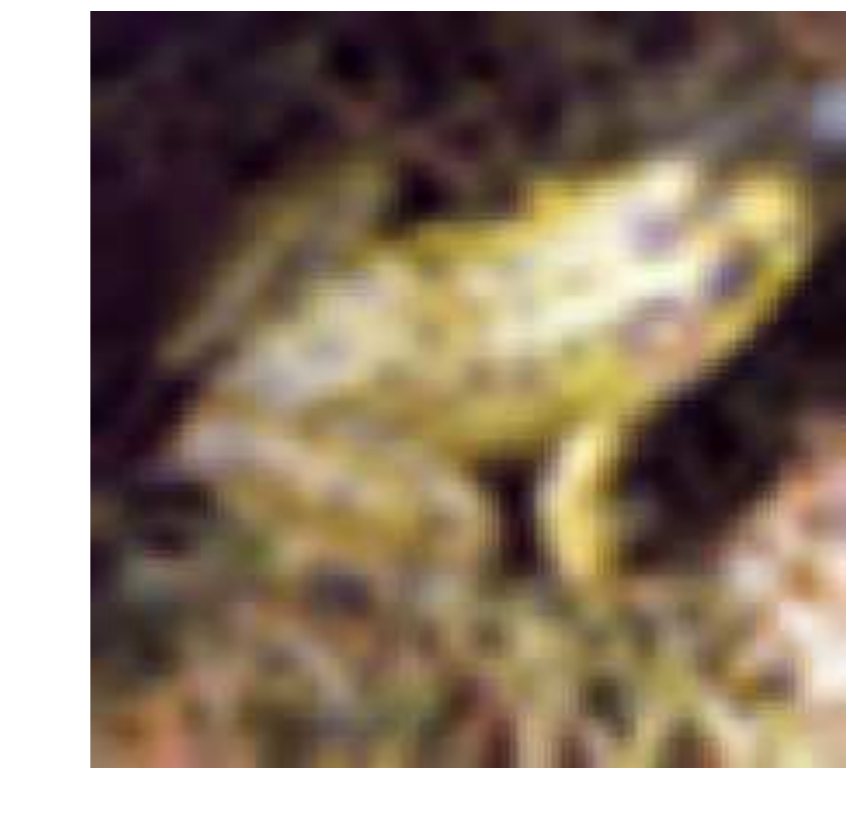
\def\svgwidth{0.18\linewidth} 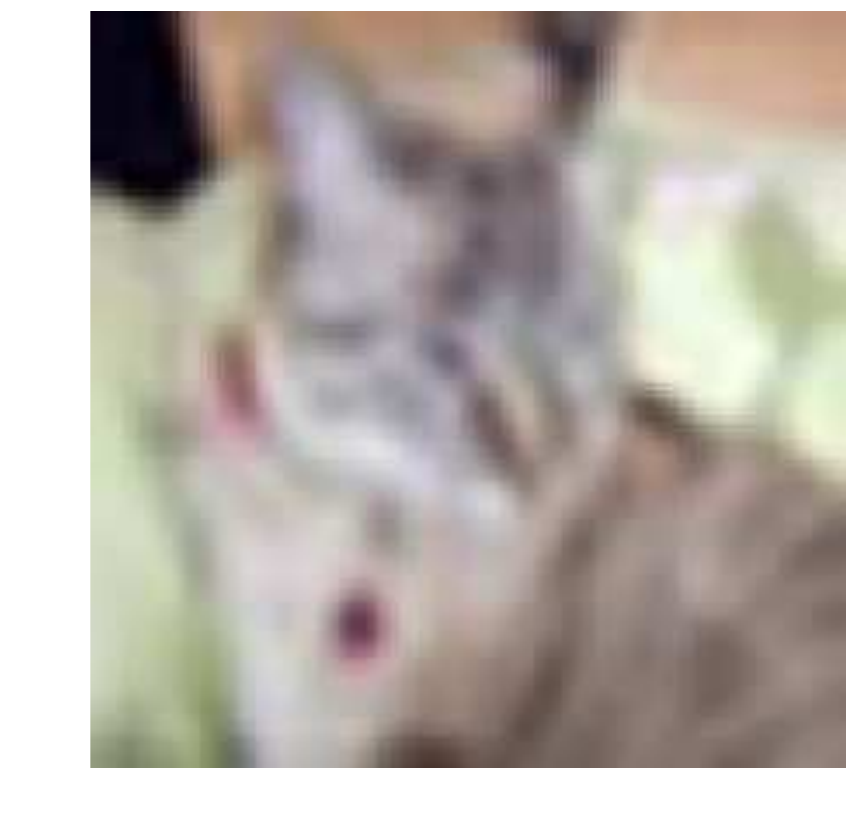
	 \caption{Adversarial images using DeepFool against different
classifiers. The images against low-rank (LR) classifiers require more
perturbations and are noticeably different.}
\label{fig:adv_images_noticeable} %
\end{wrapfigure}

Even though the rank constraint of Bottle-LR is the same as LR models, and they are placed at the same position as the LR layers, the inferior performance of Bottle-RL (sometimes even worst than N-LR) can be explained using the fact that the explicit bottleneck has a \emph{multiplicative} effect on the
 back-propagated gradients, whereas, LR's impact on gradients is \emph{additive}. With a bottleneck layer, the gradient of any $\vec{W}_{j}$ before the bottleneck layer is $\nicefrac{\partial
   \mathcal{L}}{\partial \vec{W}_{j}}=(\nicefrac{\partial
   \mathcal{L}}{\partial \vec{z}_l}) \bar{W} (\nicefrac{\partial \vec{a}_{l-1}}{\partial
   \vec{W}_j})$ where $\vec{a}_l$ and $\vec{z}_l$ are the activations
 and the pre-activations of the $l^{\it{th}}$ layer respectively. Thus, especially during early stages of training when $\bar{\vec{W}}$
 is not yet learned, important directions in  $(\nicefrac{\partial
   \vec{a}_{l-1}}{\partial \vec{W}_j})$ can be cancelled out due to the low-rank
 nature of $\bar{\vec{W}}$ thus making  $\nicefrac{\partial \mathcal{L}}{\partial \vec{W}_{j}}$  uninformative. In the case of LR, the gradients from $\mathcal{L}_c$ (depending on $W$) and
 $\mathcal{L}$ are additive and thus the low-rank $\vec{W}$ only
 affects the gradients from $\cL$ additively. The auxiliary loss $\mathcal{L}_c$ has
 less direct impact on early training (in terms of classification
 error).
 We believe this relative ``smoothness'' of our
 approach is the reason why it has a better performance on these other tasks. %

\paragraph{Minimum adversarial perturbation for $99\%$ misclassification}
Our next experiment is along the lines of that reported in~\citet{mosaavi2016}.
Table~\ref{tab:adv_rob_pert} shows the average minimum perturbation (measured
by $\rho$) required to make the classifier mis-classify more than $99\%$ of the
adversarial examples, constructed from a uniformly sampled subset of the test
set. Appendix~\ref{sec:mimim-pert-succ} describes the setup in greater detail.
Even under this scheme of attacks, our models perform better than N-LR as LR models require 4 to 11 times the amount of noise required by  N-LR
models to be fooled by adversarial attacks. This can be clearly
visualized in Figure~\ref{fig:adv_images_noticeable}. The
adversarial images for 2-LR and 1-LR are noticeably much more
perturbed than N-LR. 

\renewcommand{\arraystretch}{1}
\subsection{Noise Stability} 
To gain some understanding of this visibly better adversarial
robustness of LR models, in this section, we study the noise stability
behaviour of LR models in detail. Specifically,
we show that LR models~(and its representations) are significantly
more stable to input perturbations at test time even though training
was performed using clean data. In addition, we also measure a
quantity called {\em layer cushion} introduced by~\citet{arora18b}
that reflects the noise stability properties of layers in deep
networks.
\begin{table}[t]
  \centering\footnotesize
  \begin{tabular}{lllllll}\toprule
    \multicolumn{2}{c}{Pert. Prob.~($p$)}& 0.4 & 0.6 & 0.8 & 1.0 \\\toprule
    \multirow{2}{*}{R50}&N-LR& $69.7$ & $26.1$ & $12.6$ & $11.3$\\                                                          &1-LR&$\mathbf{75.1}$&$\mathbf{34.2}$&$\mathbf{15.8}$
                        &$\mathbf{13.0}$&\\\addlinespace
                                                      \multirow{3}{*}{R18}&N-LR &$57.7$&$27.3$&$13.0$&$7.2$&\\
                                                      &1-LR&$\mathbf{75.1}$&$33.0$&$15.2$&$11.0$\\
                                                      &2-LR&$74.1$&$\mathbf{35.5}$&$\mathbf{16.4}$&$\mathbf{11.5}$\\\bottomrule
  \end{tabular}
  \caption{Test accuracy of ResNet50~(R50) and ResNet18(R18) to
    Gaussian noise~$\cN\br{0,\nicefrac{128}{255}}$ introduced at
    each pixel with probability $p$. Evaluated on CIFAR10.}
  \label{tab:rand-noise-robust}
\end{table}

\paragraph{Random Pixel Perturbations}
In ~Table~\ref{tab:rand-noise-robust}, we measure the test accuracy
when the input is perturbed with a random additive noise.  Specifically, for a given pixel and for a given \emph{pixel perturbation
  probability} $p\in\bc{0.4, 0.6, 0.8, 1.0}$, we draw a sample from a
Bernoulli distribution  parameterized by $p$ to decide whether to
perturb the pixel or not. If the outcome of the draw is $1$, the pixel
is  perturbed with  a Gaussian noise drawn
from~$\cN\br{0, \nicefrac{128}{255}}$. This is done for all pixels in
the test set and the test accuracy is measured over this perturbed dataset. For varying levels of
perturbation,~Table~\ref{tab:rand-noise-robust} shows that LR models
are significantly  more stable to Gaussian
noise than N-LR. Our experiments indicate that learning a model that
cancels out irrelevant directions in the representations suppresses
the propagation of the input noise in a way to reduce its affect on
the output of the model. Interestingly, the level of Gaussian noise
seems to not vary the test accuracy as much as the value of $p$ does.

\begin{figure}[t]
  \begin{subfigure}[c]{0.15\linewidth}
    \centering
    \def\svgwidth{0.99\columnwidth}
    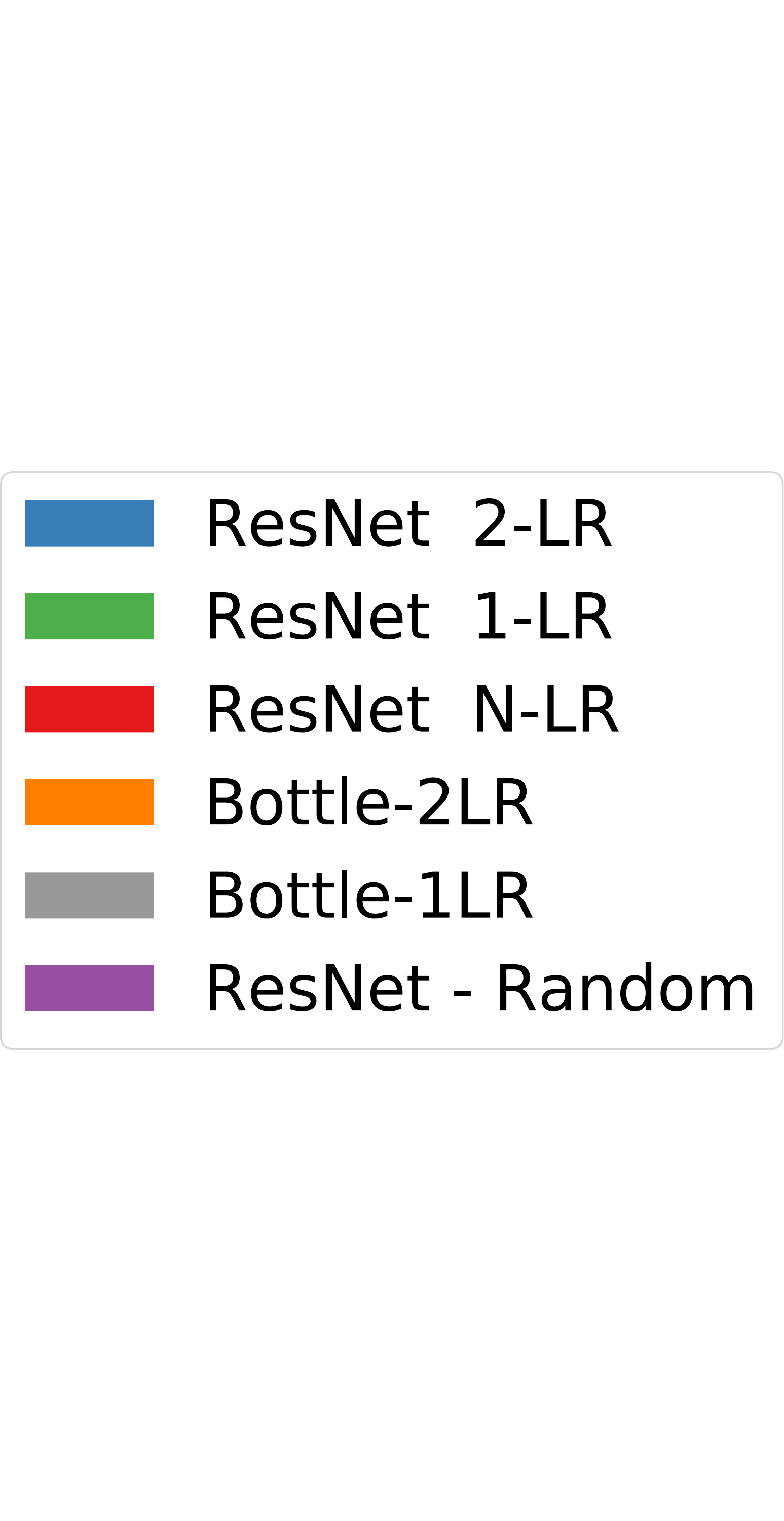
  \end{subfigure}
  \begin{subfigure}[c]{0.32\linewidth}
    \centering
    \def\svgwidth{0.99\columnwidth}
    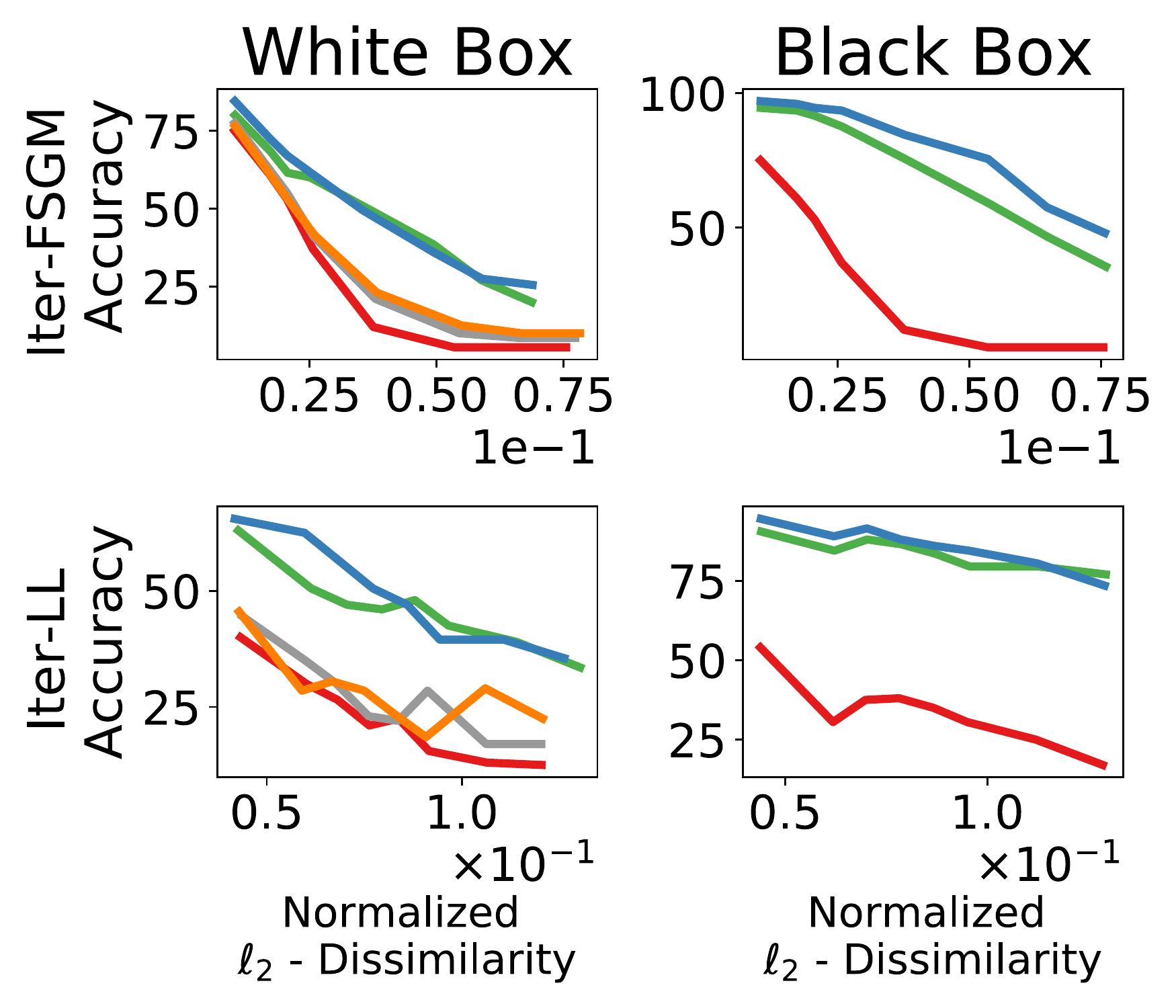 \caption{}\label{fig:adv_pert}
  \end{subfigure}\hfill
  \begin{subfigure}[c]{0.53\linewidth}
    \begin{subfigure}[c]{0.99\linewidth}
      \centering
      \def\svgwidth{0.99\columnwidth}
      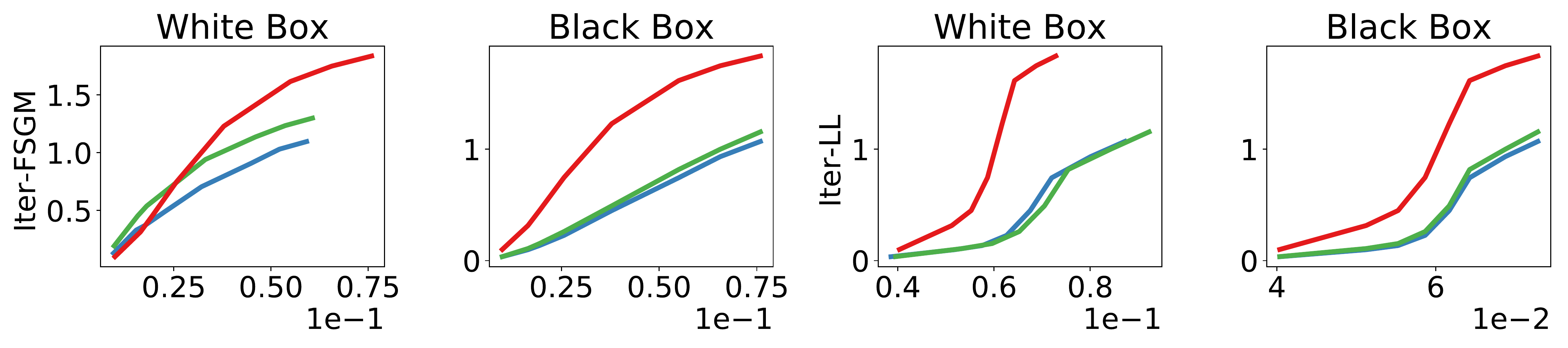\caption{}\label{fig:perturbation_spaces}
    \end{subfigure}\\
    \begin{subfigure}{0.99\linewidth}
      \vspace{-0.0em}
      \begin{subfigure}[c]{0.32\linewidth}
        \centering
        \def\svgwidth{0.99\columnwidth}
        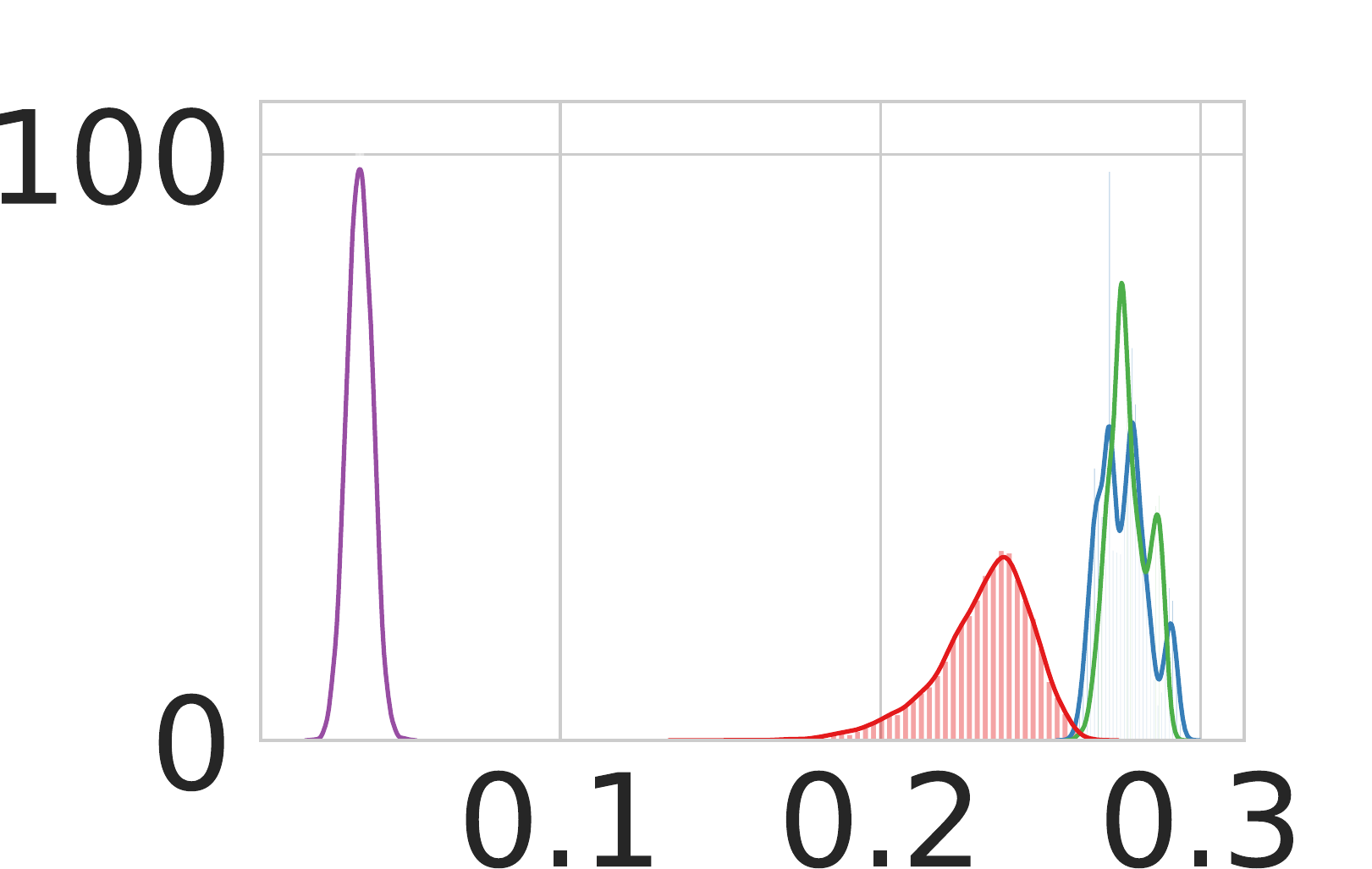
      \end{subfigure}
      \begin{subfigure}[c]{0.32\linewidth}
        \centering
        \def\svgwidth{0.99\columnwidth}
        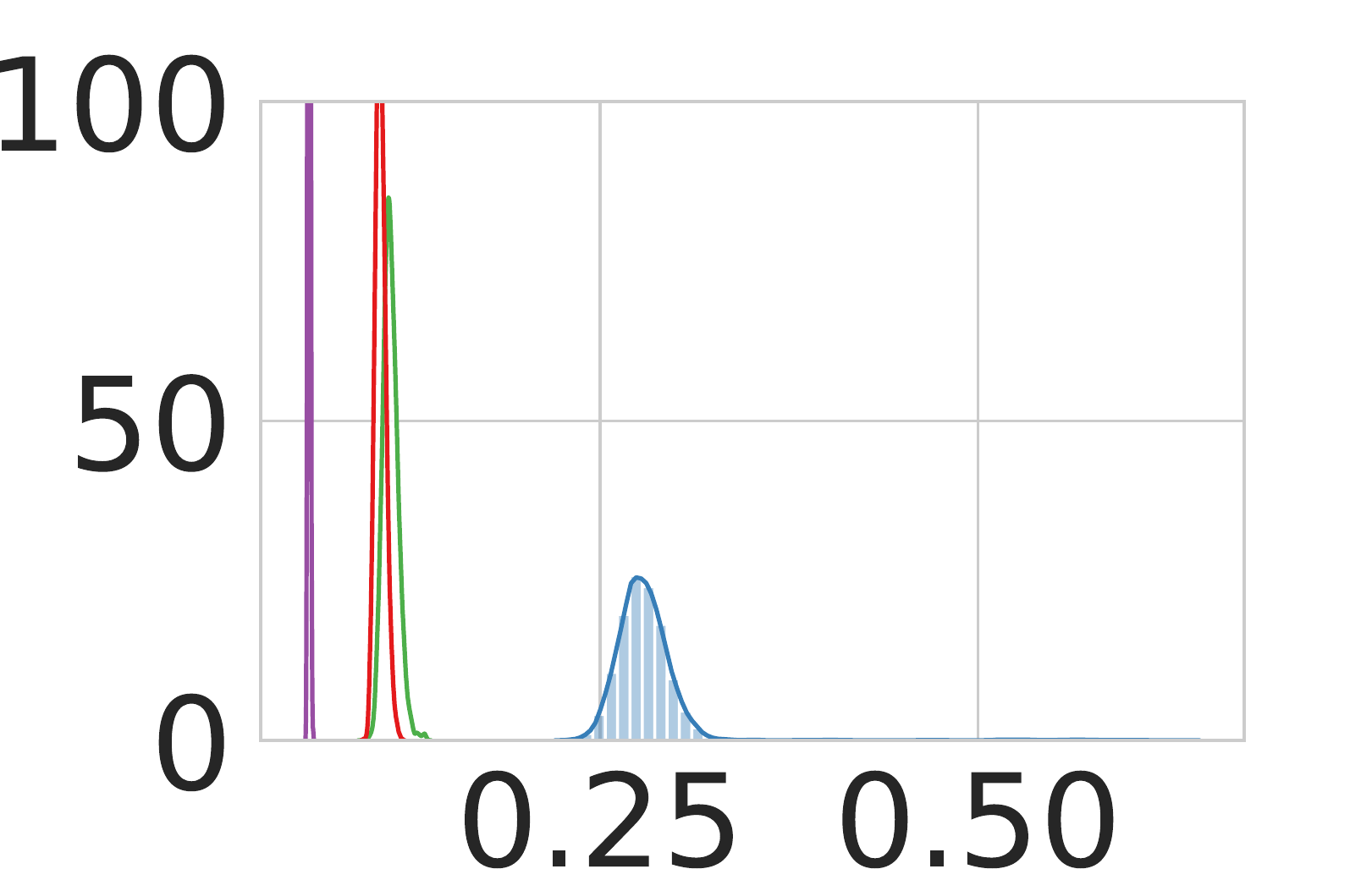
      \end{subfigure}
      \begin{subfigure}[c]{0.32\linewidth}
        \centering
        \def\svgwidth{0.99\columnwidth}
        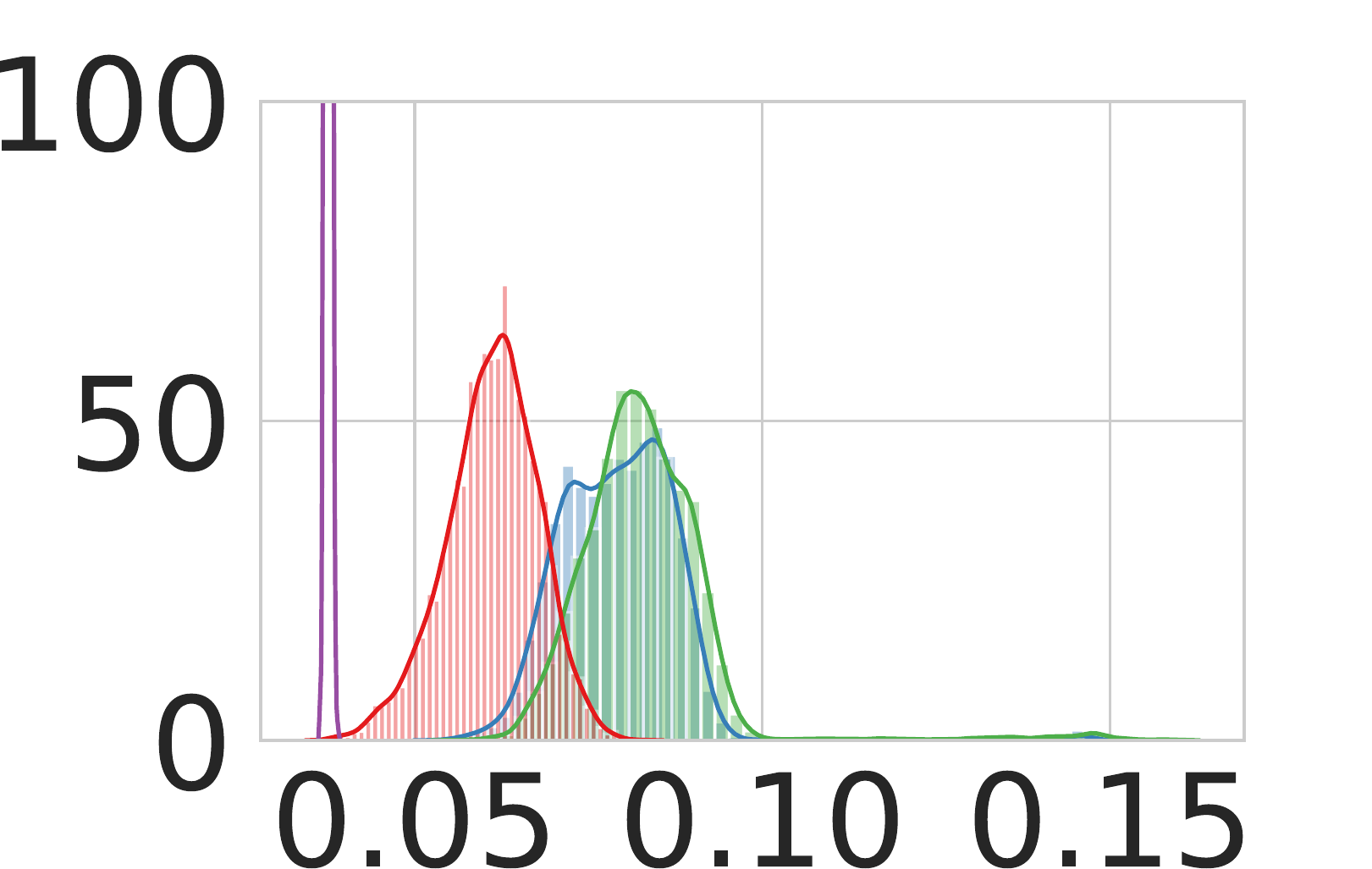
      \end{subfigure}\caption{}\label{fig:Arora}
    \end{subfigure}
  \end{subfigure}
  \vspace{-1em}
  \caption[Change in
  perturbation]{~\ref{fig:adv_pert} shows adversarial accuracy of ResNet18
    plotted against magnitude of
    perturbation~(measured by $\rho$)
    on CIFAR10. ~\ref{fig:perturbation_spaces} shows the adversarial
  Perturbation in Input Space~(x-axis) and Perturbation in
  Representation Space~(y-axis) (Lower is better).~\ref{fig:Arora} shows the
  layer cushion for the 1) last FC layer of ResNet18, 2) the last
  layer in $3^{\it rd}$ ResNet block ,and 3) the last layer in the
  $4^{\it th}$  ResNet block. (Right is better) }   \label{fig:pert} \vspace{-1em}
\end{figure}
\paragraph{Stability of representations}
In Figure~\ref{fig:perturbation_spaces}, we show how the input adversarial perturbations propagate and impact the feature space representations. Specifically, for a given adversarial perturbation $\delta$ to an input $\vec{x}$, the $x$-axis is the normalized $L_2$ dissimilarity score in the input space i.e. ${\norm{\delta}^2}/{\norm{\vec{x}}^2}$ and the $y$-axis represents the corresponding quantity in the representation space i.e. $\norm{f_{\ell}^{-}\br{\vec{x+\delta} } - f_{\ell}^{-}\br{\vec{x} }}^2 /{\norm{f_{\ell}^{-}\br{\vec{x}}}^2}$. The representations $f_{\ell}^{-}(.)$ here are taken from before the last fully connected layer. As our experiments suggest, the LR model significantly attenuates the adversarial perturbations thus making it harder to fool the softmax classifier. This observation further supports the excellent robustness of LR.
\paragraph{Layer cushion}
 \citet{arora18b} gives empirical
evidence that deep networks are stable towards injected Gaussian noise
and use a variation of this noise stability property to derive more
realistic generalization bounds. 
For any layer $i$, ~\citet{arora18b} defines the \emph{layer cushion} as the largest number $\mu_i$ such that for all examples
$\vec{x}$ in the training set:~$ \mu_i\norm{\vec{W}_i}_F\norm{\phi\br{\vec{x}_{i-1}}} \le
  \norm{\vec{W}_i\phi\br{\vec{x}_{i-1}}}$ where $\vec{W}_i$ is the weight matrix of the $i^{\it{th}}$ layer,
$\vec{x}_{i-1}$ is the pre-activation of the $i-1$ layer and $\phi$ is the
activation function. Higher the value of $\mu_i$, better is the
generalization ability of the model. Figure~\ref{fig:Arora} shows
histogram plots of distribution of this ratio for the examples in
CIFAR10 for four models---2-LR, 1-LR, N-LR and a {\em random network}
(randomly initialized, no training done) with the same
architecture.

As Figure~\ref{fig:Arora} shows, the histograms of the
LR models are to the right of the N-LR model, which is further to the
right of the randomly initialized network. Interestingly, the layer
cushion histograms of the last linear layer as well as the last layer
of the $4^{\it}$ residual block is approximately the same for both
1-LR and 2-LR whereas the histogram for the layer cushion of the last
layer in the 3rd ResNet block of the 2-LR model is significantly to
the right of the 1-LR and N-LR models. This can be explained by the
fact that 2-LR is the only model that has the low-rank regularizer
after the 3rd ResNet block and it thus has higher layer cushion~({\em and
thus better noise stability}) for the last layer in the 3rd ResNet block.
Please refer to Appendix~\ref{sec:gen_bounds} for more information on these plots and the plots for other layers.

\subsection{Compression of Model and Embeddings}
\label{sec:comp}
\begin{table}[t]
  \begin{subtable}[t]{0.45\linewidth}
\centering\small
  \begin{tabular}{c@{\quad}c@{\enskip}c@{\enskip}c@{}}
    \toprule
  \textbf{R50} & Dim & Acc($\%$) \\ \toprule 
   1-LR & $2k$  & $\mathbf{78.1}$ \\ 
  N-LR & $2k$  & $75.6$ \\ \midrule
   1-LR & 10  & $\mathbf{76.5}$ \\ 
  N-LR & 10 &  $68.4$ \\ \midrule
   1-LR & 5 & $\mathbf{72}$ \\ 
  N-LR & 5 & $48$ \\ \bottomrule
  \end{tabular}
  \caption{ Representation from before the FC layer, trained on
    CIFAR-100. Original dimension here is $2k$.} 
  \label{tbl:low_dim_emb_pred}
\end{subtable}\hspace{10pt}
  \begin{subtable}[t]{0.45\linewidth}
\centering\small
  \begin{tabular}{cc@{\enskip}c@{\enskip}}
    \toprule
    \textbf{R18} &Dim & Acc($\%$)\\ \toprule
    2-LR & $16k$ & $\mathbf{91.14}$ \\
    N-LR & $16k$ & $90.7$ \\\midrule
    2-LR & $20$ &$ \mathbf{88.5}$  \\
    N-LR & $20$ & $76.9$   \\\midrule
    2-LR & $10$ & $\mathbf{75}$ \\
    N-LR & $10$ & $61.7$\\\bottomrule
  \end{tabular}
  \caption{Representation from before the last ResNet block, trained
    on CIFAR10. Original dimension here is $16k$.} 
  \label{tbl:low_dim_lyr_rem}
\end{subtable}\hspace{10pt}
\caption{ \textit{Dim} represents the
  size of the \textit{compressed embedding} on which a separate linear
  classifier was trained. }\label{tbl:low_dim_proj_emb}
\end{table}

Experiments in the previous section show that our algorithm induces a low-rank
structure in the activation space and that these low-rank activations are
significantly more stable to input perturbations. Here we inspect the
discriminative power of these embeddings. Further, due to the low
dimensionality of the representation space, these learned representations can
be compressed without losing their discriminative power.  Among other things,
we show that low dimensional projections of our embeddings, {\em with a size of
less than $2\%$ of the original embeddings}, can be used for classification
with a significantly higher accuracy than similar sized projections of
embeddings from a model trained without our training
modification~(N-LR) on the CIFAR100 dataset.

In order to study compressibility, we use the trained
models~(N/1/2-LR-ResNet50/18) to construct hybrid max-margin models. Learned
representations are first generated using the original trained model on the
training set; then a max-margin classifier~(such as SVM) is trained on these
learned representations. In some of the experiments, before training the
max-margin classifier, the learned representations are projected onto a low
dimensional space by performing PCA on this set of learned representations.
The embedding dimension of the PCA and the particular layer from which the
representations are extracted is specified in Table \ref{tbl:low_dim_proj_emb}.
At test time, the original trained model is used to first obtain a
representation, if necessary, it is then projected using the learnt PCA
projection matrix, and is then classified using the learnt max-margin linear
classifier. Further details appear in Appendix~\ref{sec:experimental-details}. 

\textbf{Representation Compression:} Table \ref{tbl:low_dim_emb_pred} shows
that even with sharply decreasing embedding dimension, the hybrid model trained
using the LR preserves the accuracy significantly more so than N-LR. \emph{Even
with a $5$-dimensional embedding~(400x compression), the LR model looses
only $6\%$ in accuracy, but the N-LR model looses $27\%$}. 

\textbf{Model Compression:} A consequence of forcing the activations of
the $\ell^{th}$ layer of the model to lie in a low dimensional subspace
with minimal reconstruction error, is that a simpler model can 
replace the latter parts of the original model without significant
reduction in accuracy. Essentially, if we train the hybrid max-margin
classifier on the representations obtained from after the third ResNet
block, we can replace the entire fourth ResNet block and the last FC
layer with a linear classifier.
The results of this experiment, in~Table~\ref{tbl:low_dim_lyr_rem}, show
evidence of this. \emph{The entire fourth ResNet block along with the last FC
layer~(containing 8.4M parameters) is replaced by a smaller linear model which
has only $0.02$ times the number of parameters. }This yields a significant
reduction in model size at the cost of a very slight drop in accuracy~($<
1\%$). The second benefit is that as the low dimensional embeddings still
retain most of the \emph{discriminative} information, the inputs fed to the
linear model have a small number of features.

\renewcommand{\arraystretch}{1.1}
\begin{table}[h]\small\centering
  \begin{tabular}[h!]{l@{\quad}cccc@{}}
    \toprule
     &R18  &DFL  &ILL & IFSGM \\\hline
   \multirow{3}{0.7cm}{White Box}&\textsf{2-LR}&$\mathbf{0.43}$&$\mathbf{0.55}$&$\mathbf{0.55}$\\
    & \textsf{1-LR}&$0.38$&$0.35$&$0.48$\\
    & \textsf{N-LR}&$0.01$&$0.04$&$0.02$\\\hline
    \multirow{2}{0.7cm}{Black Box}&  \textsf{2-LR}&$\mathbf{0.44}$&$\mathbf{0.50}$&$\mathbf{0.48}$\\
    &\textsf{1-LR}&$0.29$&$0.31$&$0.33$\\\hline
  \end{tabular}
  \caption{\small Accuracy of classification of adversarial examples by ResNet18 Max Margin Classifiers.}
  \label{tab:max_margin_adv}
\end{table}

\textbf{Robustness of Max Margin Classifiers}: Finally, we show that the features
learned by our models are inherently more linearly discriminative
i.e. there exists a linear classifier which can be used to classify
these features with a wide margin. To this end, in Table~\ref{tab:max_margin_adv}, we show that for LR models, the
max-margin hybrid models are significantly more robust to adversarial attacks than the
corresponding original models. Hybrid max-margin models with
\lrs are particularly more robust than hybrid max-margin models without \lrs against adversarial
attacks. Specifically,  as seen in Table~\ref{tab:max_margin_adv}, \emph{a hybrid model with a
\lr correctly classifies $50\%$ of the examples that had  fooled the
original classifier while for a similar amount of noise, an N-LR hybrid model has negligible accuracy}. Our experimental setup  is explained in
Appendix~\ref{sec:advers-attack-maxim}.

\section{Conclusion}
\label{sec:concl}
We proposed a low-rank regularizer (LR) that encourages deep neural networks to learn representations that lie in a lower-dimensional linear subspace. We conducted a wide range of experiments to investigate the properties of representations learned by LR and report, among other things, that
\begin{enumerate*}
\item [(i)] these representations are robust to both, adversarial and random input perturbations,
\item [(ii)] LR models significantly attenuate the effects of perturbations introduced in the input space on the representation space
\item[(iii)] these representations have more discriminatory power than ones trained on representation from models without \lrs, and
\item[(iv)] these representations provide extreme compression without compromising much with the accuracy.
\end{enumerate*}

It is commonly believed that large sparse feature spaces are better for deep models. What we propose here is the idea that while the features themselves need not be sparse, the existence of a basis in which the feature vectors have a sparse representation can provide benefits. To the best of our knowledge, investigating properties of representations learned from NNs, and the possibility of modifying training procedures to obtain desirable properties in said representations, has remained relatively unexplored. This work shows that these representations possess some intriguing properties, which may well be worthy of further investigation.

\section{Acknowledgements}
\label{sec:acknowledgements}
AS acknowledges support from The Alan Turing Institute under the Turing Doctoral Studentship
grant TU/C/000023. VK is supported in part by the Alan Turing
Institute under the EPSRC grant EP/N510129/1. PHS and PD are supported by the ERC grant
ERC-2012-AdG 321162-HELIOS, EPSRC grant Seebibyte EP/M013774/1 and
EPSRC/MURI grant EP/N019474/1. PHS and PD also acknowledges the
Royal Academy of Engineering and FiveAI.

\bibliography{library}
\bibliographystyle{icml2020}
\newpage
\appendix

\titlespacing{\paragraph}{%
 5pt}{%
 5pt}{0.5em}%

\onecolumn

\section{Alternative Algorithms}
\label{sec:alt_algs}

\subsection*{Sparsity and Rank}
\label{sec:sparsity}

While various techniques to induce sparsity on weights or activations
in neural networks exist, we point out that sparsity doesn't
necessarily lead to low rank, e.g., the identity matrix is
sparse while being full rank. However, we also look
at whether sparsity and rank correlate empirically in the representation space in real trained networks.

Empirically, in ResNet 1-LR, the activations before the
4$^\th$ ResNet block are $39\%$ sparse and the activations after the
4$^\th$ ResNet block are $5\%$ sparse. However, the activations before the
4$^\th$ ResNet block require a larger number of singular values~($>1000$) to
explain greater than $99\%$ of the variance despite the high level of sparsity ($39\%$) whereas
the activations after the 4th ResNet block are explained almost
totally by about 10 singular values.  Thus, sparsity and low rank are
not correlated in this case. Technically, being low rank
means there exists a basis in which the vector has a sparse representation.

In Table~\ref{tab:adv-pert-compre}, we compare against SNIP~\citep{lee2018snip}, a popular technique for
pruning/increasing sparsity of neural networks. We use a sparsity
ratio of $0.7$, which was the sparsest network that SNIP could obtain
without any drop in test-accuracy.However, it should be noted that
adversarial robustness of sparse networks is an active area of
research and orthogonal to the ideas proposed in this paper.

\subsection*{Low Rank Weights} With respect to compression, it is
natural to look at low rank approximations of network
parameters~\citep{Denton2014,jaderberg2014}. By factorizing the weight
matrix/tensor $W$, for input $x$, we can get low rank
\emph{pre-activations} $Wx$. This however does not lead to low rank
\emph{activations} as demonstrated both mathematically (by the
counter-example below) and empirically.

In Table~\ref{tab:adv-pert-compre}, we compare against
SRN~\citep{sanyal2020stable} a simple algorithm for reducing the
stable rank~(a softer version of rank) of linear layers in neural
networks. We use the best hyper-parameter obtained from
the~\citet{sanyal2020stable} paper for these experiments.

\textbf{Mathematical Counter-Example}: Consider a rank 1
\emph{pre-activation} matrix $A$ and its corresponding
\emph{post-activation}(ReLU) matrix as below. It is easy to see that
the rank of \emph{post-activation} has increased to $2$. \[ A
= \begin{bmatrix} &1 &-1 &1\\ &-1 &1
&-1\end{bmatrix}\qquad\text{Relu}(A) = \begin{bmatrix} &1 &0 &1\\ &0
&1 &0\end{bmatrix} \]

\textbf{Empirical Result}: In order to see if techniques for low rank
approximation of network parameters like ~\citet{Denton2014} would
have produced low rank activations, we conducted an experiment by
explicitly making the \emph{pre-activations} low-rank using SVD. Our
experiments showed that inspite of setting a rank of $100$ to the
\emph{pre-activation} matrix, the \emph{post-activation} matrix had
full rank. Though all but the first hundred singular values of the
\emph{pre-activation} matrix were set to zero, the
\emph{post-activation} matrix’s $101^{\it{st}}$ and $1000^\th$
singular values were $49$ and $7.9$ respectively, and its first $100$
singular values explained only $94\%$ of the variance.

We try to explain the above empirical results as follows:
Theoretically, a bounded activation function lowers the Frobenius norm
of the \emph{pre-activation} matrix i.e. the sum of the squared
singular values. However, it also causes a smoothening of the singular
values by making certain 0 singular values non-zero to compensate for
the significant decrease in the larger singular values. This leads to
an increase in rank of the \emph{post-activation} matrix.

\textbf{Structure in Linear Transformation}: Added to these arguments,
it must be noted that widely used transformation in deep networks like
convolution layers etc are highly structured and introducing low
structure while maintaining the structure is not as simple as doing it
for a fully connected layer where the matrix does not have any special
structure that needs to be maintained. 
\subsection*{Bottleneck LR Layer}

\textbf{Bottleneck Layer e.g. Autoencoder}: It is easy to see that the effective
dimension of the representation of an input, obtained after passing
through a bottleneck layer (like an auto-encoder), will not be greater
than the dimension of the bottleneck layer itself. However, due to the
various non-linearities present in the network, while the
representation is guaranteed to lie in a low dimensional manifold it
is not guaranteed to lie in a low rank (affine) subspace.

\textbf{ Factorized Linear Layer~(LR Bottleneck)}: Another alternative is to include the low-rank
projection and reconstruction as part of the network instead of as a
regularizor so that the \lr is an actual layer and not a
\emph{virtual} layer. We have indeed experimented with this setup and
observed that this often made the training very unstable. Also, if one
were to add this bottleneck as a fine-tuning process, the test
accuracy of the network decreases by a much higher extent than it does
for our method.

\section{Algorithmic Details}
\label{sec:alg_details_app}

\subsection{Algorithm for LR Layer}
\label{sec:alg-lr}

We solve our optimization problem~\eqref{eq:aug_opt_2} by adding a
\emph{virtual} low rank layer that penalizes representations that are
far from its closest low rank affine representation. We call it virtual as the layer is not
a part of the network's prediction pathway and can be discarded once
the network is trained. Algorithm~\ref{alg:lr_layer} lists the
forward and the backpropagation rules of the \lr.

\begin{algorithm}[h!]  \centering
   \caption{LR Layer}
   \label{alg:lr_layer}
   \begin{algorithmic}[1]
     \STATE {\bfseries Input:} Activation Matrix $A$, Grad\_input $g$\\
     \STATE{\bfseries Forward Propagation}
     \STATE{$Z  \gets W^\top(A+\vb)$}\footnotemark \COMMENT{Compute the affine \textit{Low rank} projection}
     \STATE {\bfseries Output : $A$}\COMMENT{Output the original activations for the next layer}\\
     \STATE{\bfseries Backward Propagation} \STATE {$D_1 \gets \frac{\lambda_1}{n} \norm{Z - (A + \vb)}_2^2$}
     \COMMENT{Computes the reconstruction loss $\mathcal{L}_c$}
     \STATE {$D_2\gets \frac{\lambda_2}{n}\sum_{i=0}^{n-1}\big\vert\mathbf{1} - \norm{\va_i} \big\vert$}
     \COMMENT{Computes the loss for the norm constraint $\mathcal{L}_N$}
     \STATE {$D\gets D_1 + D_2$}
     \STATE {$g_W \gets \frac{\partial D}{\partial W}, g_i\gets g + \frac{1}{n}\sum_{i=0}^{n-1}\frac{\partial D}{\partial \va_i}$}
     \STATE{\bfseries Output : $g_i$}
     \COMMENT{Outputs the gradient to be passed to the layer before}\\
     \STATE{\bfseries Update Step}
     \STATE{$W \gets W - \lambda g_W$}\COMMENT{Updates the weight with the gradient from $D$.}\\
     \STATE{$W\gets \prnk{k}{W}$ \label{alg:hard_thresh_step}}\COMMENT{Hard thresholds the rank of $W$}
\end{algorithmic}
\end{algorithm} \footnotetext{$\vb + A$ is computed by adding $\vb$ to
every row in $A$}

The rank projection step in Line~13 in
Algorithm~\ref{alg:lr_layer} is executed by a hard thresholding
operator $\prnk{r}{W}$, which finds the best $r$-rank approximation of
$W$. Essentially, $\prnk{r}{W}$ solves the following optimization
problem, which can be solved using a singular value decomposition
(SVD). However, the projection can be very expensive due to the large
dimension of the representations space~(e.g. $16,000$). To get around this, we use the
ensembled Nystr\"om SVD algorithm~\citep{williams2001using,
halko2011finding, kumar2009ensemble}.
  \begin{align}
    \label{eq:svp} \prnk{r}{W} &= \argmin_{\rank{Z} = r}\norm{W
                                 - Z}_F^2 \nonumber
  \end{align}

\textbf{Handling large activation
matrices}: %
Singular Value Projection~(SVP) introduced in
~\citet{jain2010guaranteed} is an algorithm for rank minimization
under affine constraints. In each iteration, the algorithm performs
gradient descent on the affine constraints alternated with a rank-k
projection of the parameters and it provides recovery guarantees under
weak isometry conditions. However, the algorithm has a complexity of
$O(mnr)$ where $m, n$ are the dimensions of the matrix and $r$ is the
desired low rank. Faster methods for SVD for sparse matrices are not
applicable as the matrices in our case are not necessarily sparse.  We
use the ensembled Nystr\"om method~\citep{williams2001using,
halko2011finding, kumar2009ensemble} to boost our computational speed
at the cost of accuracy of the low rank approximation. It is
essentially a sampling based low rank approximation to a matrix. The
algorithm is described in detail in the next section(~\ref{sec:ensembl-nystr-meth}). Though the overall complexity
for projecting $W$ still remains $O(m^2r)$ , the complexity of the
hard-to-parallelize SVD step is now $O(r^3)$, while the rest is due to
matrix multiplication, which is fast on modern GPUs.

The theoretical guarantees of the Nystr\"om method hold only when the
weight matrix of the \lr is symmetric and positive semi-definite (PSD)
before each $\prnk{r}{\cdot}$ operation; this restricts the
projections allowed in our optimization, but empirically this does not
seem to matter. However, note that the PSD constraint is not
really a restriction as all projection matrices are PSD by definition. For example, on
the subspace spanned by columns of a matrix $X$, the projection matrix
is $P = X(X^\top X)^{-1}X^\top$, which is always PSD. We know that a symmetric
diagonally dominant real matrix with non-negative diagonal entries is
PSD. With this motivation, the matrix $W$ is smoothened by repeatedly
adding $0.01\mathbf{I}$ until the \textsf{SVD} algorithm
converges where $\vec{I}$ is the identity matrix.\footnote{The computation of the singular value decomposition
sometimes fail to converge if the matrix is ill-conditioned}. This is
a heuristic to make the matrix well conditioned~(as well as diagonally
dominant) and it helps in the convergence of the algorithm
empirically.

\textbf{Symmetric Low Rank Layer}: %
The Nystr\"om method requires the matrix $W$ of the \lr to be
symmetric and PSD (SPSD), however, gradient updates may make the
matrix parameter non-SPSD, even if we start with an SPSD matrix.
Reparametrizing the \lr fixes this issue; the layer is parameterized
using $W_s$ (to which gradient updates are applied), but the layer
projects using $W = (W_s + W_s^\top)/2$. After the rank projection is
applied to the (smoothed version of) $W$, $W_s := \prnk{r}{W}$ is an
SPSD matrix (using Lemma~\ref{thm:nystrom_sym} in~\ref{sec:proof-lemma}). As a result the updated $W$ is also
SPSD.  This layer also has a bias vector $\vb$ to be able to translate
the activation matrix before performing the low rank projection.

\subsection{Ensembled Nystr\"om Method}
\label{sec:ensembl-nystr-meth}

 Let $W\in \reals^{m\times m}$ be a symmetric positive semidefinite
matrix (SPSD). We want to generate a matrix $W_r$ which is a r-rank
approximation of $W$ without performing SVD on the full matrix $W$ but
only on a principal submatrix\footnote{A principal submatrix of a
matrix $W$ is a square matrix formed by removing some columns and the
corresponding rows from $W$~\citep{Meyer2000MAA}} $Z\in
\reals^{l\times l}$ of $W$, where $l\ll m$. We sample $l$ indices from
the set $\{1\cdots m\}$ and select the corresponding columns from $W$
to form a matrix $C\in \reals^{m\times l}$. In a similar way,
selecting the $l$ rows from $C$ we get $Z\in \mathbb{R}^{l \times
l}$. We can rearrange the columns of $W$ so that \[W = \begin{bmatrix}
Z \quad W_{21}^T \\ W_{21}\quad W_{22}
  \end{bmatrix} \quad \quad C =
  \begin{bmatrix} Z \\ W_{21}
  \end{bmatrix}
\] According to the Nystr\"om approximation, the low rank
approximation of $W$ can be written as
\begin{equation}
  \label{eq:nystrom_approx} W_r = C Z_r^{+}C^T
\end{equation} where $Z_r^{+}$ is the pseudo-inverse of the best $r$
rank approximation of $Z$. Hence, the entire algorithm is as follows.
\begin{itemize}
\item Compute $C$ and $Z$ as stated above.
\item Compute the top $r$ singular vectors and values of $Z$ : $U_r,
\Sigma_r, V_r$.
\item Invert each element of $\Sigma_r$ as this is used to get the
Moore pseudo-inverse of $Z_r$.
\item Compute $Z_r^{+} = U_r\Sigma_r^{-1} V_r$ and $W_r = C
Z_r^{+}C^T$.
\end{itemize} Though by trivial computation, the complexity of the
algorithm seems to be $O(l^2r + ml^2 + m^2l) = O(m^2r)$~(In our
experiments $l = 2r$), it must be noted that the complexity of the SVD
step is only $O(k^3)$ which is much lesser than $O(m^2r)$ and while
matrix multiplication is easily parallelizable, parallelization of SVD
is highly non-trivial and inefficient.

To improve the accuracy of the approximation, we use the ensembled
Nystr\"om sampling based methods~\citep{kumar2009ensemble} by
averaging the outputs of $t$ runs of the Nystr\"om method. The $l$
indices for selecting columns and rows are sampled from an uniform
distribution and it has been shown~\citep{kumar2009sampling} that
uniform sampling performs better than most sampling
methods. \textbf{Theorem 3} in ~\citet{kumar2009ensemble} provides a
probabilistic bound on the Frobenius norm of the difference between
the exact best r-rank approximation and the Nystr\"om sampled r-rank
approximation.

\subsection{Lemma \ref{thm:nystrom_sym}}
\label{sec:proof-lemma}

\begin{lem}
  \label{thm:nystrom_sym} If $X\in\reals^{m\times m}$ is a SPSD matrix
and $X_r\in\reals^{m\times m}$ is the best Nystr\"om ensembled, column
sampled r-rank approximation of $X$, then $X_r$ is SPSD as
well. (Proof in Appendix \ref{sec:proof-lemma})
\end{lem}

\begin{proof} By the Construction of the Nystr\"om SVD algorithm, we
know that $X_r = C W_r^{+}C^T$. We will first show that $W_r^{+}$ is a
symmetric matrix.

  We know that $X$ is SPSD. Let $I$ be a sorted list of distinct
indices such that $|I| = l$. Then by construction of $W$, \[W_{i,j} =
X_{I[i],I[j]}\] Hence, as $X_{I[i],I[j]} = X_{I[j], I[i]}$, $W$ is
symmetric.

  At this step, our algorithm adds $\delta\cdot I$ to $W$ where
$\delta\ge0$. It is easy to observe that $W + \delta\cdot\mathcal{I}$
is positive semidefinite.

  Consider a vector $a\in \reals^{|X|}$.  Create a vector
$\bar{a}\in\reals^m $ where \[ \bar{a}_i =
    \begin{cases} 0 &\text{if } i\not\in I\\ a_i & \text{o.w.}
    \end{cases}
  \]
  
  \begin{equation} a^\top \br{W + \delta\cdot \cI}a= \bar{a}^\top
X\bar{a} + \delta\cdot a^\top\cI a \ge 0 + \delta\norm{a}^2\ge
0\label{eq:spsd_proof}
\end{equation}

Let $W + \delta\mathcal{I}$ be the new $W$ and \eqref{eq:spsd_proof}
shows that $W$ is positive semidefinite.

  Now we will show that $X_r$ is symmetric as well.  As $W$ is
symmetric, there exists an orthogonal matrix $Q$ and a non-negative
diagonal matrix $\Lambda$ such that
  $$W = Q\Lambda Q^T$$ We know that $W_r = Q_{[1:r]}\Lambda_{[1:r]} Q_{[1:r]}^T$ and  $W_r^{+} = Q_{[1:r]}\Lambda_{[1:r]}^{-1} Q_{[1:r]}^T$.\\ Hence,
  \begin{align*} X_r &= C W_r^{+}C^T\\ &=C
Q_{[1:r]}\Lambda_{[1:r]}^{-1} Q_{[1:r]}^T C^T\\ X_r^T &= (C
Q_{[1:r]}\Lambda_{[1:r]}^{-1} Q_{[1:r]}^T C^T)^T\\ &= C
Q_{[1:r]}\Lambda_{[1:r]}^{-1} Q_{[1:r]}^T C^T\\ &= X_r\\
  \end{align*} $\therefore X_r$ is symmetric. We can also see that the
$X_r^T$ is positive semi definite by pre-multiplying and post
multiplying it with a non-zero vector and using the fact that
$W_r^{+}$ is positive semi-definite.

\end{proof}

\section{Experimental details with additional experiments}
\label{sec:experimental-details}

In this section, we first describe the training setting and then describe the setting of each experiment in greater
detail so that they can be easily reproduced.

\textbf{Experimental Settings:} We used a single NVIDIA RTX6000 for training our networks. All
our experiments were run with a learning rate of $0.1$ for 350 epochs
with a batch size of 128 with the learning reduced reduced to 0.01 and
0.001 on the 150th and 250th epoch respectively. The hyper-parameter $l$ in the Nystr\"om
method was set to double of the target rank. The model was pre-trained
with SGD for the first 50 epochs and then Algorithm
\ref{alg:lr_layer} was applied. The rank cutting operation was
performed every 10 iterations.

The target rank for the \lrs, which were placed before the FC layers
was set to $100$ while the target rank for the \lr before the last
ResNet block was set to 500. However, experiments suggest that this
hyper-parameter is \textbf{not very crucial} to the training process as the
training procedure converged the effective rank of the
\emph{activation matrix} to a value lesser than the designated target
ranks.

In this section, we also show experiments using  VGG19 2-LR. It contains two \lrs. The VGG model has three FCs
  after 16 convolution layers. The \lrs are before the $1\it{st}$ and
  $3\it{rd}$ FC layers~(to maximize the distance between them).

 We also report results on the fine labels of CIFAR100 and on SVHN in
 this section. For the SVHN dataset, as is common practice, we used
 both the training dataset as well as the extra examples without any
 data augmentation.
 
 The linear classifier in the \emph{hybrid max-margin} model
 is trained using SGD with
hinge loss and $L_2$ regularization with a coefficient of $0.01$. The
learning rate is decreased per iteration as $\eta_t =
\frac{\eta_0}{\br{1+\alpha t}}$ where $\eta_0$ and $\alpha$ are set by
certain heuristics~\footnote{\url{https://goo.gl/V995mD}}.

\subsection{Experiments on test accuracy }
\label{sec:no-decr-perf}
We investigate whether the additional constraints on training have any
significant effect on model
performance. Tables~\ref{tab:adv-robust-cifar},\ref{tab:test-acc-cifar-vgg},
and \ref{tab:test-acc-svhn} show that the additional constraints cause no loss in accuracy. In some cases, we even \emph{observe modest gains in performance}.

To test whether the learned representations can be used in a different
task, we conduct a transfer learning exercise where embeddings
generated from a ResNet-50 model, trained on the coarse labels of
CIFAR-100, are used to predict the fine labels of CIFAR-100. A set of
ResNet-50 hybrid max-margin classifiers are trained for this purpose on these embeddings.
First, two ResNet-50 models were trained with and without
the \lr respectively on the coarse labels of CIFAR-100. Then, 2048
dimensional embeddings were extracted from after the fourth ResNet
block using the train set and the test set of CIFAR 100 for both the models. Essentially, this  resulted in two new datasets for training the linear
classifiers - one for the LR model and the other for the N-LR model.

\begin{table}[t]\centering
  \begin{subtable}{0.32\linewidth}
      \begin{tabular}{l@{\quad}c@{\quad}c@{}}
      \toprule
      Models &\multicolumn{1}{c}{Acc (\%)}\\\midrule
     VGG19 2-LR & $\mathbf{89.8}$\\ %
      VGG19 N-LR & $89.1$\\ \bottomrule
    \end{tabular}  \caption{CIFAR10}
    \label{tab:test-acc-cifar-vgg}
  \end{subtable}
  \begin{subtable}{0.32\linewidth}
    \begin{tabular}{l@{\quad}c@{\quad}c@{}}
      \toprule
      Models &\multicolumn{1}{c}{Acc (\%)}\\\midrule
      R18 2-LR &   $97.86$\\ %
      R18 1-LR &   $\mathbf{97.98}$\\ %
      R18 N-LR & $97.97$\\ %
      R18 Bottle-2LR &  $97.70$\\ \midrule
      VGG19 2-LR & $\mathbf{97.79}$\\ %
      VGG19 N-LR & $97.21$\\ \bottomrule
    \end{tabular}  \caption{SVHN}
    \label{tab:test-acc-svhn}
  \end{subtable}
  \begin{subtable}{0.3\linewidth}
    \begin{tabular}{l@{\quad}C{1.5cm}C{1.5cm}}
      \toprule
      &\multicolumn{2}{c}{ Acc($\%$)}\\ %
      Models& Coarse & Fine \\\hline%
      R50 1-LR & $\mathbf{78.1}$  & $48$ \\
      R50 N-LR & $75.6$ & $\mathbf{52}$  \\
      R50 Bottle-1LR & $76$ & $38$  \\\hline%
    \end{tabular}
    \caption{CIFAR-100}
    \label{tab:test-acc-transfer}
  \end{subtable}
  \caption{Table~\ref{tab:test-acc-cifar-vgg} shows test accuracy on
    CIFAR10 for VGG19, Table~\ref{tab:test-acc-svhn} shows test
    accuracy on SVHN for ResNet18 and VGG19, and
    Table~\ref{tab:test-acc-transfer} shows the transfer learning
    experiment as described below.}
\end{table}
  
The embeddings from each train set were used to train a
separate max-margin linear classifier with the same hyper-parameters as
described above but by replacing the coarse labels with fine
labels. The accuracy of the linear classifier trained on the
representations from the LR model and the N-LR model are reported in the third column of
Table~\ref{tab:test-acc-transfer}. The results of the same experiment, when the linear model was trained on the coarse labels
are reported in the second column of Table~\ref{tab:test-acc-transfer}. It
is surprising that the low rank model performs well at all in this
experiment as one would expect that all information in the
representations that are not strictly required in the classification
of original task are discarded from the model.

Table~\ref{tab:test-acc-transfer} shows that the LR model suffers a small
loss of $4\%$ in accuracy as compared to the N-LR model when its
embeddings are used to train a max-margin classifier for predicting
fine labels. It should be noted that the accuracy of the LR model
actually increases when the max-margin classifier is trained to
perform the original task, i.e. classifying the coarse labels. On the
other hand, the Bottle-LR model suffers a loss of $14\%$ in accuracy
compared to N-LR model and shows no significant advantage in the
original task either.

\paragraph{Effective rank of activations for VGG and SVHN}:
\label{psec:vari-rati-capt-vgg}

In Figure~\ref{fig:var_ratio_plots_vgg}, we plot the variance ratio of representations obtained before the first and
third FC layers of VGG19 models trained on CIFAR10. The LR models show a better low rank structure than N-LR models and
is consistent with experiments on ResNet18.

\begin{figure}[t]
  \begin{subfigure}[t]{0.49\linewidth} \centering
\def\svgwidth{0.95\columnwidth}
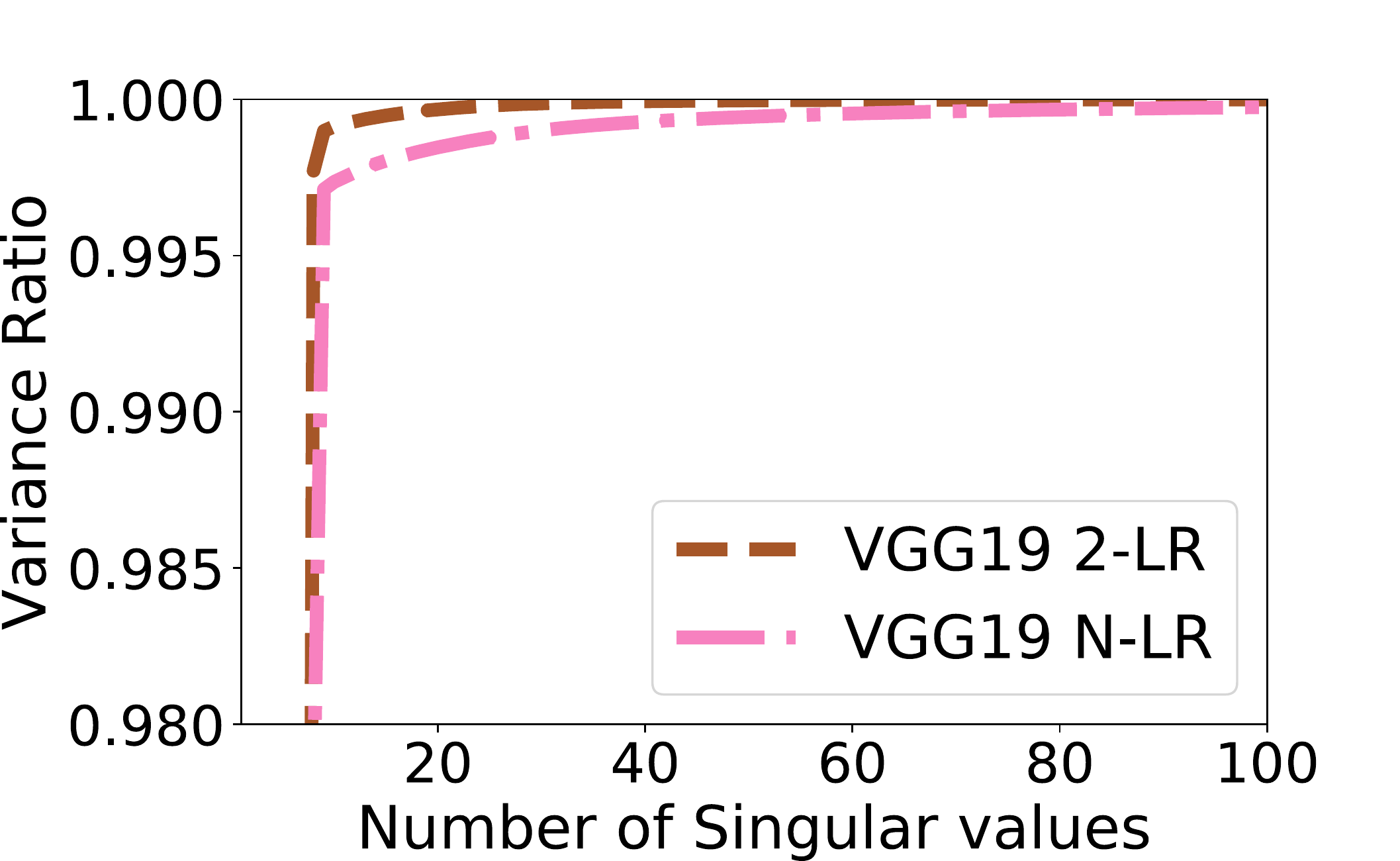
          \caption{Activations after last FC layer.}
          \label{fig:var_1_vgg}
        \end{subfigure}
        \begin{subfigure}[t]{0.49\linewidth} \centering
\def\svgwidth{0.95\columnwidth}
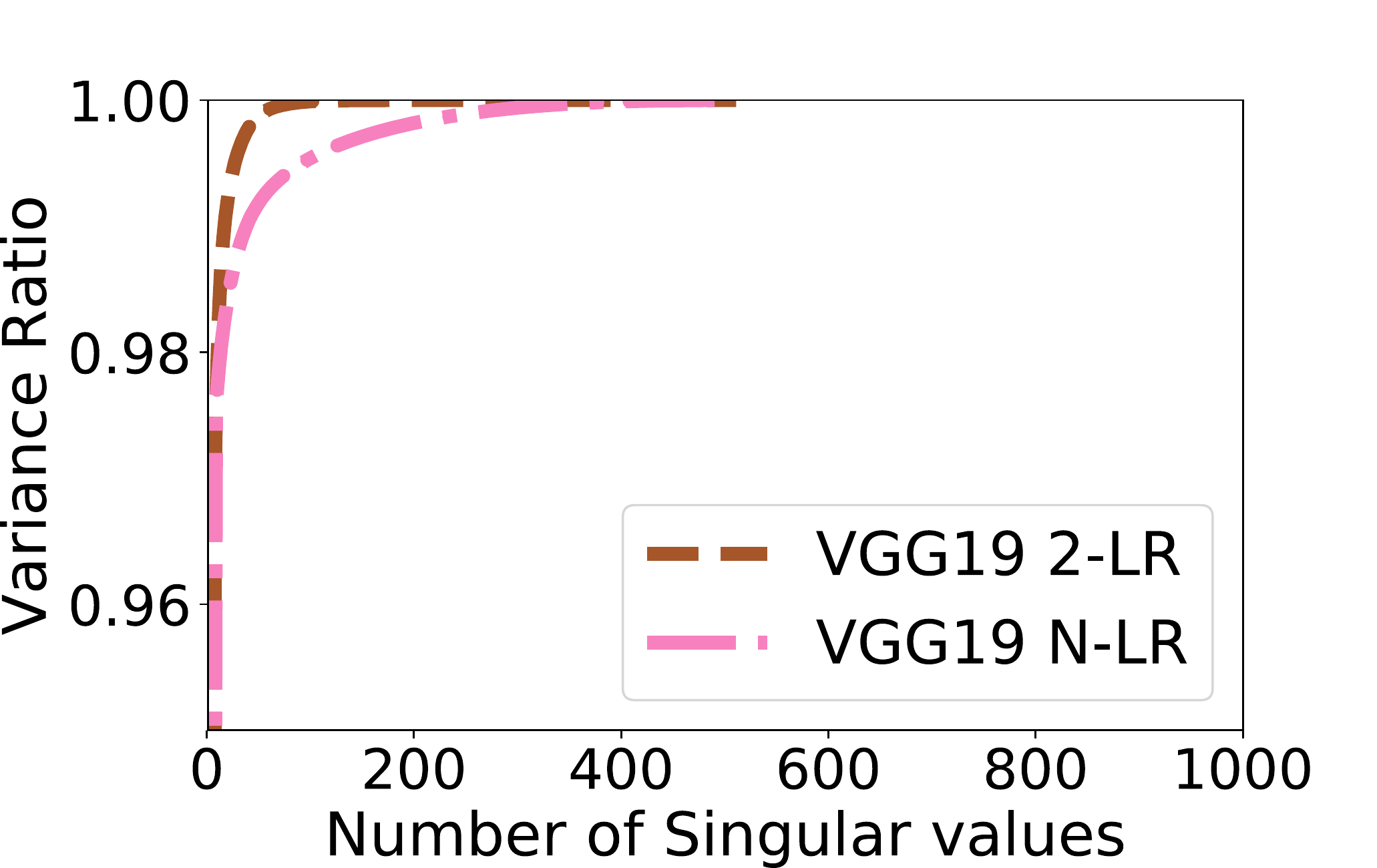
          \caption{Activations before third last FC layer}
          \label{fig:var_2_vgg}
        \end{subfigure}\caption{Variance Ratio captured by varying
number of Singular Values in VGG19 trained on CIFAR10.}
        \label{fig:var_ratio_plots_vgg}
      \end{figure}

In Figure~\ref{fig:var_ratio_plots_svhn}, we plot the variance ratio of
representations obtained before and after the last resnet block of a ResNet18 models trained on SVHN. The LR models show a better low rank structure than N-LR models and
is consistent with experiments on CIFAR10.

\begin{figure}[h!]
   \begin{subfigure}[t]{0.45\linewidth}
          \centering
          \def\svgwidth{0.95\columnwidth}
          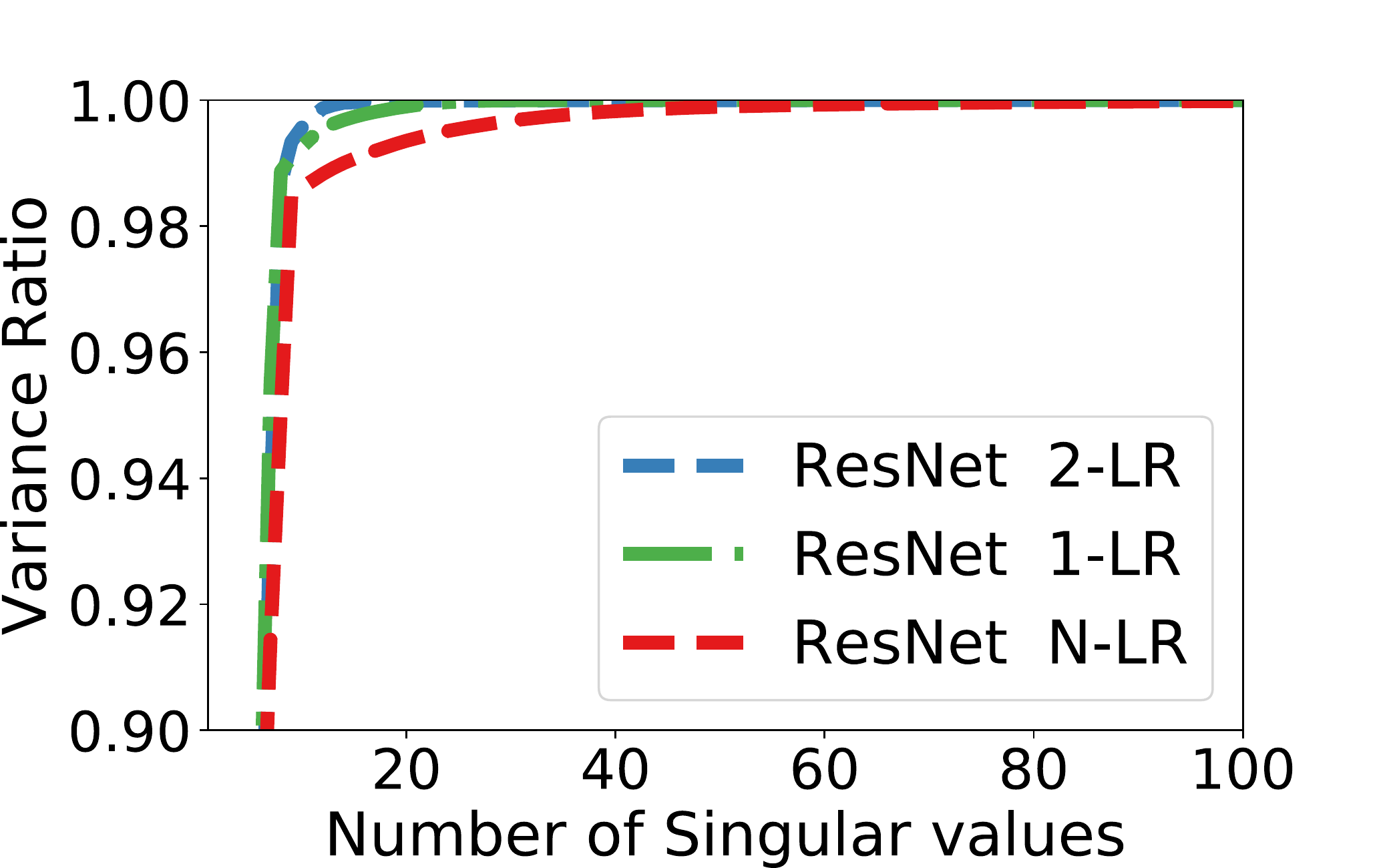
          \caption{SVHN: Activations after last ResNet block}
          \label{fig:var_3}
        \end{subfigure}\hfill
        \begin{subfigure}[t]{0.45\linewidth}
          \centering
          \def\svgwidth{0.95\columnwidth}
          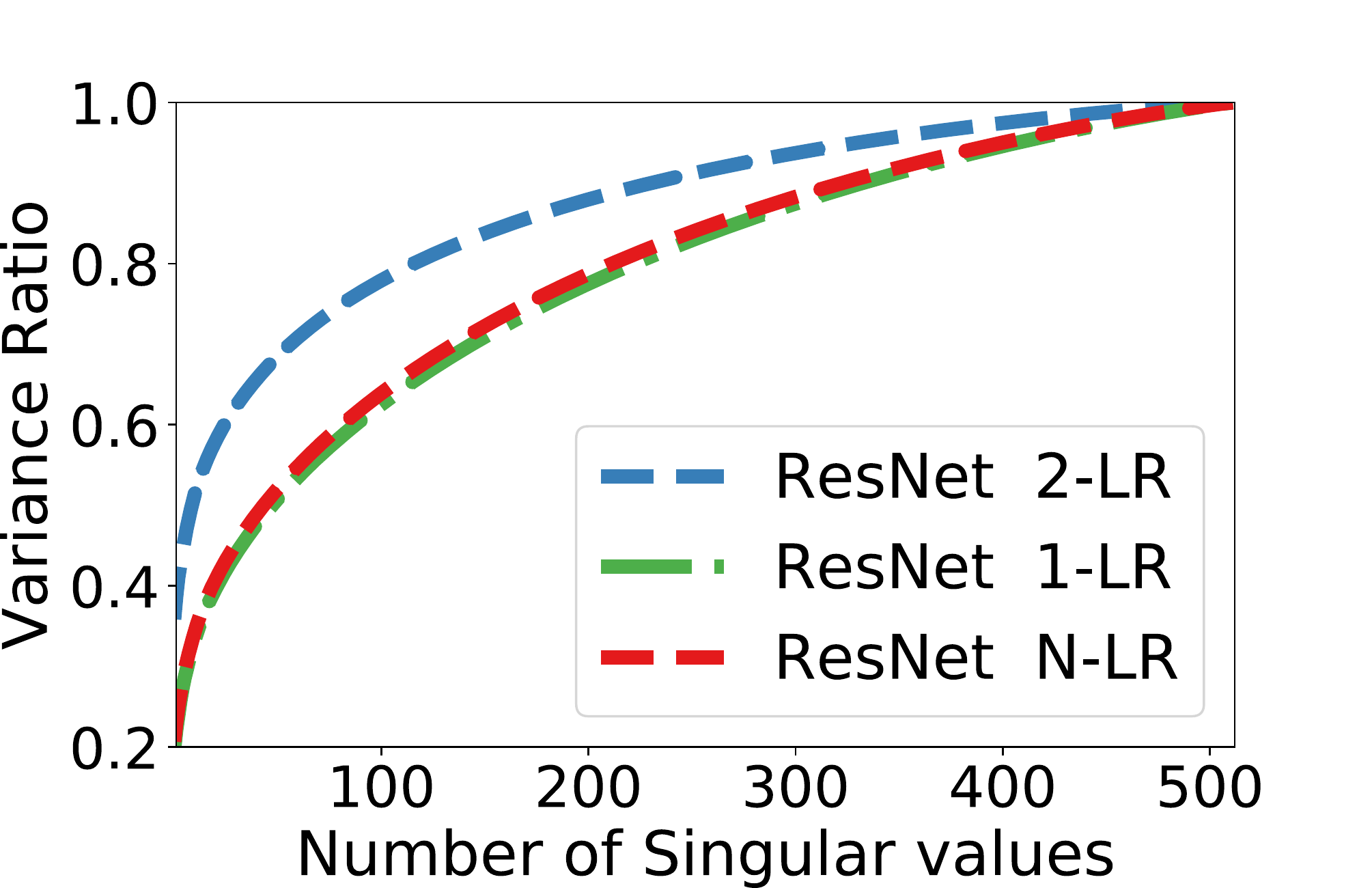
          \caption{SVHN: Activations before last ResNet block}
          \label{fig:var_4}
        \end{subfigure}\caption{Variance Ratio captured by varying
number of Singular Values in ResNet18 trained on SVHN.}
        \label{fig:var_ratio_plots_svhn}
      \end{figure}

\subsection{ Additional Classification Experiments for low dimensional embeddings}
\label{sec:show-validity-low}

In the first experiment, reported in Table~\ref{tbl:low_dim_emb_pred},
we trained two ResNet-50 hybrid max-margin models -with and without the \lr
respectively- on the 20 super-classes of CIFAR-100. As our objective
here is to see if the embeddings and their low dimensional projections
could be effectively used for discriminative tasks, we used PCA, with
standard pre-processing of scaling the input, to project the
embeddings onto a low dimensional space and then trained a linear
maximum margin classifier on it.

The experiments in~Table~\ref{tbl:low_dim_lyr_rem} were run with
ResNet-18 on CIFAR-10. Two ResNet-18 max-margin classifiers - with and
without the \lr respectively- were trained on CIFAR-10. The
representations were obtained from before the fourth ResNet block and
had a dimension of 16,384. Similar to the previous experiment, we used
PCA, with standard pre-processing, to obtain low dimensional
projections and then trained a linear max margin classifiers on it. The test set of CIFAR10 was converted using the same
ResNet models and projected using the same PCA vectors and then the accuracy of the linear classifier on it is reported
in~Table~\ref{tbl:low_dim_lyr_rem} for varying target dimensions of PCA projections.

\begin{table}[h]\centering\small
  \begin{tabular}{c@{\quad}c@{\enskip}c@{\enskip}c@{}}
    \toprule
   \textbf{V19} & Dim & Acc($\%$) \\ \hline 
   2LR & 512  & $\mathbf{89.8}$ \\ 
  NLR & 512  & $89.7$ \\ \hline
  2LR & 20  & $\mathbf{89.85}$ \\ 
  NLR & 20 &   $89.78$\\ \hline
  2LR & 10 & $\mathbf{89.79}$ \\ 
  NLR & 10 &  $89.65$\\ \hline
  \end{tabular}
  \caption{Representation from before the third last FC layer of a VGG19 trained on CIFAR-10.}
  \label{tbl:vgg_low_dim_lyr_rem}
\end{table}

Similar experiments were conducted with VGG19 on CIFAR10 and the
results are reported in Table~\ref{tbl:vgg_low_dim_lyr_rem}. As expected, due to the smaller size of the representation space, the difference in accuracy here is less stark than the
case of ResNet. This is because of two reasons - 1. The dimension of
the activation layer in ResNet before the last ResNet block is 16,384
whereas the activations before the third last FC layer is only
512. 2. Figure~\ref{fig:var_ratio_plots_vgg} shows that the difference
in the variance ratio between the LR and the N-LR network is much
smaller as compared to Figure~\ref{fig:var_ratio_plots} for ResNets.

\subsection{Class wise variance}
\label{sec:classwise-variance}

\begin{figure}[t]
\begin{subfigure}[t]{0.43\linewidth} \centering
\def\svgwidth{0.99\columnwidth}
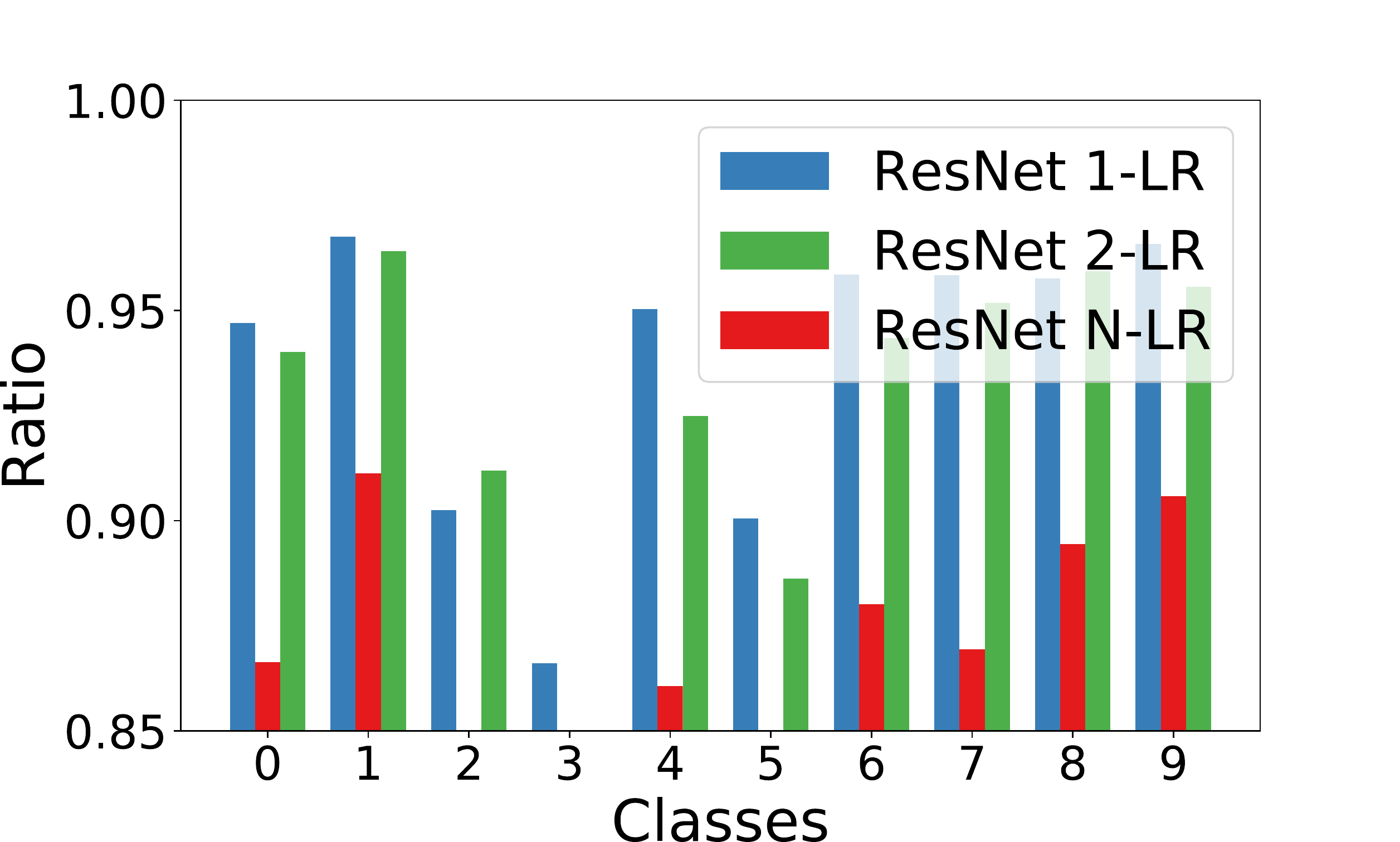
   \caption{$512$ dimensional activations from after last ResNet
block.}
   \label{fig:class_wise_var_1}
\end{subfigure}\qquad\qquad
\begin{subfigure}[t]{0.43\linewidth} \centering
\def\svgwidth{0.99\columnwidth}
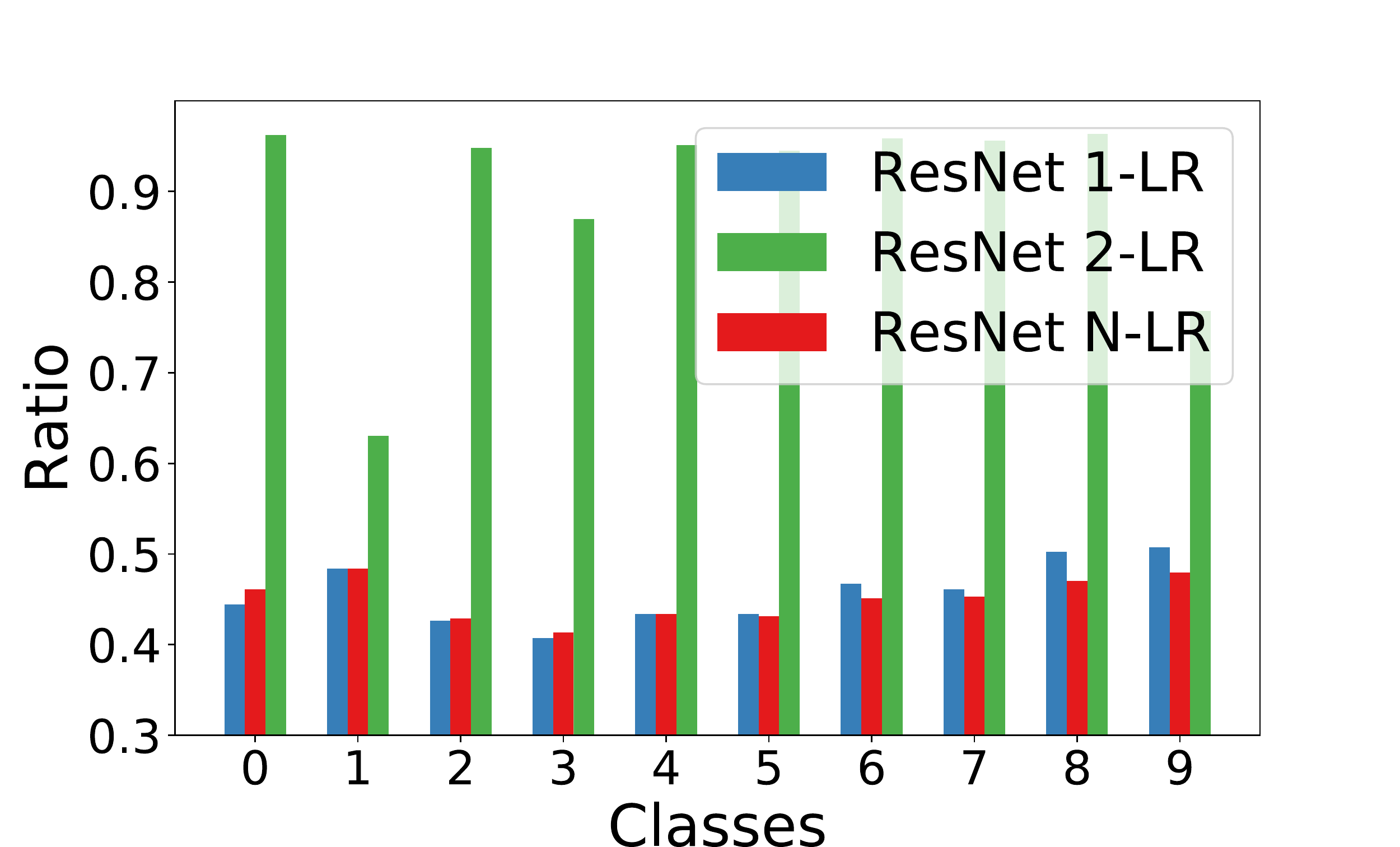
   \caption{$16k$ dimensional activations from before last ResNet
block.}
   \label{fig:class_wise_var_10}
\end{subfigure}
\caption{Class wise variance ratio of one singular values for the
activations before the last ResNet block.}
\end{figure}

In this experiment, we plot the variance ratio captured by the first
singular value~(i.e. the inverse stable rank) for embeddings of
examples restricted to individual
classes. Figure~\ref{fig:class_wise_var_1} shows the variance ratio
captured by the largest singular value for the activations before the
last FC layer while Figure~\ref{fig:class_wise_var_10} shows the
variance ratio captured by the largest singular value for the
activations before the last ResNet block. These experiments give us
some idea about the extent to which the set of basis vectors assigned
to individual classes are intersecting.Figure~\ref{fig:var_1} shows
that the rank of the entire activation matrix is almost
10. Figure~\ref{fig:class_wise_var_1} shows that the first singular
vector captures a huge portion of the variance of the restriction of
the activation matrix on each individual classes. It, thus, gives us
an intuition that, by sub-additivity of rank, the sets of basis
vectors explaining the activations belonging to each individual
classes are less intersecting than it is in the case of the model
trained without the \lr.

\subsection{Clusters of low dimensional embeddings}
\label{sec:clust-low-dimens}

\begin{figure}[htb]%
  \begin{center}
\includegraphics[width=0.3\linewidth]{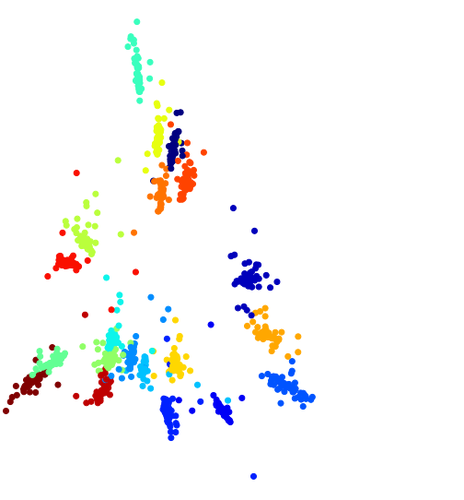}\hspace{20pt}
\includegraphics[width=0.3\linewidth]{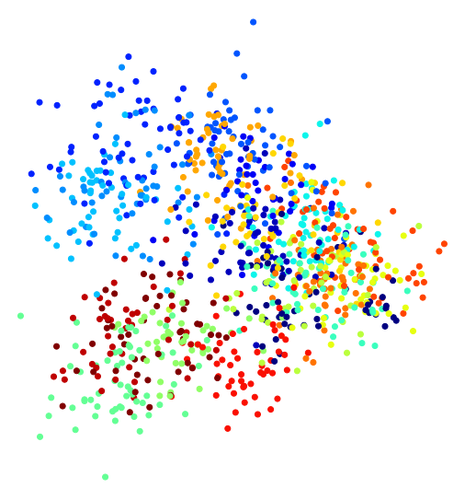}
    \end{center}
  		\caption{2-D PCA projection of representations from
ResNet50 trained on coarse labels of CIFAR 100 with~(left) and
without~(right) low-rank constraints, colored according to the
original 20  coarse labels.}
  \label{fig:coarse_lbl_PCA}
\end{figure}

Figure \ref{fig:coarse_lbl_PCA} shows the two dimensional
projections of the 2048 dimensional embeddings obtained from
ResNet-50-LR and ResNet-50-N-LR. The coloring is done according to the
coarse labels of the input. We can see that the clusters are more
separable in the case of the model with \lr than the model without,
which gives some insight into why a max-margin classifier performs
better for the LR model than the N-LR model. Thus, the representations
of the low rank model are more discriminative in the sense that for
the low rank representations there are low dimensional linear
classifiers that can classify the dataset with a higher margin than
the vanilla models.

\section{Adversarial Attacks}

\subsection{Types of Attacks}
\label{sec:types-attacks}

Here, $\vx_d$ refers to an example from the data distribution and $\vx_a$ the adversarially perturbed version of $\vx_d$. For vectors $\vz$ and $\vx$, let $\clip{\vz}$ denote the element-wise clipping of $\vz$, with $z_i$ clipped to the range $\bs{x_i-\epsilon, x_i + \epsilon}$.

\begin{itemize}
\item \ifsgm - The Fast Sign Gradient Method
(\fsgm)~\cite{goodfellow2014explaining} was proposed as the existing
methods (\citet{szegedy2013intriguing}) of the time were slow. \fsgm
tries to maximize the loss function by perturbing the input
slightly. Iterative Fast Sign Gradient method (\ifsgm) is a simple
extension of \fsgm that follows the following simple iterative step.
  \begin{align}
    \label{eq:ifsgm} &\vx_a^0 = \vx_d,\\ &\vx_{a}^{n+1}= \clip{\vx_a^n
+ \alpha \cdot\sgn{\nabla_{\vx_a^n}\cL(\vx_a^n, \vy_t)}}\nonumber
  \end{align}
  
\item \ill - \ifsgm is an untargeted attack. Iterative less likely
fast sign gradient method (\ill)~\citep{kurakin2016adversarial} is a
way to choose the target label wisely. Consider $\mathbb{P}_{M}(\vy
\vert \vx)$ to be the probability assigned to the label $\vy$, for the
example $\vx$, by the model $M$.  In this attack, the target is set as
$\vy_t^n = \argmin_{\vy\in \mathcal{Y}}\mathbb{P}_{M}(\vy \vert
\vx^n)$ and the following iterative update steps are performed.
  \begin{align}
    \label{eq:ill_upd} &\vx_a^0 = \vx_d,\\ &\vx_{a}^{n+1}=
\clip{\vx_a^n - \alpha \cdot\sgn{\nabla_{\vx_a^n}\mathcal{L}(\vx_a^n,
\vy_t^n)}}\nonumber
  \end{align} Intuitively, this method picks the least likely class in
each iteration and then tries to increase the probability of
predicting that class. In both of these methods, $\alpha$ was set to
$1$ as was done in ~\citet{kurakin2016adversarial}.
\item \deepfool - ~\citet{mosaavi2016} describes the \deepfool
procedure to find the optimal (smallest) perturbation for the input
$\vx$ that can fool the classifier. In the case of affine classifiers,
\deepfool finds the closest hyper-plane of the boundary of the region
where the classifier returns the same label as $\vx$ and then adds a
small perturbation to cross the hyper-plane in that direction.

  As deep net classifiers are not affine, the partitions of the input
space where the classifier outputs the same label are not not
polyhedrons. Hence, the algorithm takes an iterative approach.
Specifically, the algorithm assumes a linerization of the classifier
around $\vx$ to approximate the polyhedron and then it takes a step
towards the closest boundary. For a more detailed explanation please
look at ~\citet{mosaavi2016}.
\end{itemize}
\subsection{Further Experiments on Adversarial Robustness}
\label{sec:more-exp-adv-robust}
In Table~\ref{tab:adv-robust-cifar100-fine}, we show the adversarial
test accuracy for varying perturbation budgets for ResNet50 trained on
the fine labels of CIFAR100. Low Rank models, not have the best
adversarial test accuracies but also the best natural test accuracies.

\begin{table*}[!htb]\centering\footnotesize
\begin{tabular}{llllllllllllc}
\toprule
&&&\multicolumn{8}{c}{Adversarial Test Accuracy($\%)$}           &
                                                                     Clean
                                                                    Test Accuracy ($\%$)    \\ \midrule
\multicolumn{3}{c}{$L_\infty$ radius}               &   \multicolumn{2}{c}{$8/255$}      &   \multicolumn{2}{c}{$10/255$}               &   \multicolumn{2}{c}{$16/255$}   &\multicolumn{2}{c}{$20/255$}&       \\
\multicolumn{3}{c}{Attack iterations}            & $7$     & $20$       &$7$    &$20$           &$7$    &$20$      &$7$    &$20$   &       \\\toprule
\multirow{2}{*}{\rotatebox[origin=c]{0}{White Box}}    & \multirow{2}{*}{R50}       &N-LR& 27.8 & 21.8  & 25.2 & 17.7  & 21.1 & 17.7   & 19.4 & 7.6  & 77.2 \\
                      &                       &1-LR&  $\mathbf{28.8}$ & $\mathbf{23.6}$   & $\mathbf{26.8}$ & $\mathbf{21.4}$   & $\mathbf{24.4}$ & $\mathbf{17.8}$  & $\mathbf{23.5}$ & $\mathbf{16.5}$  & $\mathbf{77.7}$ \\\addlinespace
\multirow{1}{*}{\rotatebox[origin=c-10]{0}{ Black Box}}     & R50       &1-LR & 38.3 & 32.8   & 34.3 & 26.4   & 28.9 & 15.0   & 27.0 & 11.7 &   -   \\\bottomrule
\end{tabular}\caption{Adversarial Test Accuracy against a
  $\ell_\infty$ constrained PGD adversary with the $\ell_\infty$
  radius bounded by $\epsilon$ and the number of attack steps bounded
  by $\tau$. R50 and R18 denotes ResNet50 and ResNet18
  respectively. C10 and C100 refer to CIFAR10 and CIFAR100~(Coarse
  labels) respectively.}\label{tab:adv-robust-cifar100-fine}
\end{table*}

In Figure~\ref{fig:adv_pert_svhn_vgg}, we compare the change in the
adversarial test accuracy with respect to the amount of adversarial
noise added for ResNet18 models trained on SVHN~\ref{fig:adv-r18-svhn} and VGG19 on
CIFAR10~\ref{fig:adv-vgg-c10}.
\begin{figure}[t]
  \begin{subfigure}[c]{0.15\linewidth} \centering
\def\svgwidth{0.99\columnwidth} %
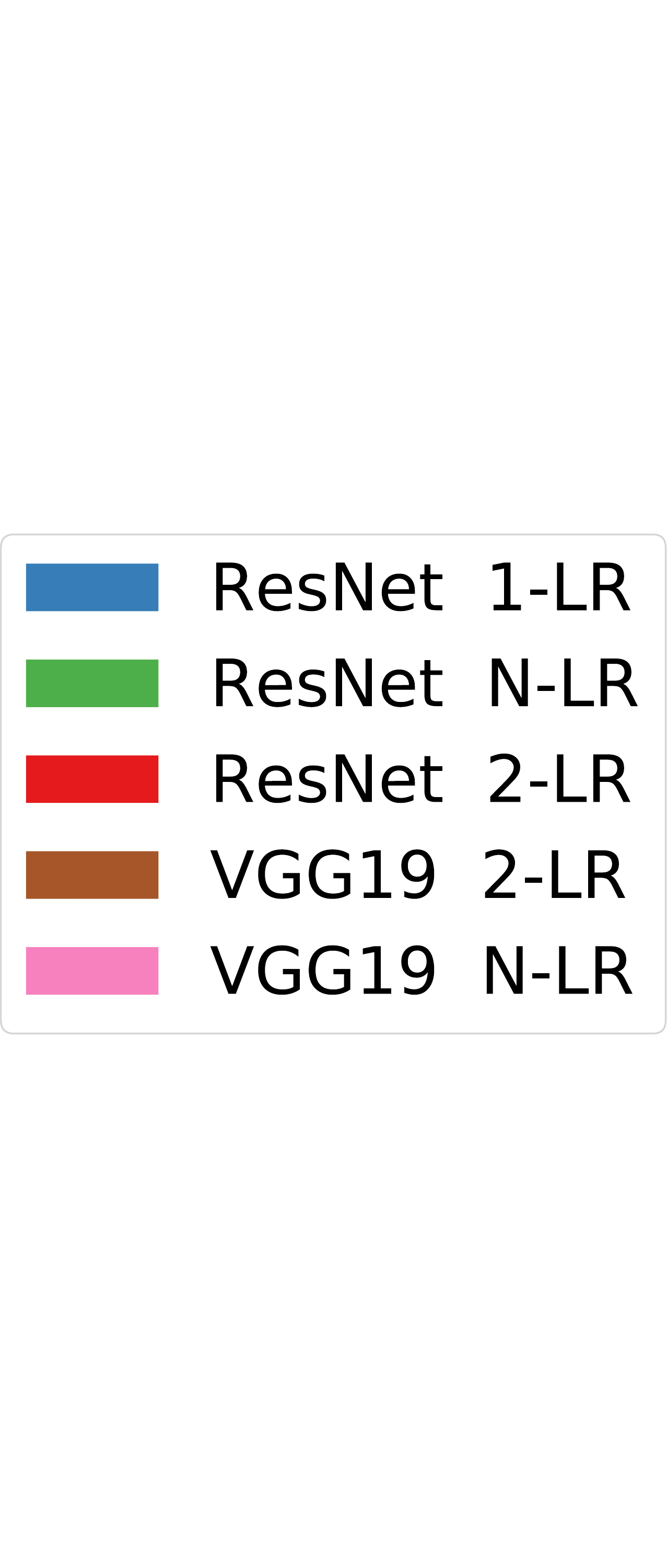
\end{subfigure}\hfill
\begin{subfigure}[c]{0.4\linewidth} \centering
\def\svgwidth{0.99\columnwidth}
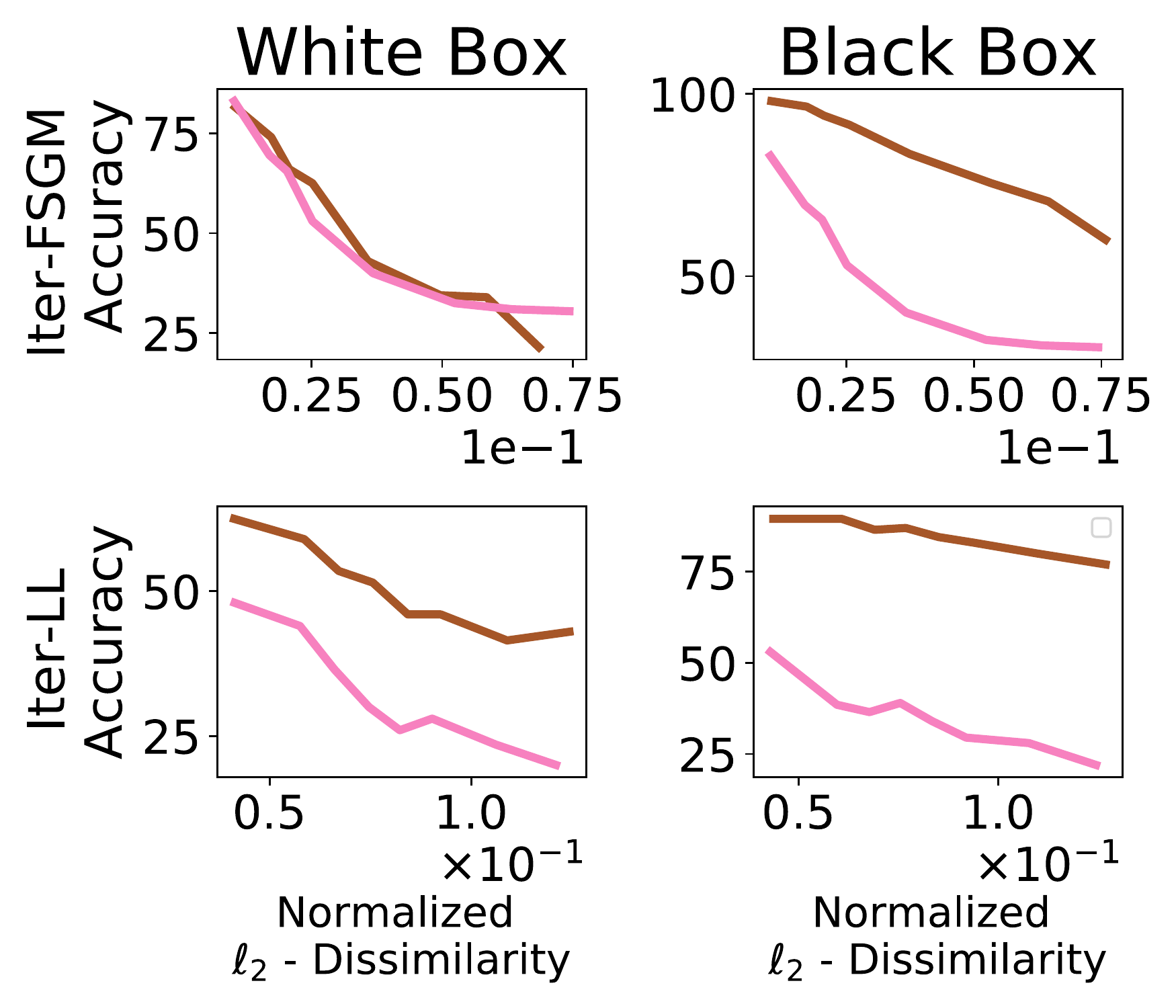\caption{VGG19-CIFAR10}\label{fig:adv-vgg-c10}
\end{subfigure}
\begin{subfigure}[c]{0.4\linewidth} \centering
\def\svgwidth{0.99\columnwidth}
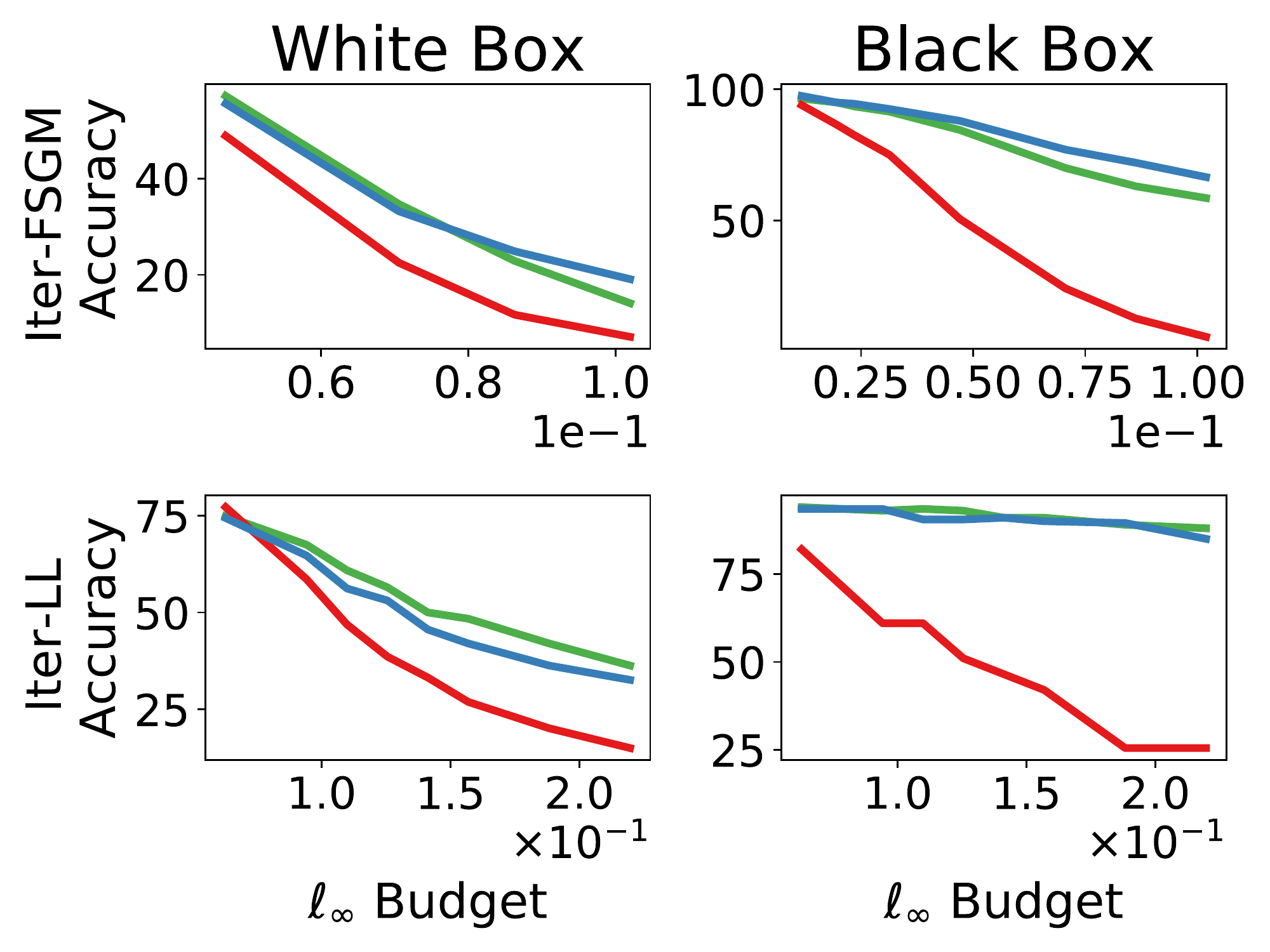\caption{ResNet18-SVHN}\label{fig:adv-r18-svhn}
\end{subfigure}
\caption[Change in
perturbation]{Adversarial accuracy plotted against magnitude of
  perturbation. }
\label{fig:adv_pert_svhn_vgg}
	\end{figure}

\subsection{Minimum Perturbation for a successful Attack}
\label{sec:mimim-pert-succ}

Table~\ref{tab:adv_rob_pert} lists the minimum perturbation required
to fool the classifier under the particular attack scheme. For \ifsgm
and \ill, there are essentially three hyper-parameters($t, \alpha,
\epsilon$) in the experiments as can be seen below.

\ifsgm
\begin{align}
  \label{eq:ifsgm_app} \textbf{Repeat $t$ times}&\\\nonumber &\vx_a^0
= \vx_d,\\\nonumber &\vx_{a}^{n+1}= \clip{\vx_a^n + \alpha
\cdot\sgn{\nabla_{\vx_a^n}\cL(\vx_a^n, \vy_t)}}
\end{align} \ill
\begin{align}
    \label{eq:ill_upd_app} \textbf{Repeat $t$ times}&\\\nonumber
&\vx_a^0 = \vx_d,\\\nonumber &\vx_{a}^{n+1}= \clip{\vx_a^n - \alpha
\cdot\sgn{\nabla_{\vx_a^n}\mathcal{L}(\vx_a^n, \vy_t^n)}}
\end{align}

Following the convention of ~\citet{kurakin2016adversarial}, we set
$\alpha = 1$. We tuned the hyper-parameter $\epsilon$ function to
obtain the smallest $\epsilon$ that resulted in over $99\%$
misclassification accuracy for some $t$ and then repeated the
experiments until such a $t$ was achieved. Finally $\rho$ was
calculated.

Algorithm 2 in ~\citet{mosaavi2016} gives details about the DeepFool
algorithm for multi-class classifiers. The algorithm returns the
minimum perturbation $r(\vx)$ required to make the classifier
misclassify the instance $\vx$.  The $L_2$ dissimilarity is obtained
by calculating $ \rho = \frac{r(\vx)}{\norm{\vx}_2}$

For the benefit of reproducibility of experiments, we list the values
of $\epsilon$ for \ill and \ifsgm in Table~\ref{tab:adv_rob_pert_eps}
corresponding to the values in Table~\ref{tab:adv_rob_pert} . The values of the perturbation budget also show that the
minimum perturbation required for $99\%$ mis-classification is much higher for LR models than N-LR models.
For
\deepfool, we used the publicly available code~\footnote{
\url{https://github.com/LTS4/DeepFool/blob/master/Python/deepfool.py}}.

\begin{table}[h!]  \centering
  \begin{tabular}{|c|c|c|c|} \hline&Model&$\epsilon$[\ill]&$\epsilon$
[\ifsgm]\\\hline \multirow{3}{*}{White Box}&ResNet
\textbf{2-LR}&$0.04$&$0.02$\\\cline{2-4} &ResNet
\textbf{1-LR}&$0.06$&$0.01$\\\cline{2-4} &ResNet
\textbf{N-LR}&$0.01$&$0.01$\\\hline \multirow{2}{*}{Black Box}&ResNet
\textbf{1-LR}&$0.08$&$0.01$\\\cline{2-4} &ResNet
\textbf{2-LR}&$0.1$&$0.01$\\\hline
  \end{tabular} \vspace{1em}
  \caption{Value for $\epsilon$ required for Adversarial
Misclassification corresponding to Table~\ref{tab:adv_rob_pert}.}
  \label{tab:adv_rob_pert_eps}
\end{table}

 \begin{figure}[h!]  \centering
\includegraphics[width=4in,height=2in]{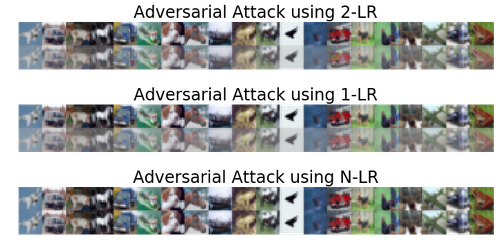}
  \caption{For each model, original images are on the top row and the
images generated by DeepFool are below.}
  \label{fig:adv_img}
\end{figure}

We also look at some of the adversarial images generated by DeepFool
in Figure~\ref{fig:adv_img}. We observe that it is immediately clear
that the adversarial images are different from the original images in
the case of LR models whereas it is not so apparent in the case of
N-LR models.

\subsection{Unstability of Adversarial Attacks}
\label{sec:fixed-number-steps}

 An interesting observation is that the
values of $\rho$ in Table~\ref{tab:adv_rob_pert} are lower than
those in Figure~\ref{fig:adv_pert} though the attacks have a higher rate of
success. To explain this behaviour, we show empirical evidence that an attack that adds noise for a fixed number of
steps~\citep{kurakin2016adversarial,kurakin2016} to the input is
significantly weaker than one that stops on successful
misclassification. %
The essential difference between the attacks in
Figure~\ref{fig:adv_pert} and Table~\ref{tab:adv_rob_pert} is in the
number of iterations for which the updates (Step~\ref{eq:ifsgm_app}
and Step~\ref{eq:ill_upd_app}) are executed. In
Figure~\ref{fig:adv_pert}, the step is executed $t$ times whereas in
Table~\ref{tab:adv_rob_pert}, the updates are executed until the
classifier makes a mistake.
\begin{figure}[h!]
  \begin{subfigure}[c]{0.4\linewidth} \centering
\def\svgwidth{0.99\columnwidth} 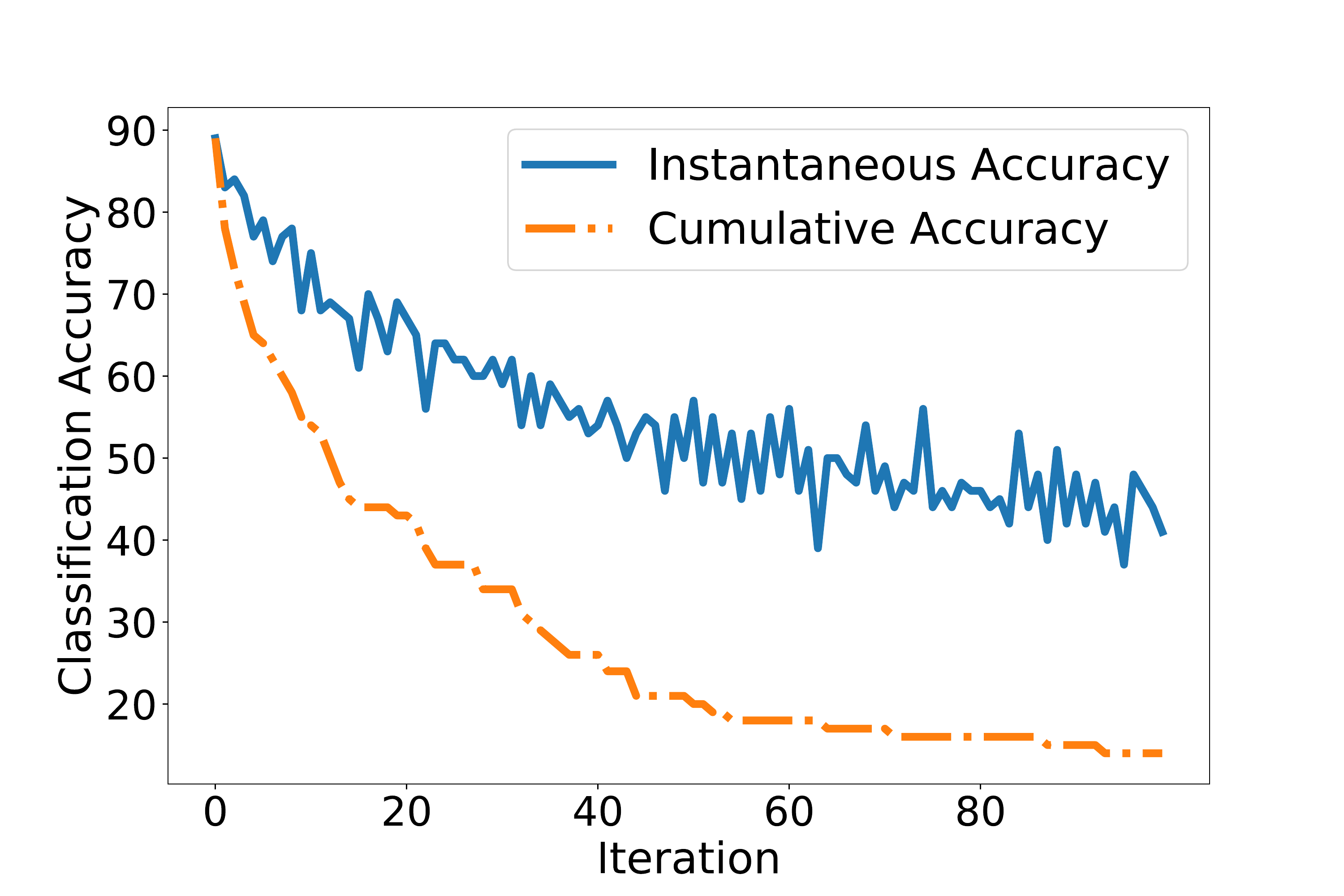
			 \caption{\ill \label{sfig:ll_inst_cum_adv}}
  \end{subfigure}\hfill
     \begin{subfigure}[c]{0.4\linewidth} \centering
\def\svgwidth{0.99\columnwidth}
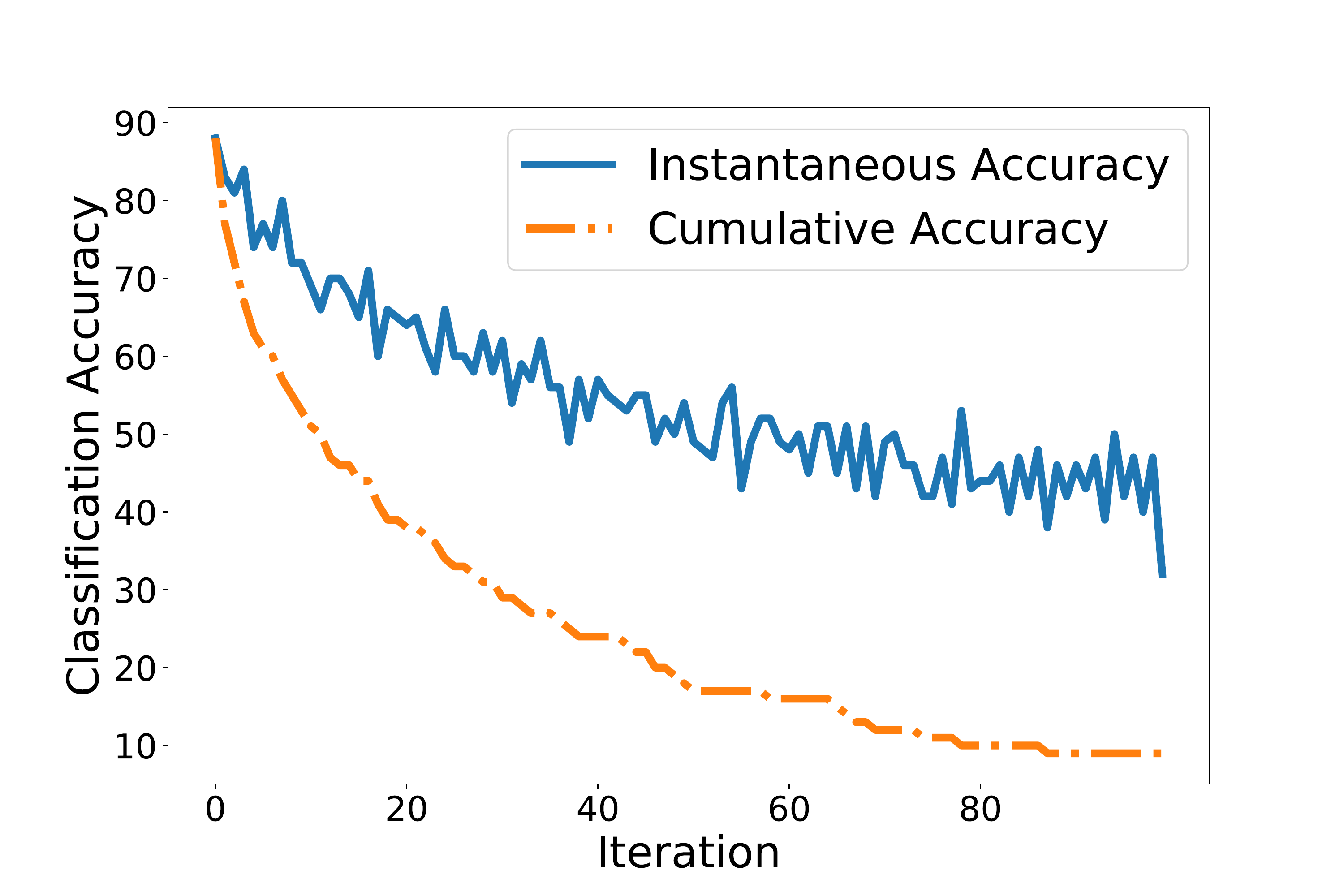
		  \caption{\ifsgm \label{sfig:fsgm_inst_cum_adv}}
        \end{subfigure}
         \caption{This shows that an adversarial example that has
successfully fooled the classifier in a previous step can be
classified correctly upon adding more
perturbation. Figure~\ref{sfig:ll_inst_cum_adv} and
~\ref{sfig:fsgm_inst_cum_adv} refers to the two attack schemes - \ill
and \ifsgm respectively.}
         \label{fig:adv_fix_step}
       \end{figure}
It would be natural to expect that once a classifier has misclassified
an example, adding more adversarial perturbation will not make
the classifier classify it correctly. However,
Figure~\ref{fig:adv_fix_step} suggests that a misclassified example
can be possibly classified correctly upon further addition of noise.

Let $y_{a}(\vx; k)$ be the label given to $\vx$ after adding
adversarial perturbation to $\vx$ for $k$ steps. We define
\emph{instantaneous accuracy} ($a_{\cI}(k)$) and \emph{cumulative
accuracy} ($a_{\cC}(k)$) as \[a_{\cI}(k) = 1 -
\dfrac{1}{m}\sum_{i=1}^m \cI_{0,1}\bc{y_{a}(\vx; k) \neq y_{a}(\vx;
0)};\quad a_{\cC}(k) =1 - \dfrac{1}{m}\sum_{i=1}^m \max_{1\le j\le
k}\bc{ \cI_{0,1} \bc{y_{a}(\vx; j) \neq y_{a}(\vx; 0)}} \]

In Figure~\ref{fig:adv_fix_step}, we see the \emph{instantaneous
accuracy} and the \emph{cumulative accuracy} for ResNet 1-LR where
$\alpha=0.01,\epsilon=0.1$ and $t$ is plotted in the x-axis. The
cumulative accuracy is by definition a non-increasing
sequence. However, surprisingly the instantaneous accuracy is not
monotonic and has a lower rate of decrease than the cumulative
accuracy. It also appears to stabilize at a value much higher than the
cumulative accuracy.

\subsection{Adversarial Attack on Maximum Margin Model}
\label{sec:advers-attack-maxim}

Here, we train max-margin classifiers on representations of images
obtained from different ResNet models~(similar to Section~\ref{sec:show-validity-low}) and see whether the
representations of adversarial images, that had successfully fooled
the ResNet model, can fool the max-margin classifier as well. We
train a variety of hybrid max-margin models  with ResNet18-1-LR,
ResNet18-2-LR, and ResNet18-N-LR along with black box versions of the
same. Then we generate adversarial examples for all three attacks (both black box and white box) on the trained ResNet
models~(not the hybrid models). Then we use these adversarial examples to attack the corresponding max-margin models and report the accuracy of the
max-margin models in Table~\ref{tab:max_margin_adv}.

To perform a fair comparison with the hybrid ResNet18-N-LR, it is essential
to add a similar amount of noise to generate the examples for the hyrbrid
ResNet18-N-LR as is added to hybrid ResNet18-1-LR. The adversarial
examples are hence generated by obtaining the gradient using
ResNet18-N-LR but stopping the iteration only when the adversarial
example could fool ResNet18-1-LR. This is, in-fact, the black box
attack on ResNet18-1-LR. As Table~\ref{tab:max_margin_adv} suggests,
the max-margin classifiers are not only more robust to adversarial examples in general but are es-specially more robust
when the representations come from LR models than N-LR models.

\section{Noise Cancellation Properties}
\label{sec:gen_bounds}

Here we plot a quantity called \textit{layer cushion}, first mentioned in~\citet{arora18b}, for various layers
in ResNet18~\textsf{1-LR}, ResNet18~\textsf{2-LR}, ResNet18~\textsf{N-LR} and a randomly
initialized ResNet. As suggested in~\citet{arora18b}, this quantity appears in the denominator in the generalization bound of the network
and has a positive correlation with the noise-cancellation property of the network. Thus a higher value of this quantity
can be used to justify the low sensitivity of the network to noise.

The motivation for these quantity is that if the “real” data
$\vec{x}$ is more aligned with the high singular values of the linear
transformations, the linear transformations are more robust to noise at
that point $\vec{x}$. It can roughly be thought of as the inverse of the
sensitivity of the transformation. It measures the ratio of the norm
of the actual output of the layer at $\vec{x}$ with the upper bound on
the norm of the output at $\vec{x}$. If this quantity is large for
most $\vec{x}$, it means that most of the signal is aligned with the
high singular values of the linear transformation, which means that
the transformation is more resilient to noise.
This partly explains why the network attenuates the adversarial noise
at the data points. For any layer $i$, the layer cushion is defined as the largest number
$\mu_i$ such that the following holds for all examples $\vec{x}\in
\cS$ where $\cS$ is the training set.

\[ \mu_i\norm{\vec{W}_i}_F\norm{\phi\br{\vec{x}_{i-1}}} \le
\norm{\vec{W}_i\phi\br{\vec{x}_{i-1}}}\] $\vec{W}_i$ is the weight
matrix of the $i^{\it{th}}$ layer, $\vec{x}_i$ is the pre-activation
of the layer and $\phi$ is the activation function. As observed
by~\citet{arora18b}, higher the value of $\mu_i$, better is the
generalization ability of the model. Here we plot a distribution of
the ratio for the examples in the dataset.

\subsection{Layer Cushion for ResNet18 on CIFAR10}
\label{sec:resnet18-cifar10}

 The following corresponds to the four ResNet blocks in ResNet18. Each block has two smaller
sub-blocks where each sub-block has two convolutional layers. The
value of layer cushion for these modules of each of these blocks are plotted
below. Note that only 2-LR shows an increased cushion in Layer 3
whereas both 1-LR and 2-LR have higher cushions in all other layers.

\begin{center}
  \begin{figure}[h!]
  \begin{subfigure}[c]{0.24\linewidth} \centering
\def\svgwidth{0.99\columnwidth} 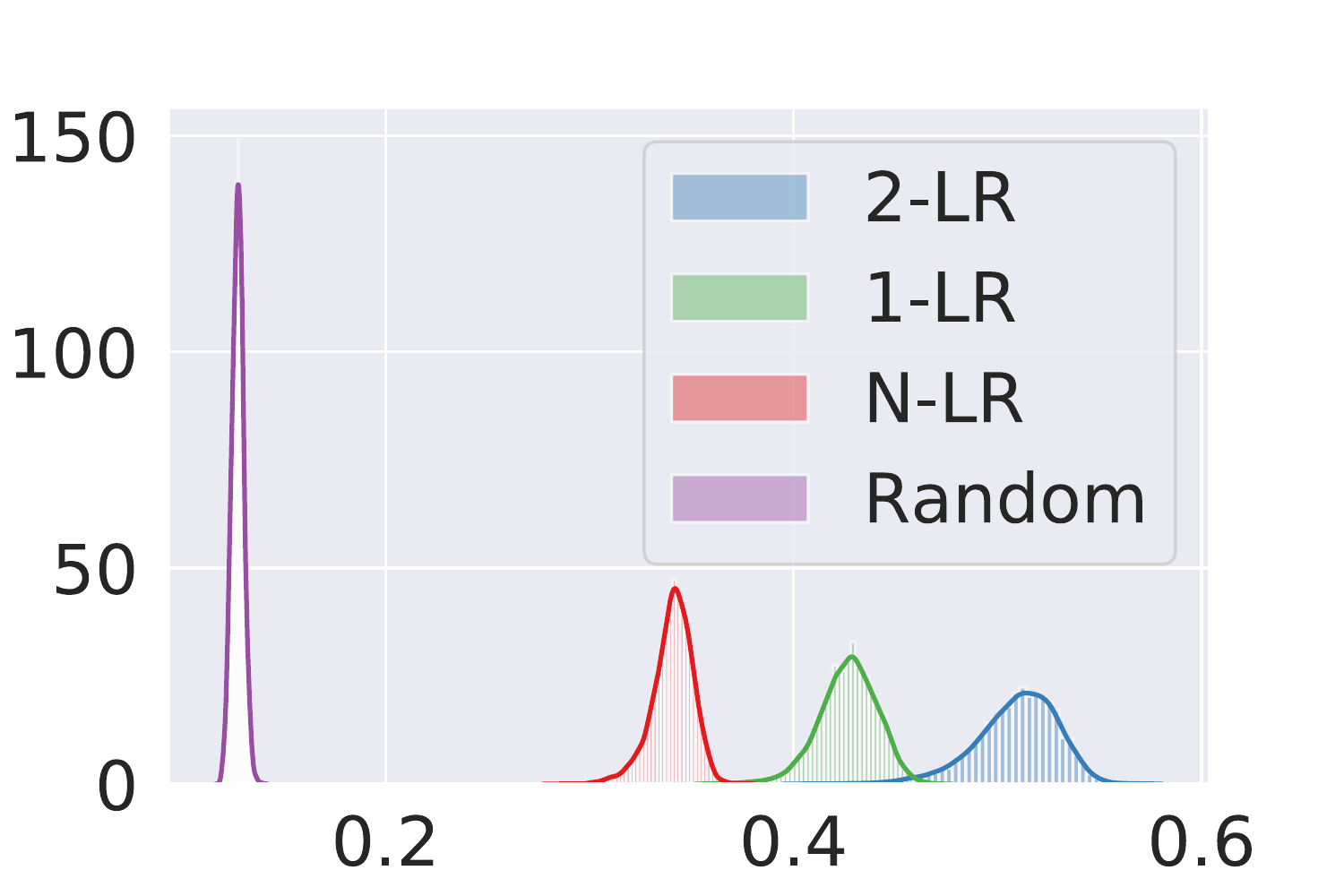
  \end{subfigure}
  \begin{subfigure}[c]{0.24\linewidth} \centering
\def\svgwidth{0.99\columnwidth} 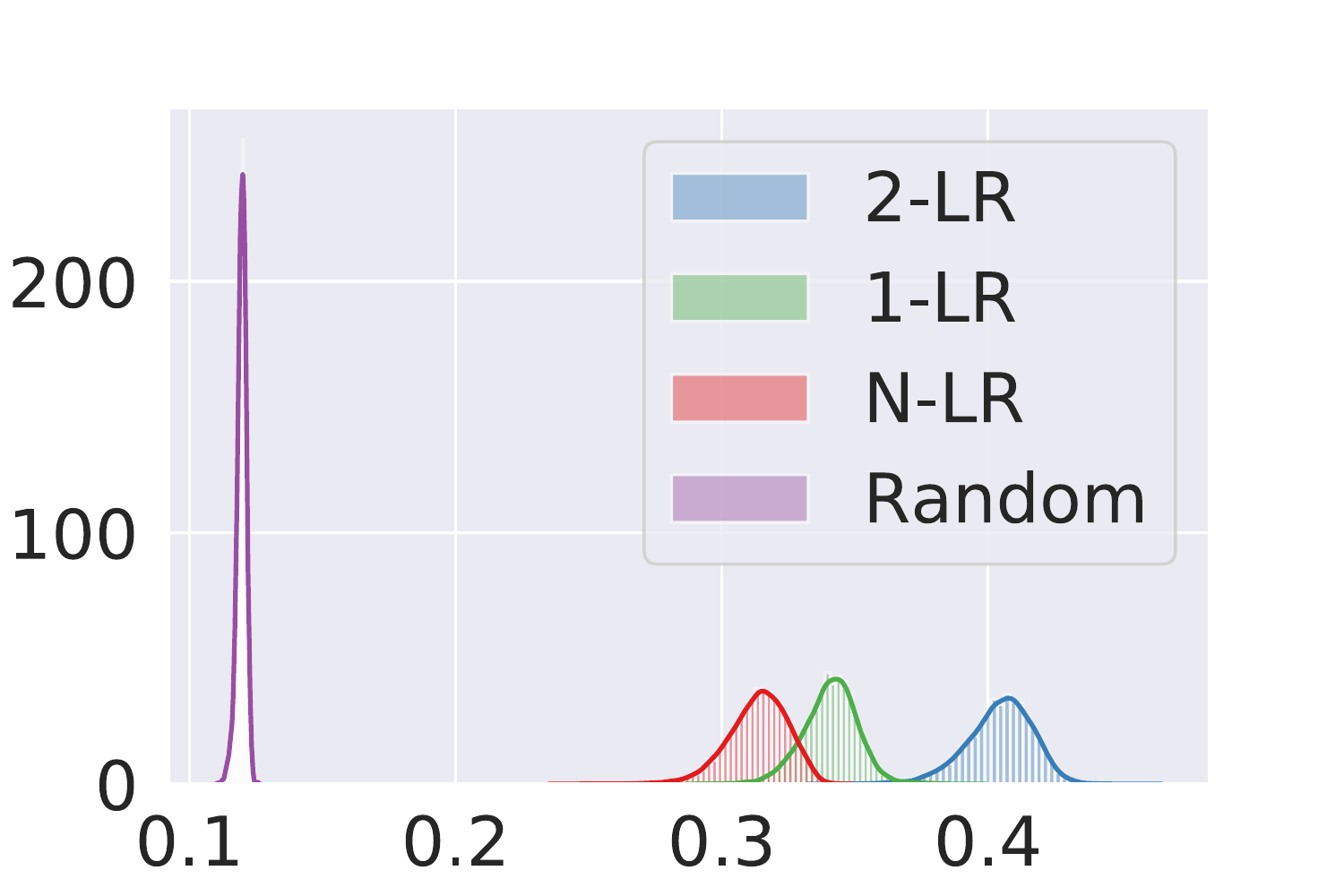
  \end{subfigure}
  \begin{subfigure}[c]{0.24\linewidth} \centering
\def\svgwidth{0.99\columnwidth} \input{./figs/spec_lyr1_b1_c1.pdf_tex}
  \end{subfigure}
  \begin{subfigure}[c]{0.24\linewidth} \centering
\def\svgwidth{0.99\columnwidth} 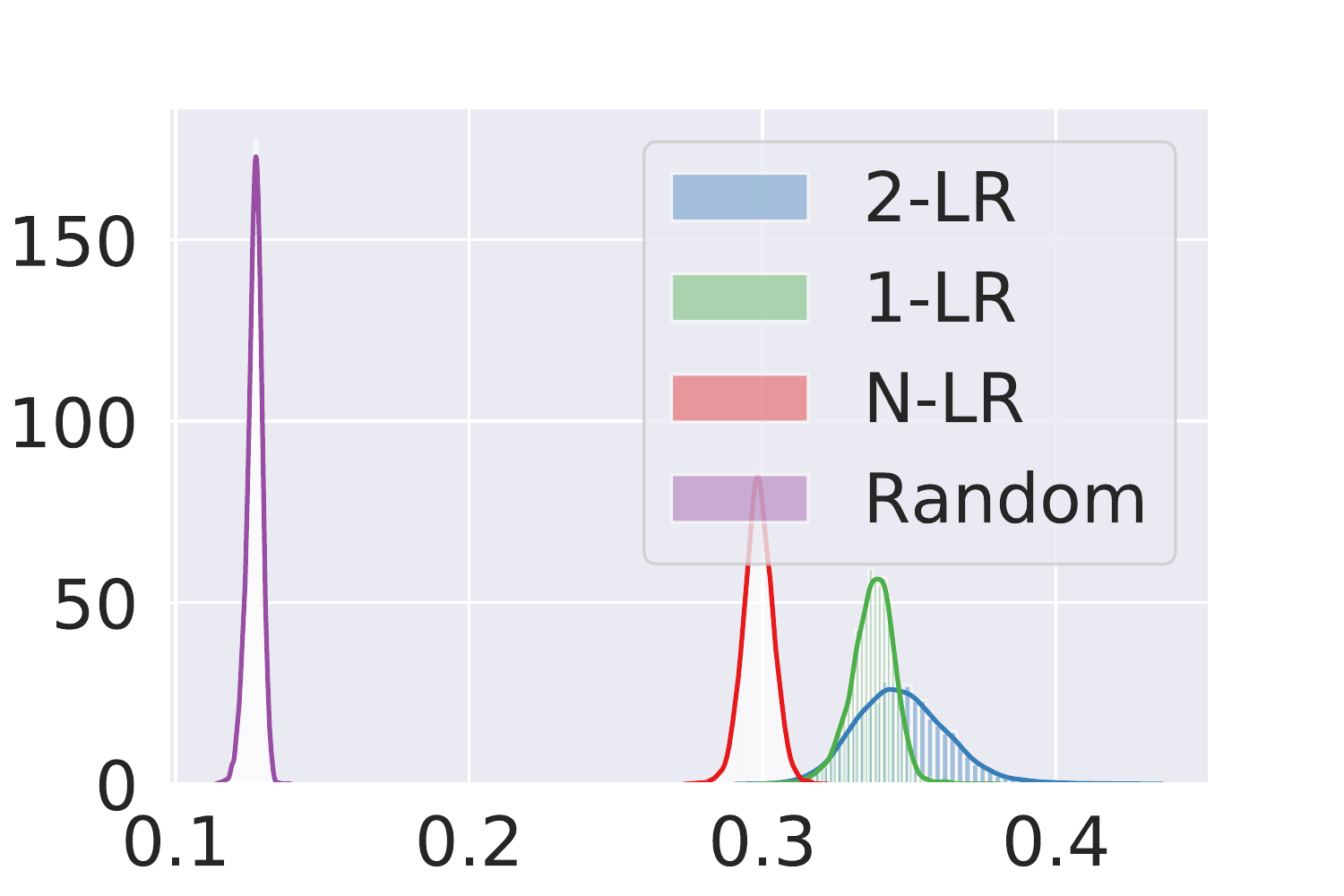
  \end{subfigure}
  \caption{Cushion of Layer 1}
  \label{fig:int_spec_lyr_cush}
\end{figure}
\end{center}

\begin{center}
  \begin{figure}[h!]
  \begin{subfigure}[c]{0.24\linewidth} \centering
\def\svgwidth{0.99\columnwidth} 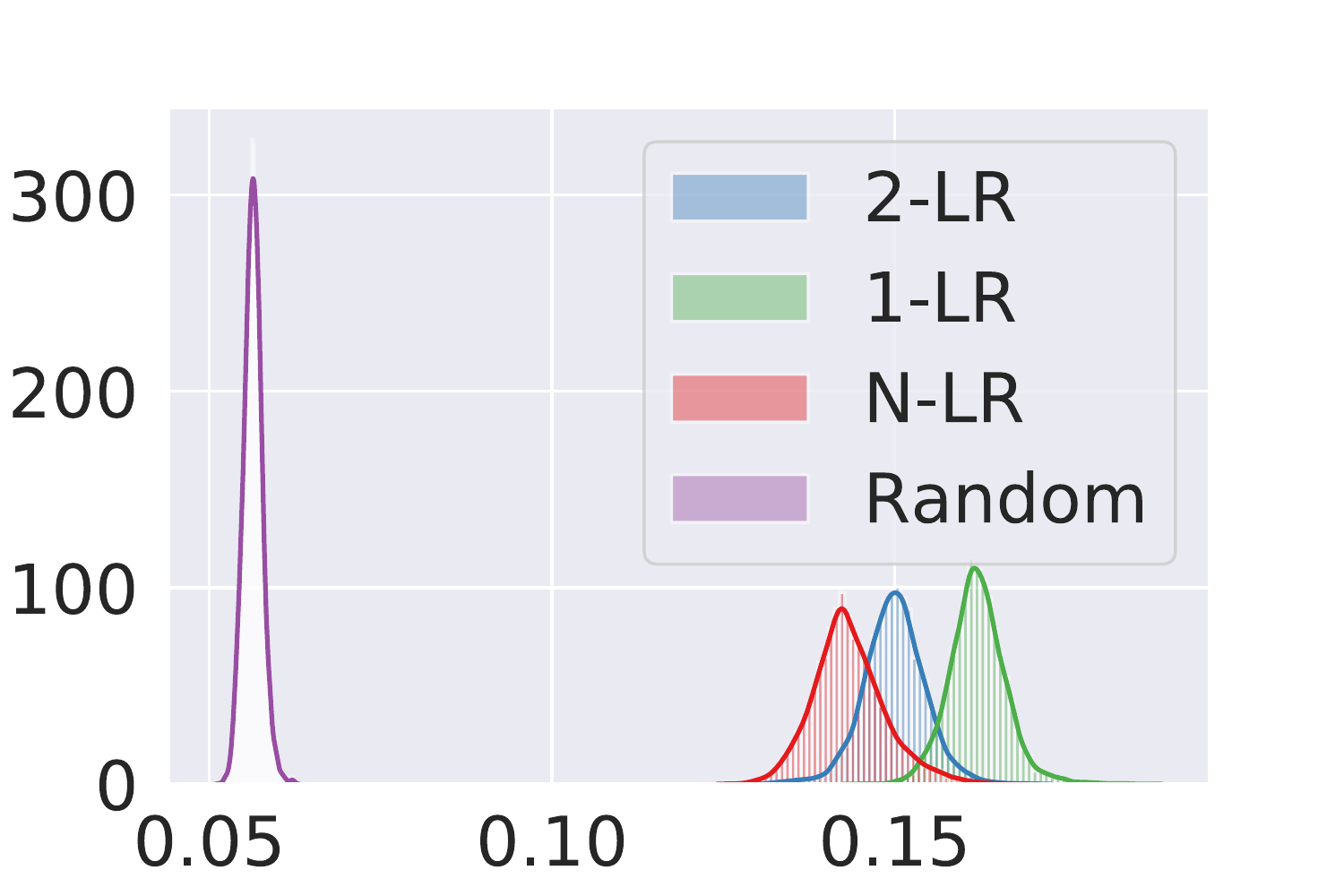
  \end{subfigure}
  \begin{subfigure}[c]{0.24\linewidth} \centering
\def\svgwidth{0.99\columnwidth} 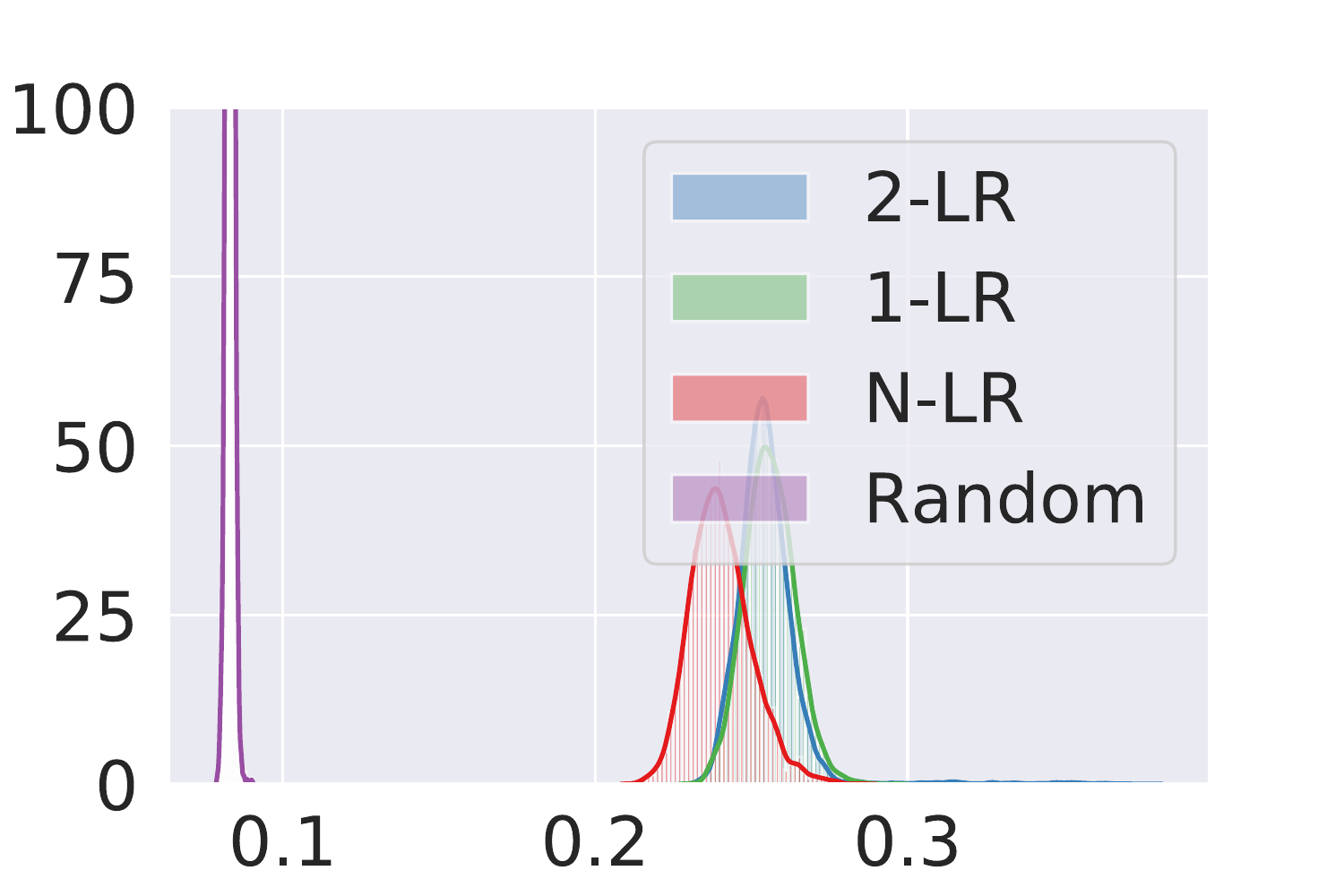
  \end{subfigure}
  \begin{subfigure}[c]{0.24\linewidth} \centering
\def\svgwidth{0.99\columnwidth} 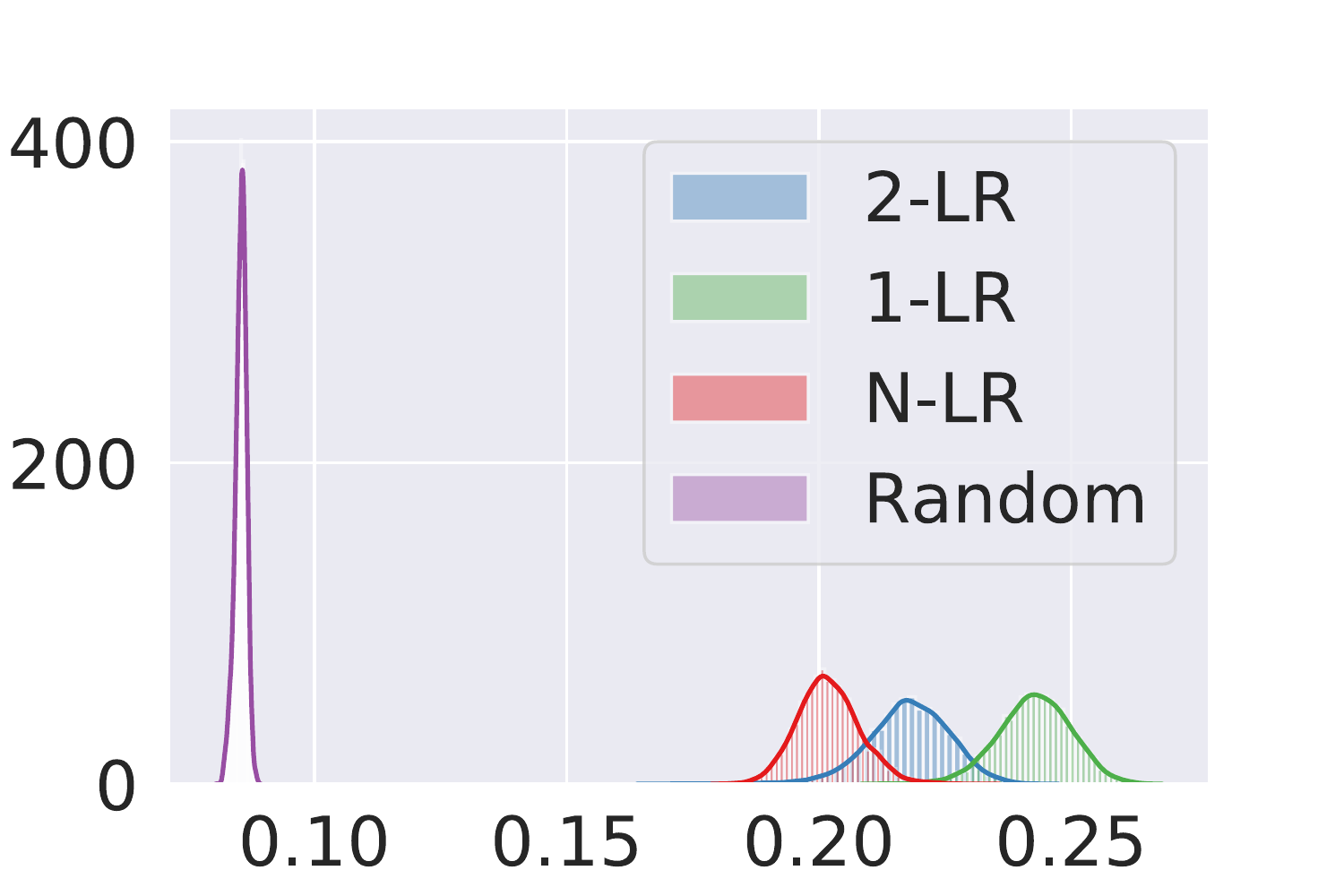
\end{subfigure}
\begin{subfigure}[c]{0.24\linewidth} \centering
\def\svgwidth{0.99\columnwidth} 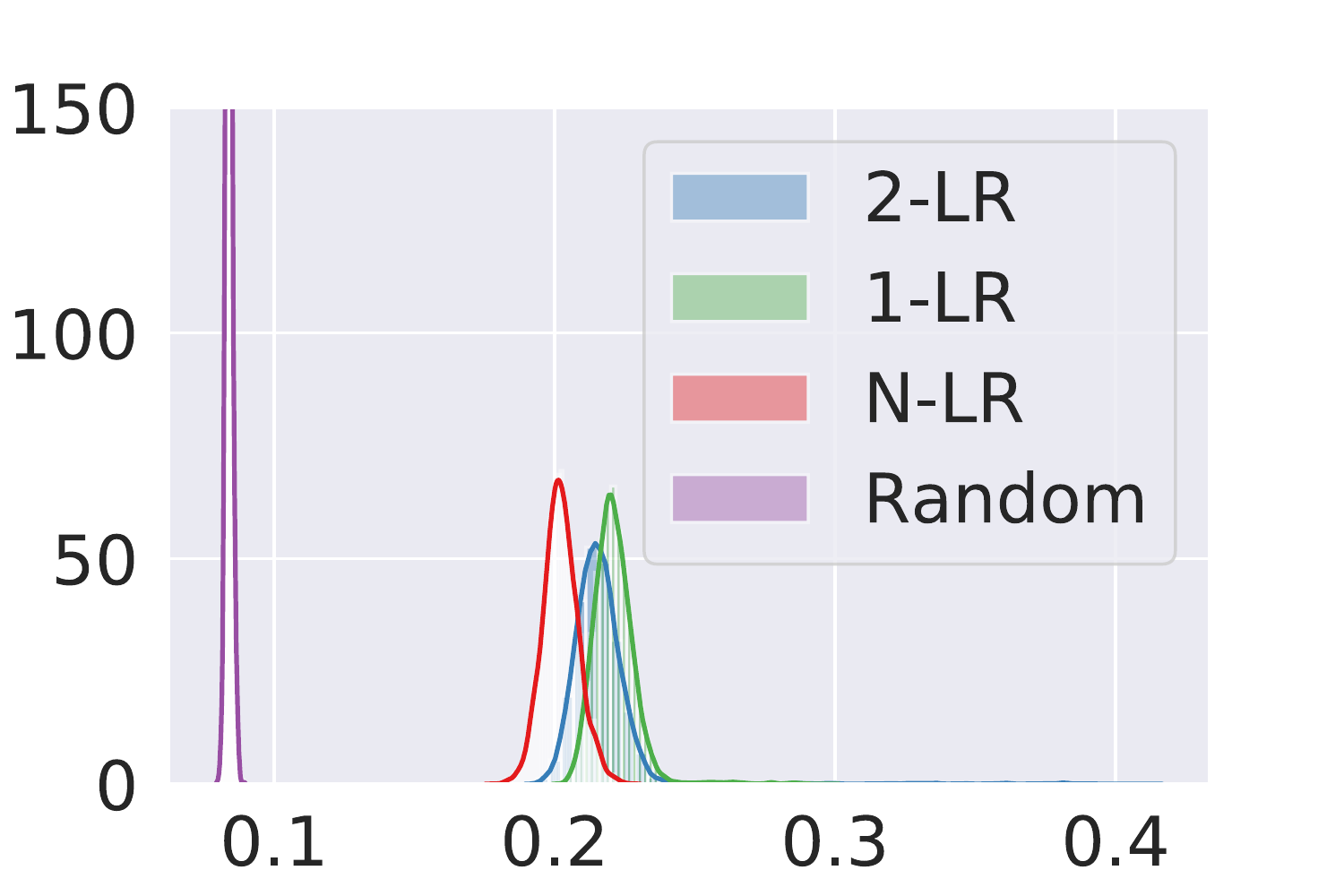
  \end{subfigure}
  \caption{Cushion of Layer 2}
  \label{fig:int_spec_lyr2_cush}
\end{figure}
\end{center}

\begin{center}
  \begin{figure}[h!]
  \begin{subfigure}[c]{0.24\linewidth} \centering
\def\svgwidth{0.99\columnwidth} 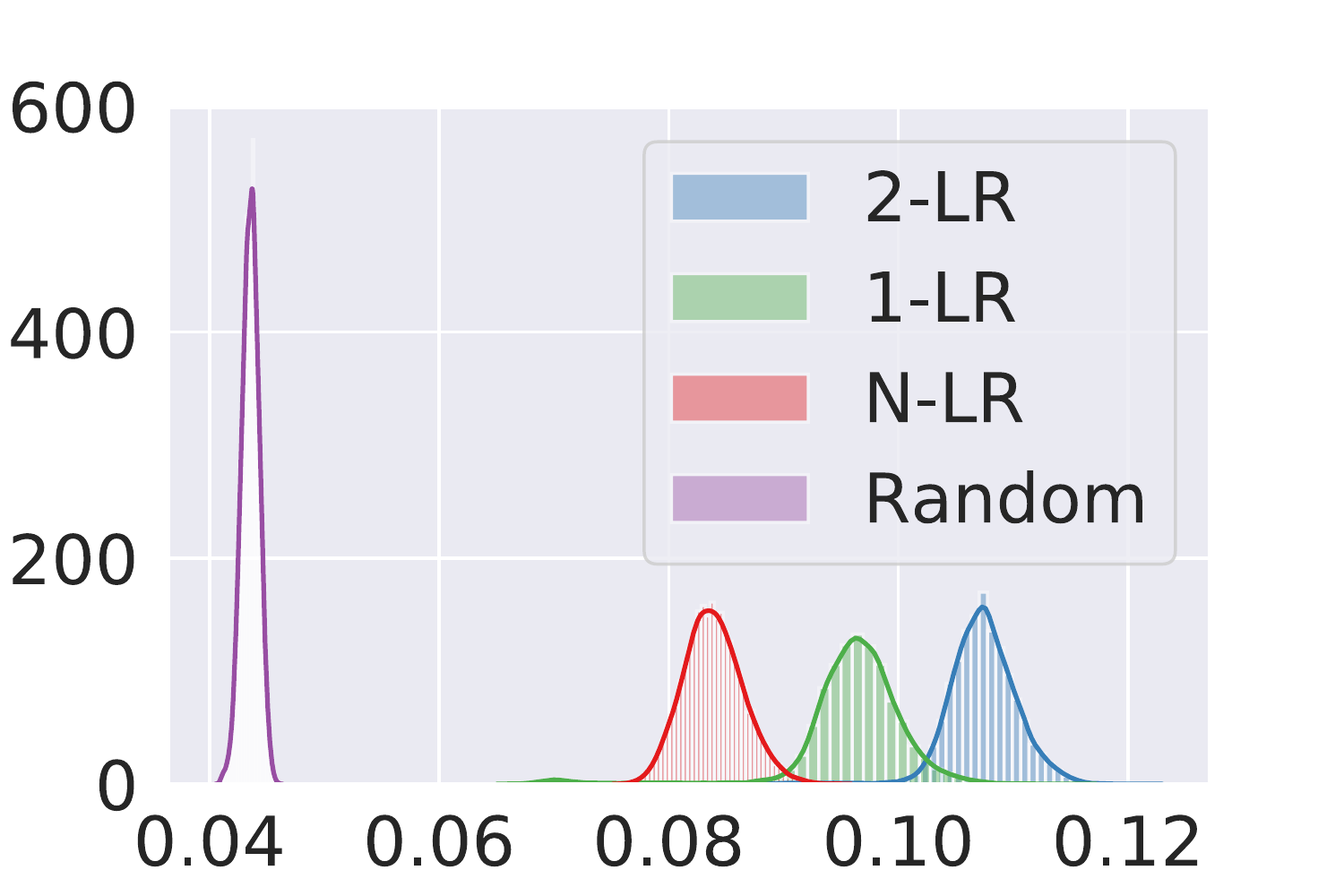
  \end{subfigure}
  \begin{subfigure}[c]{0.24\linewidth} \centering
\def\svgwidth{0.99\columnwidth} 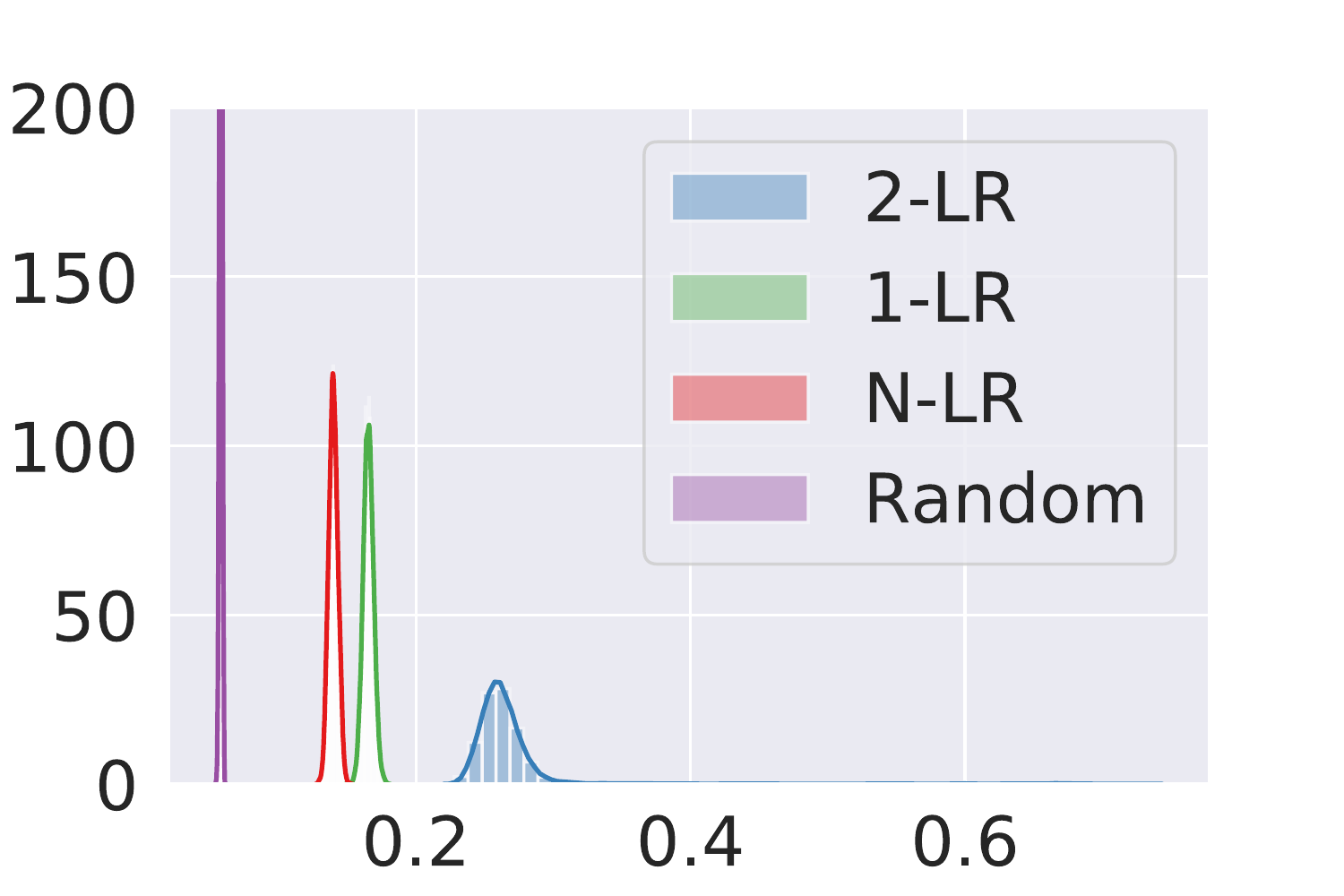
  \end{subfigure}
  \begin{subfigure}[c]{0.24\linewidth} \centering
\def\svgwidth{0.99\columnwidth} 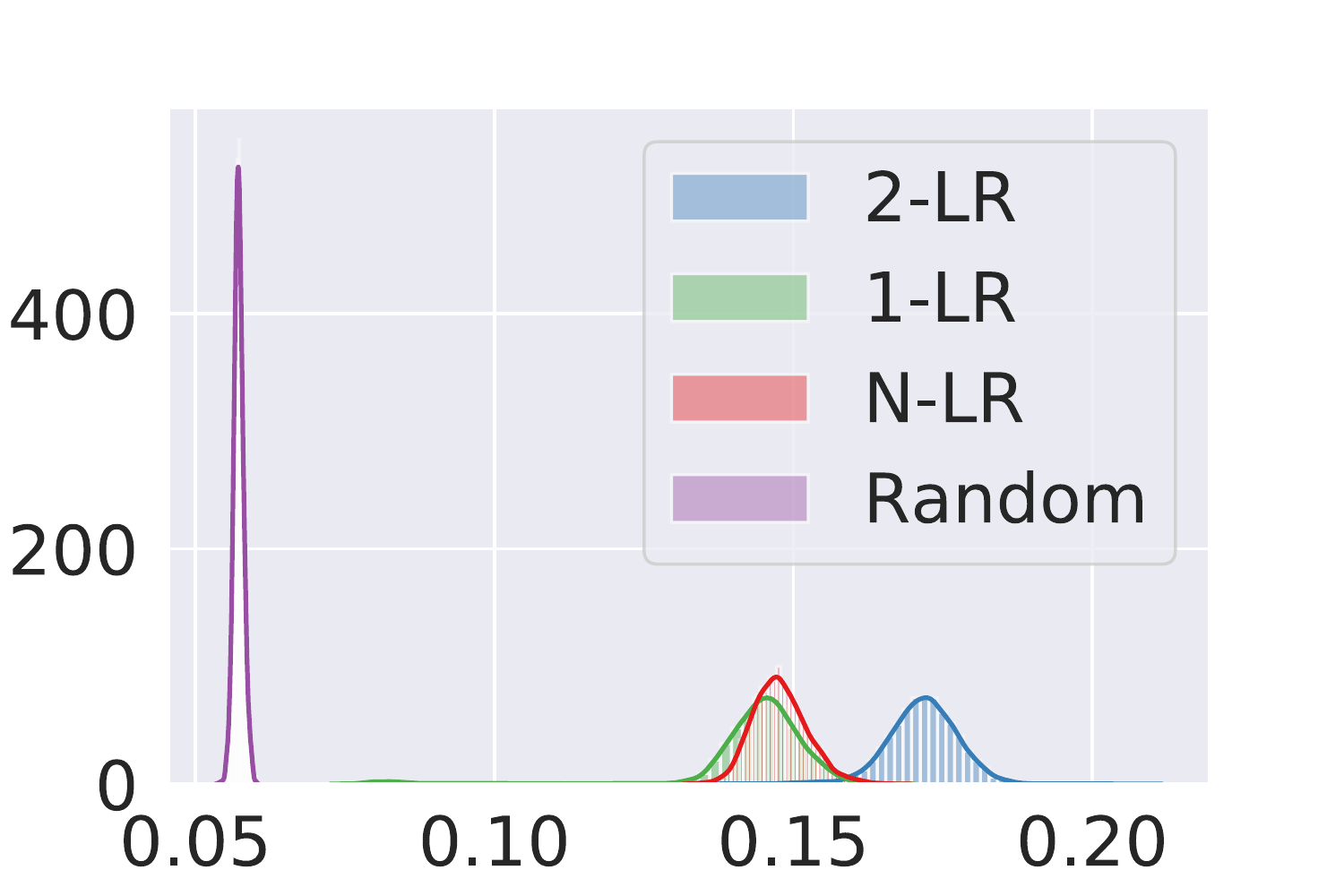
  \end{subfigure}
  \begin{subfigure}[c]{0.24\linewidth} \centering
\def\svgwidth{0.99\columnwidth} 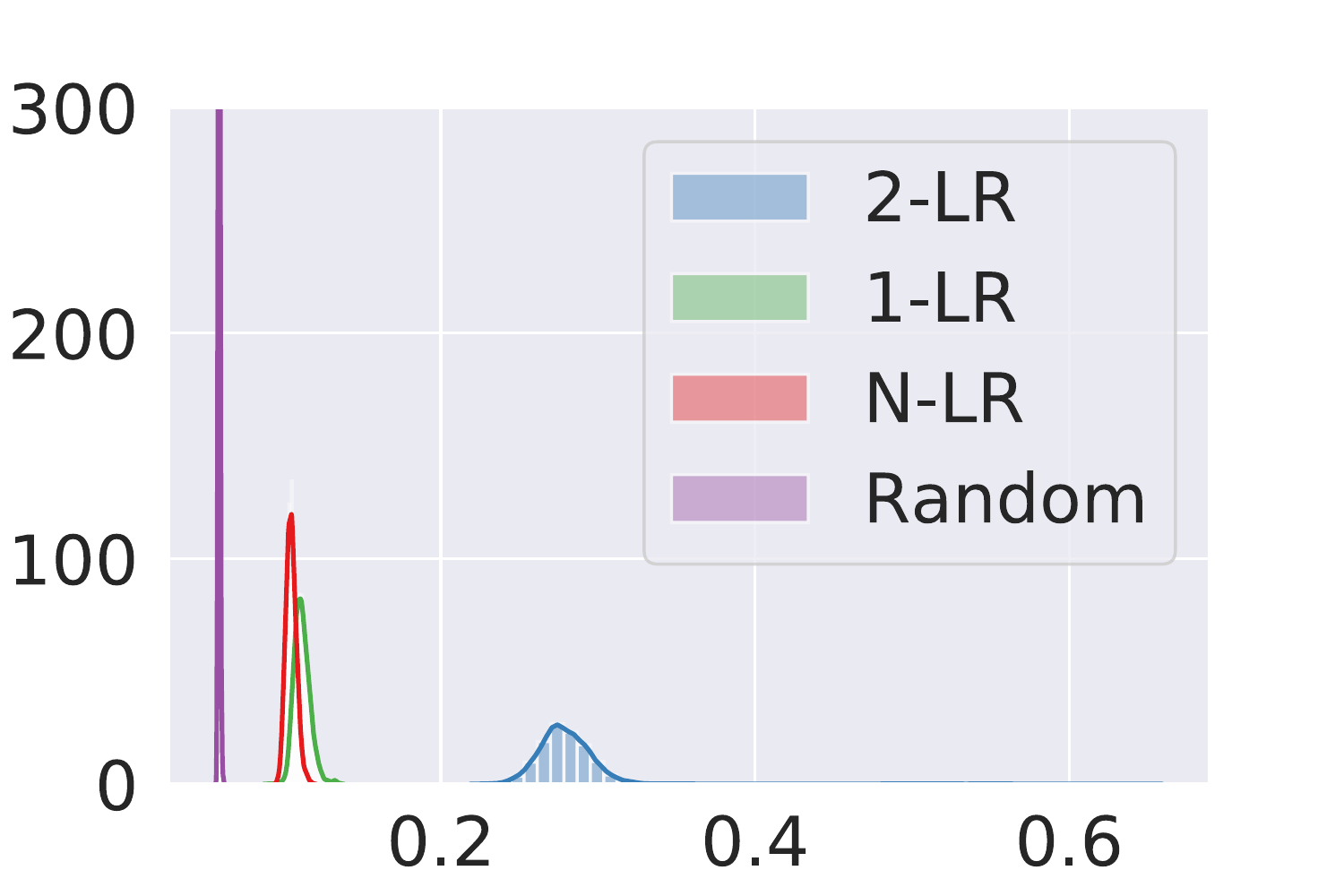
  \end{subfigure}
  \caption{Cushion of Layer 3}
  \label{fig:int_spec_lyr3_cush}
\end{figure}
\end{center} 
  \begin{figure}[h!]
  \begin{subfigure}[c]{0.24\linewidth} \centering
\def\svgwidth{0.99\columnwidth} 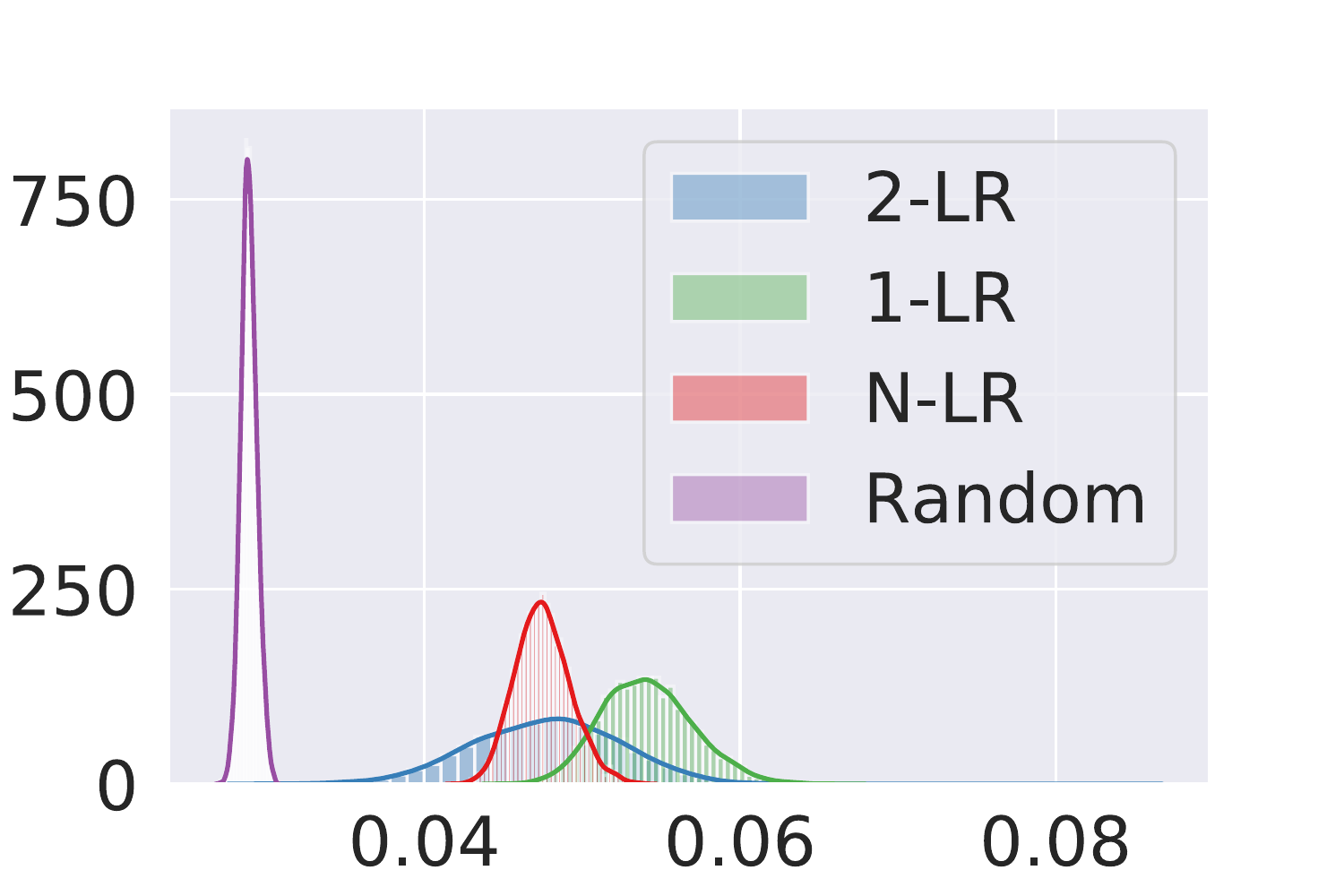
  \end{subfigure}
  \begin{subfigure}[c]{0.24\linewidth} \centering
\def\svgwidth{0.99\columnwidth} 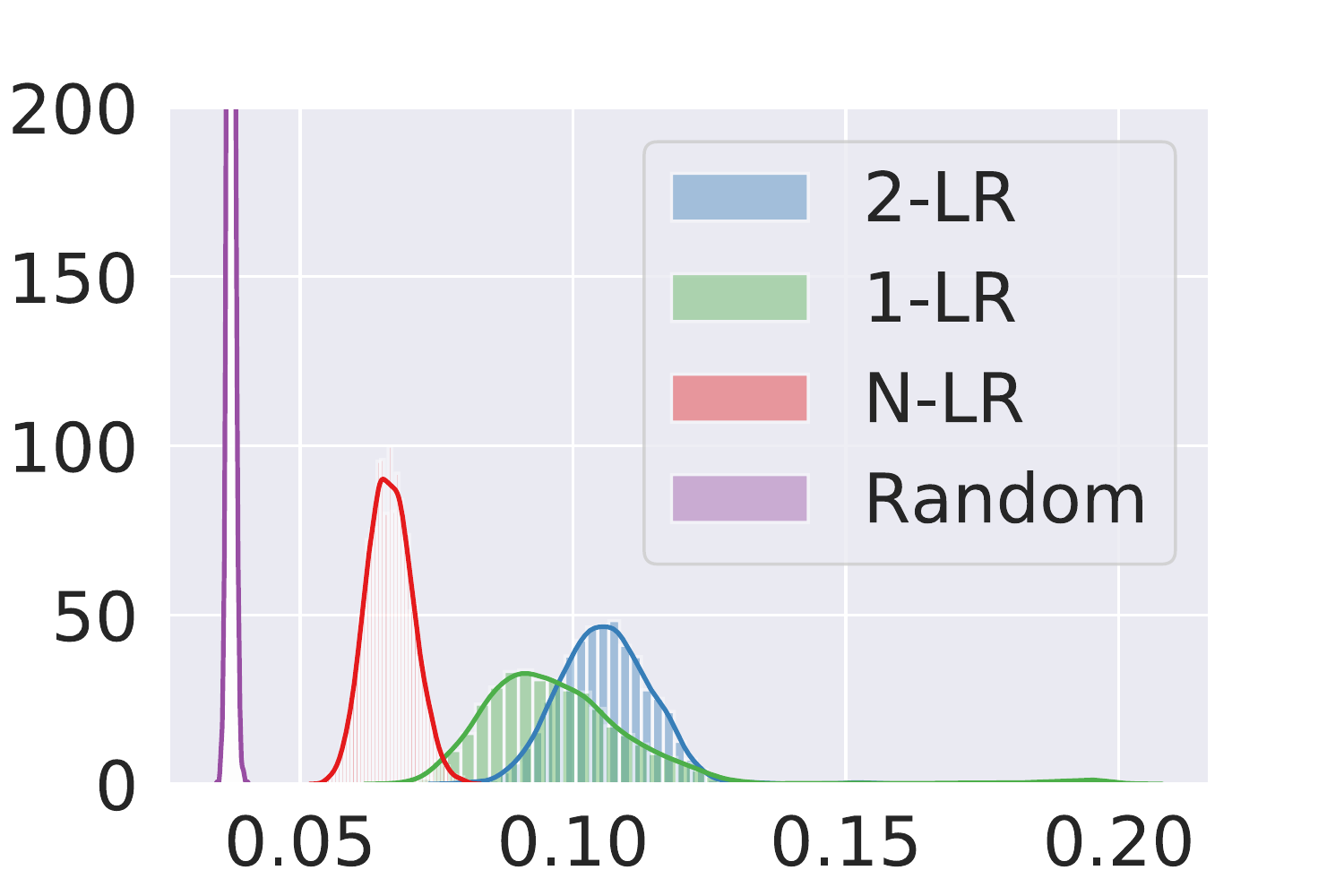
  \end{subfigure}
  \begin{subfigure}[c]{0.24\linewidth} \centering
\def\svgwidth{0.99\columnwidth} 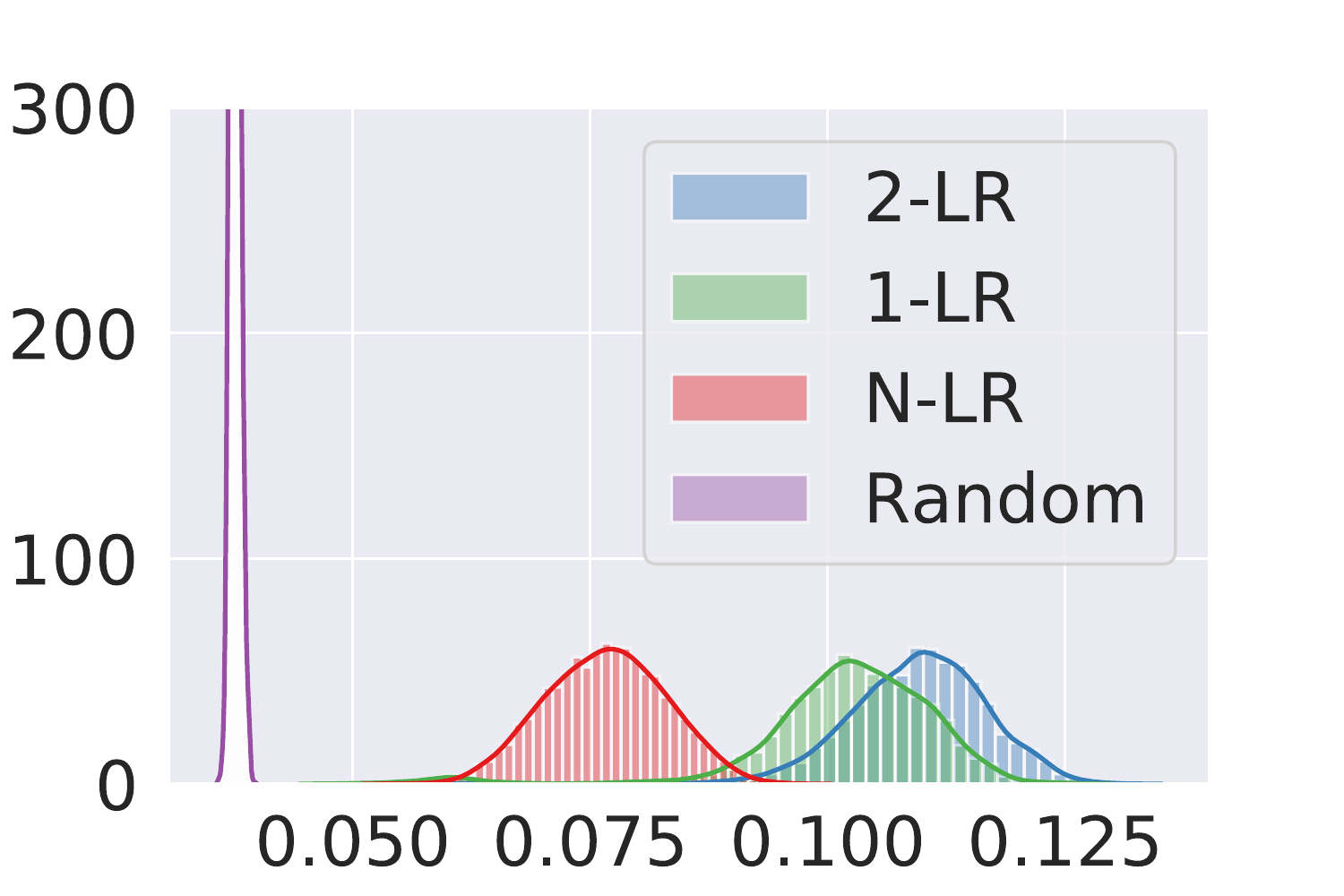
\end{subfigure}
 \begin{subfigure}[c]{0.24\linewidth} \centering
\def\svgwidth{0.99\columnwidth} 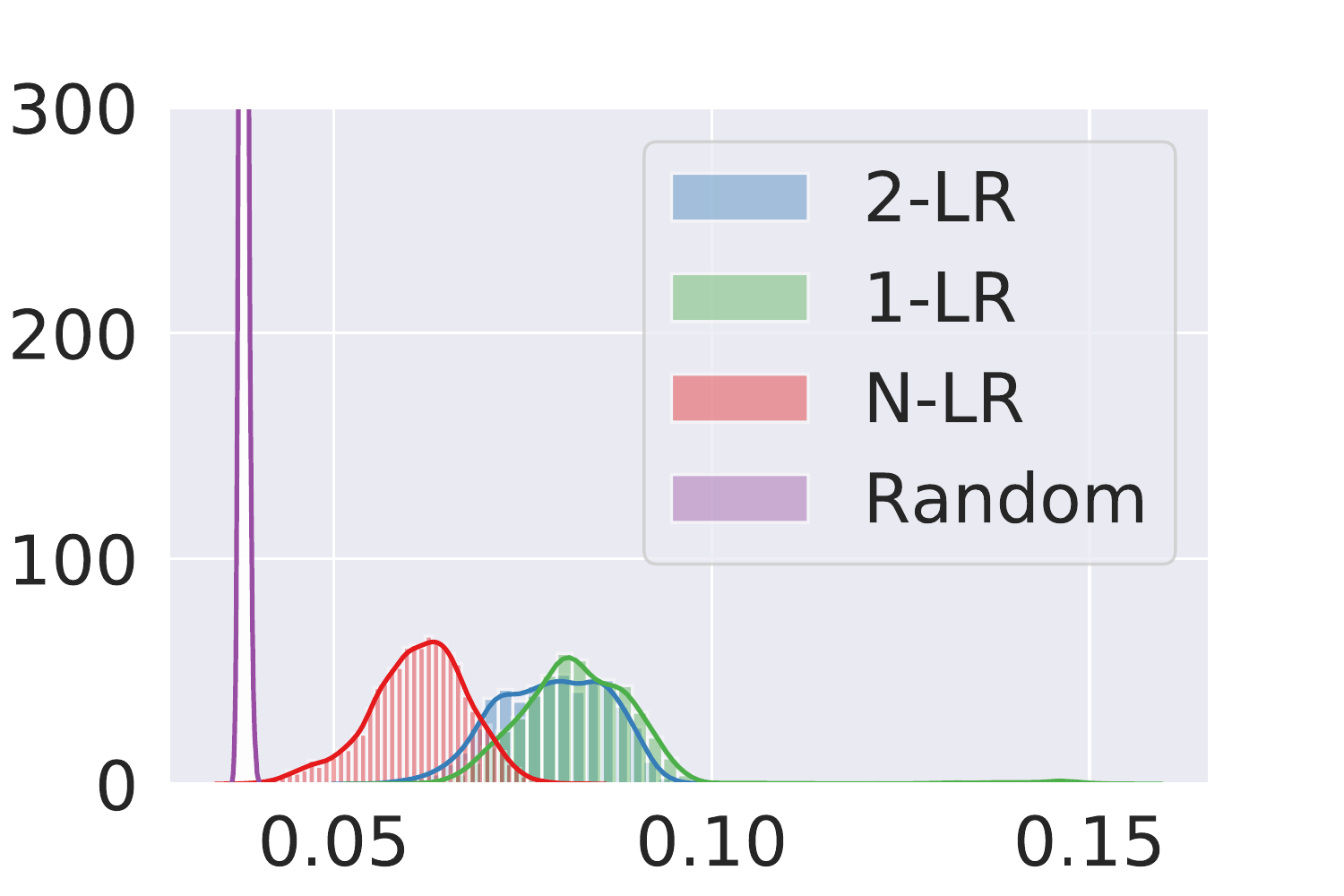
  \end{subfigure}
  \caption{Cushion of Layer 4}
  \label{fig:int_spec_lyr4_cush}
\end{figure}
\clearpage
   
\subsection{ResNet18 on SVHN}
\label{sec:resnet18-svhn}

\paragraph{Adversarial Noise Attenuation for SVHN}
\label{sec:noise-stability-pert-svhn}
In Figure~\ref{fig:perturbation_spaces_svhn}, we plot experiments similar to the one in~\ref{fig:perturbation_spaces} for SVHN dataset and we observe a similar trend.

\vspace{-10pt}   \begin{figure}[h]
     \begin{subfigure}[c]{0.14\linewidth} \centering
\def\svgwidth{0.90\columnwidth}
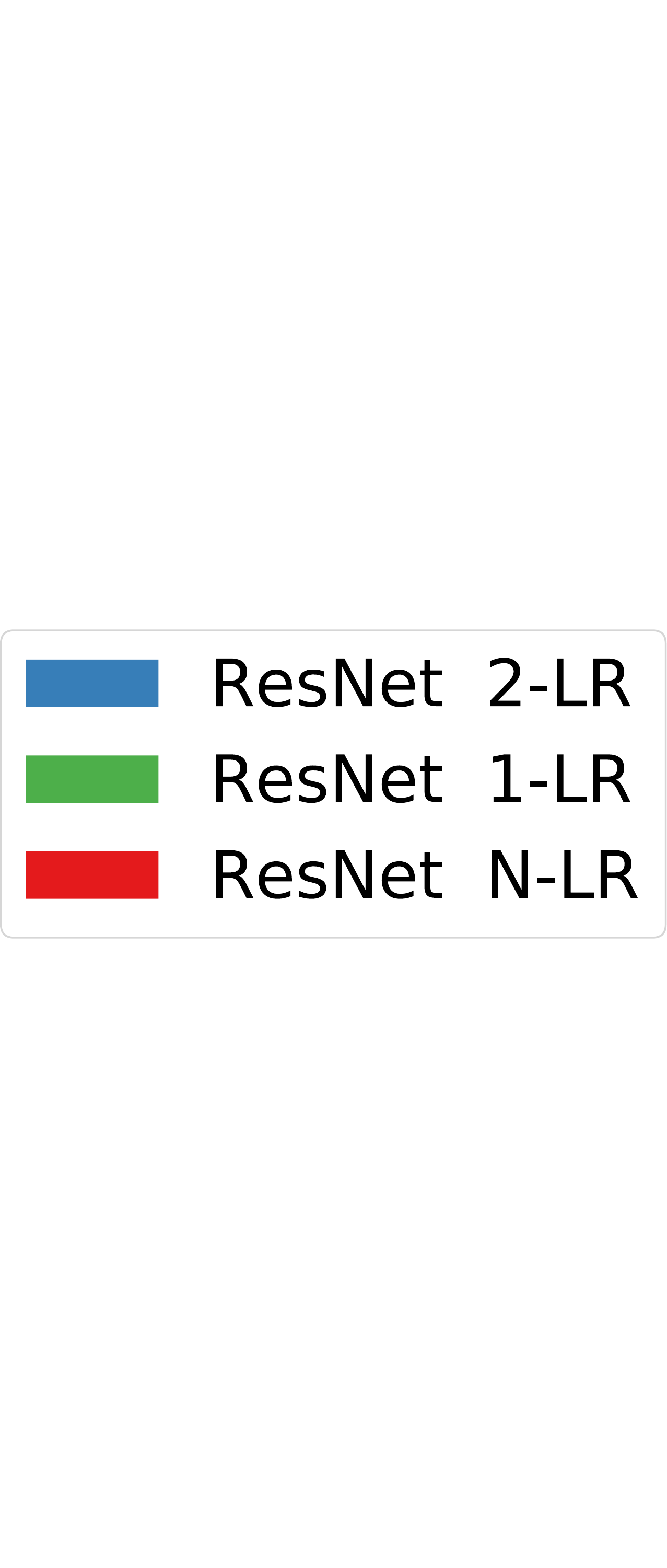
     \end{subfigure}\hspace{20pt}
     \begin{subfigure}[c]{0.6\linewidth} \centering
\def\svgwidth{0.99\columnwidth}
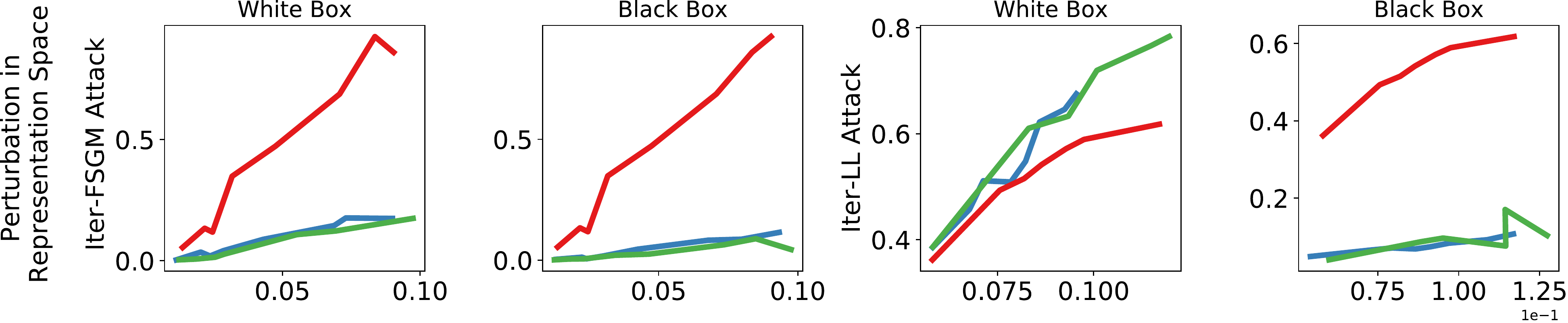
     \end{subfigure}\vspace{-4em}
     \caption[Change in perturbation]{\small Adversarial Perturbation in
Input Space and Perturbation in Representation Space in SVHN}
     \label{fig:perturbation_spaces_svhn}\vspace{-1em}
   \end{figure}

\vspace{-5pt}\paragraph{Layer Cushion  for SVHN}
 The following correspond to ResNet blocks. Each block has two smaller
sub-blocks where each sub-block has two convolutional layers. The
value of layer cushion for these modules of one block are plotted
below.

\begin{center}
  \begin{figure}[h!]
  \begin{subfigure}[c]{0.22\linewidth} \centering
\def\svgwidth{0.99\columnwidth}
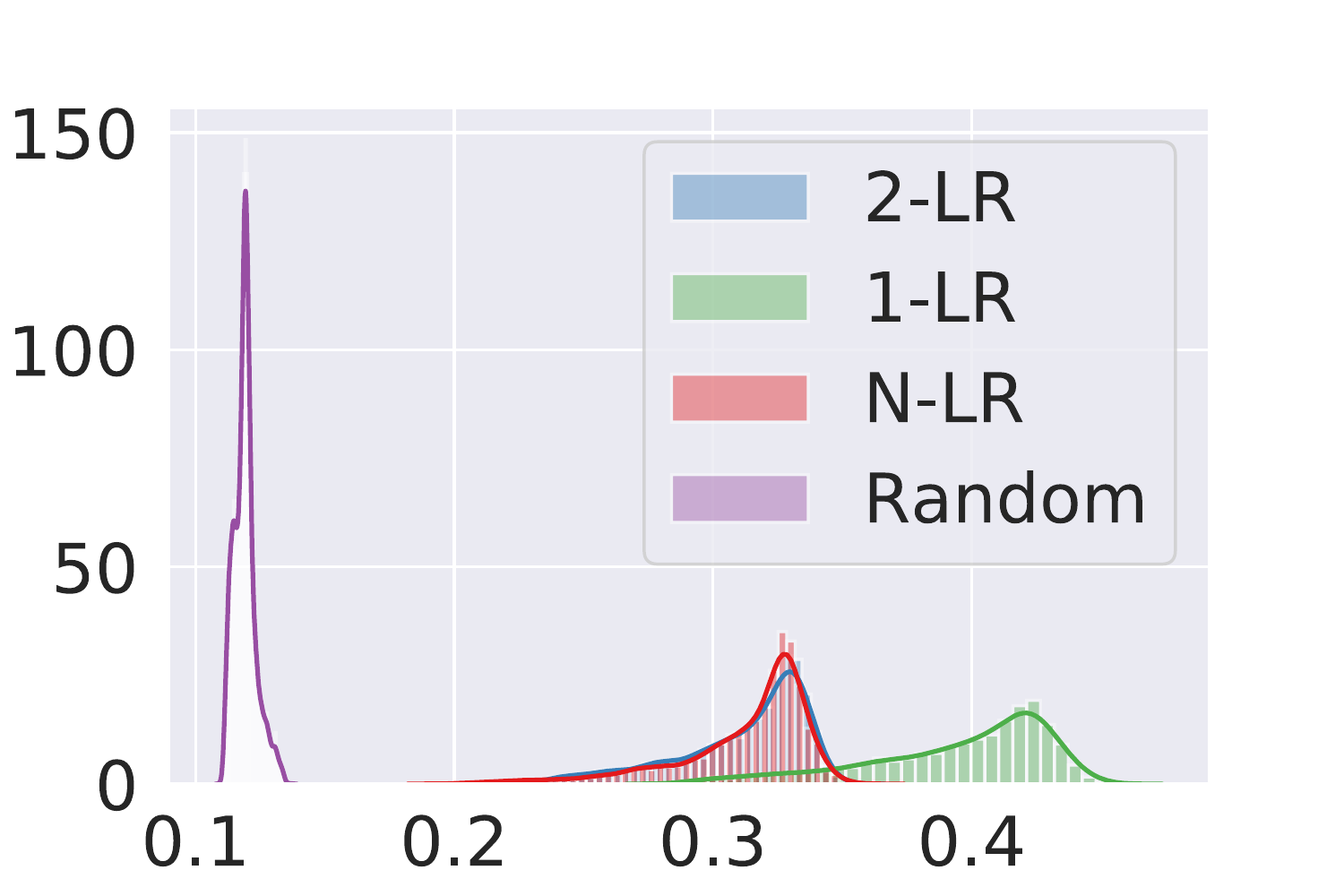
  \end{subfigure}\hspace{15pt}
  \begin{subfigure}[c]{0.22\linewidth} \centering
\def\svgwidth{0.99\columnwidth}
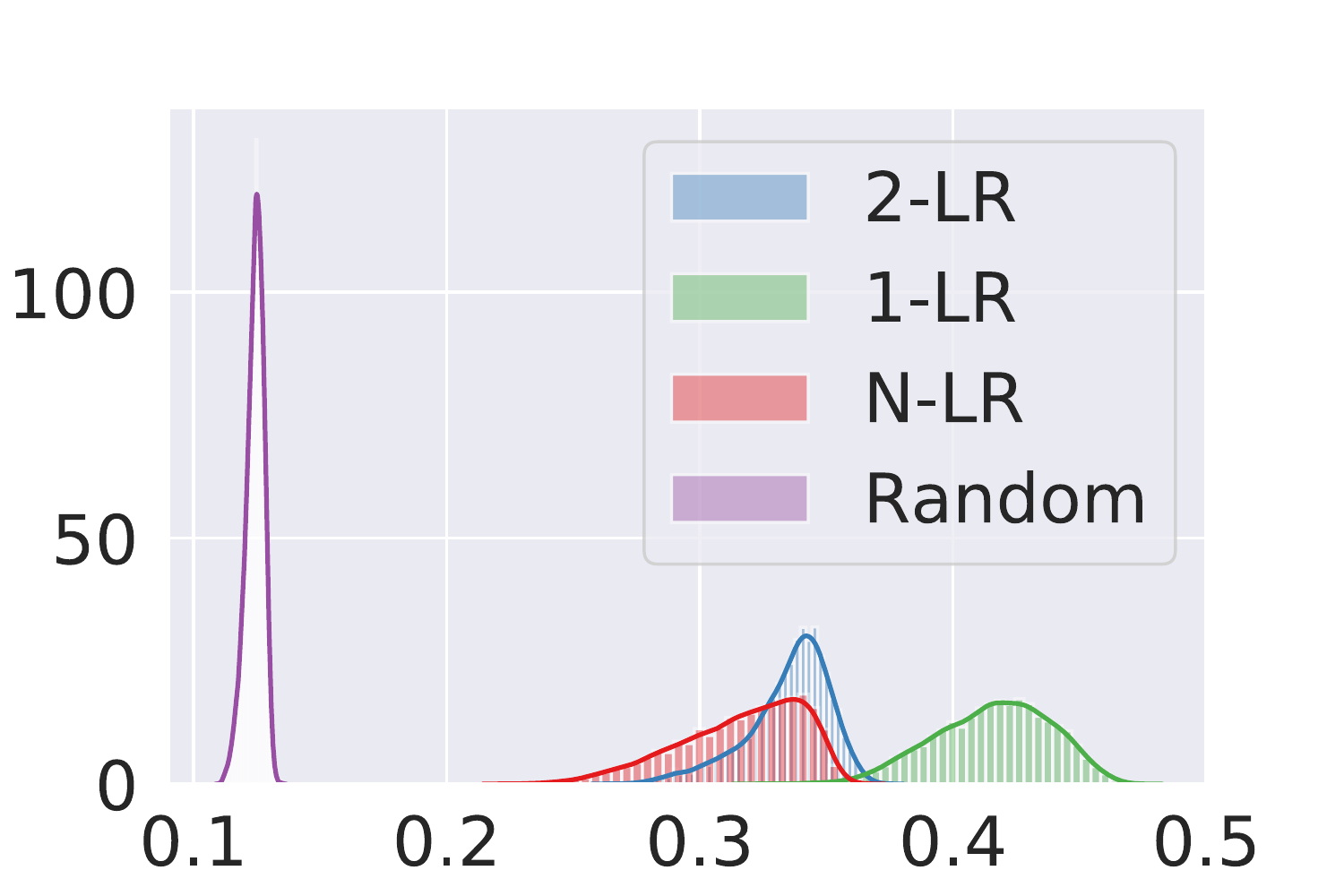
  \end{subfigure}\hspace{15pt}
  \begin{subfigure}[c]{0.22\linewidth} \centering
\def\svgwidth{0.99\columnwidth}
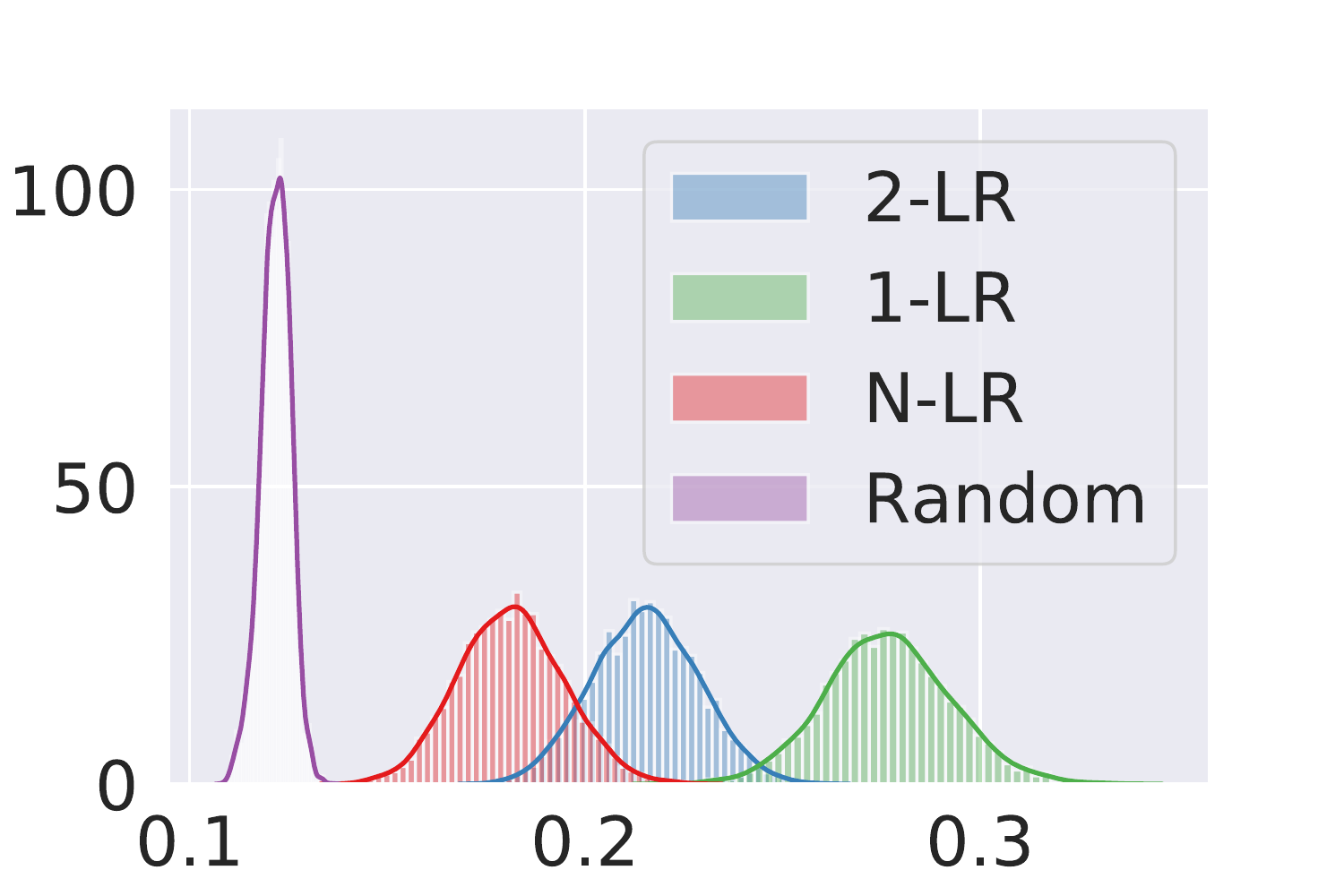
  \end{subfigure}\hspace{15pt}
  \begin{subfigure}[c]{0.22\linewidth} \centering
\def\svgwidth{0.99\columnwidth}
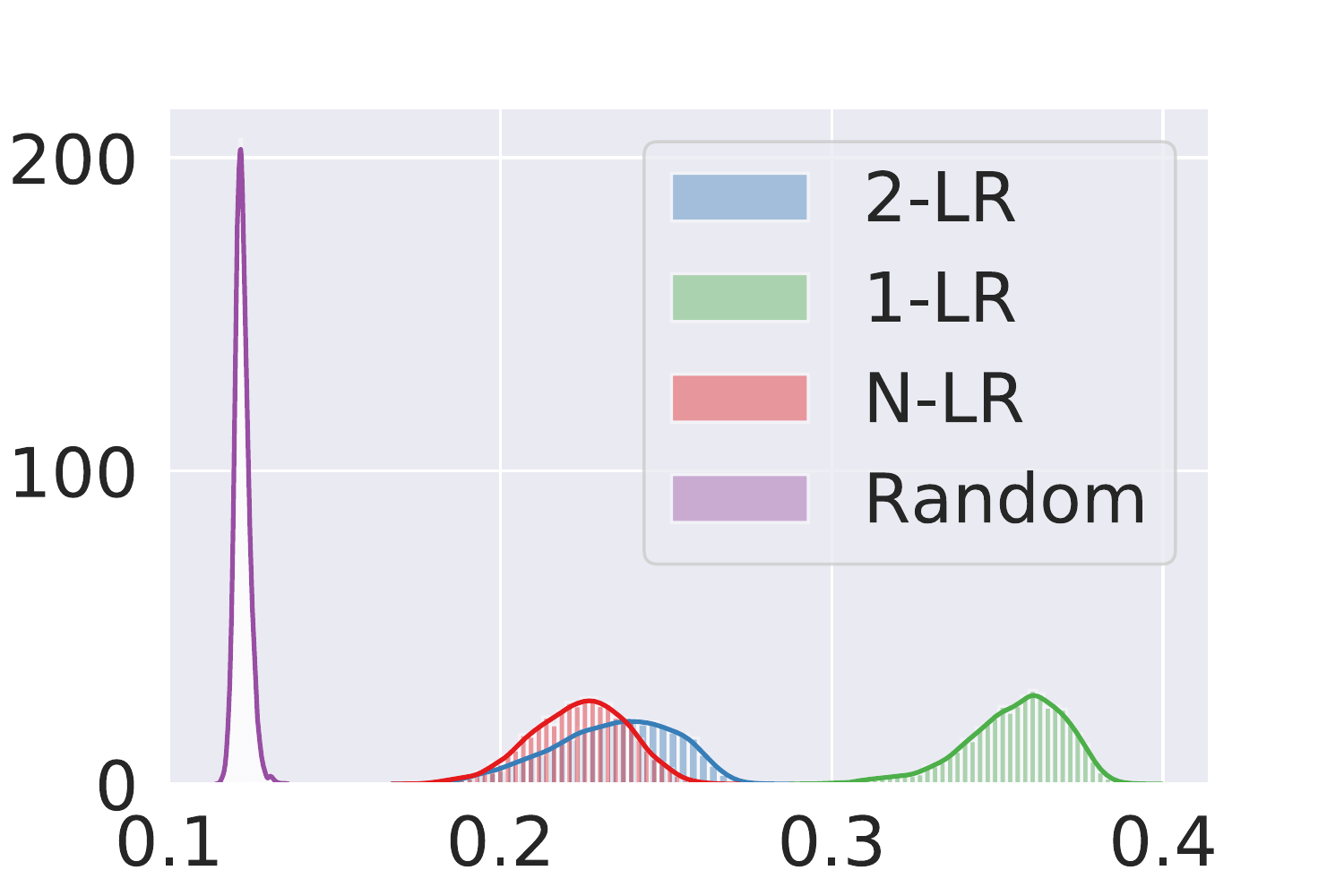
  \end{subfigure}
  \caption{Cushion of Layer 1}
  \label{fig:int_svhn_spec_lyr_cush}
\end{figure}
\end{center}
\vspace{-35pt}
\begin{center}
  \begin{figure}[h!]
  \begin{subfigure}[c]{0.22\linewidth} \centering
\def\svgwidth{0.99\columnwidth}
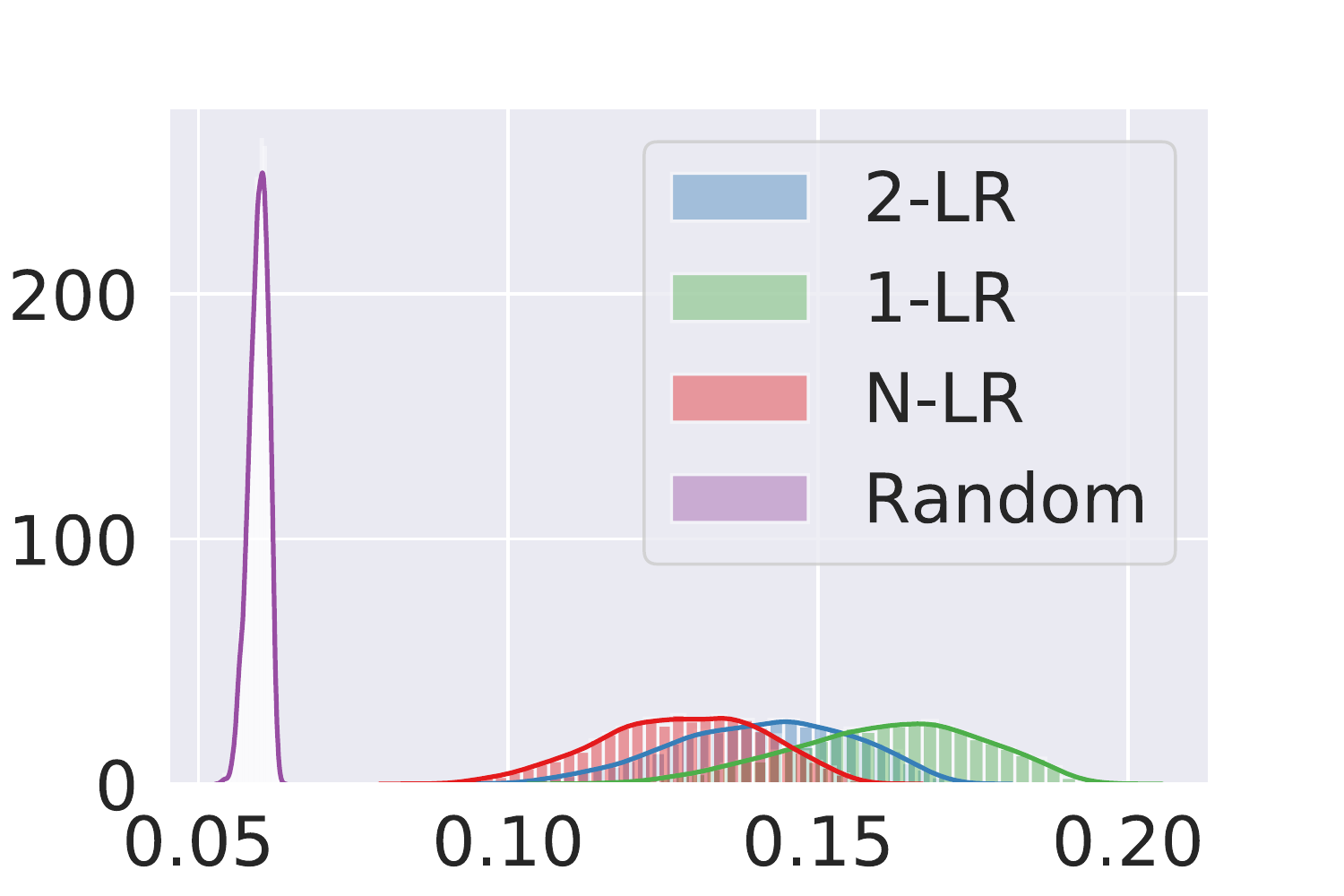
  \end{subfigure}\hspace{15pt}
  \begin{subfigure}[c]{0.22\linewidth} \centering
\def\svgwidth{0.99\columnwidth}
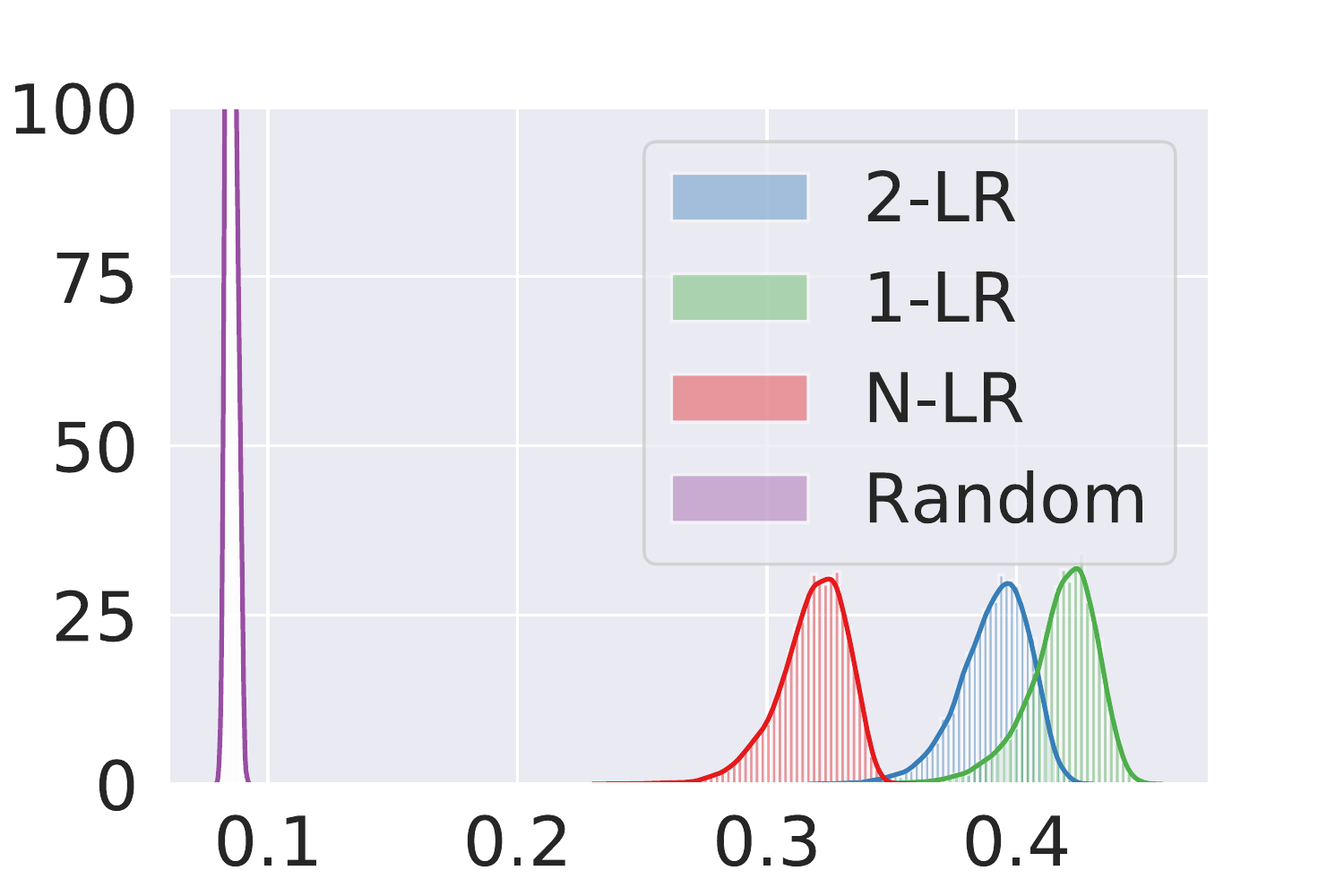
  \end{subfigure}\hspace{15pt}
  \begin{subfigure}[c]{0.22\linewidth} \centering
\def\svgwidth{0.99\columnwidth}
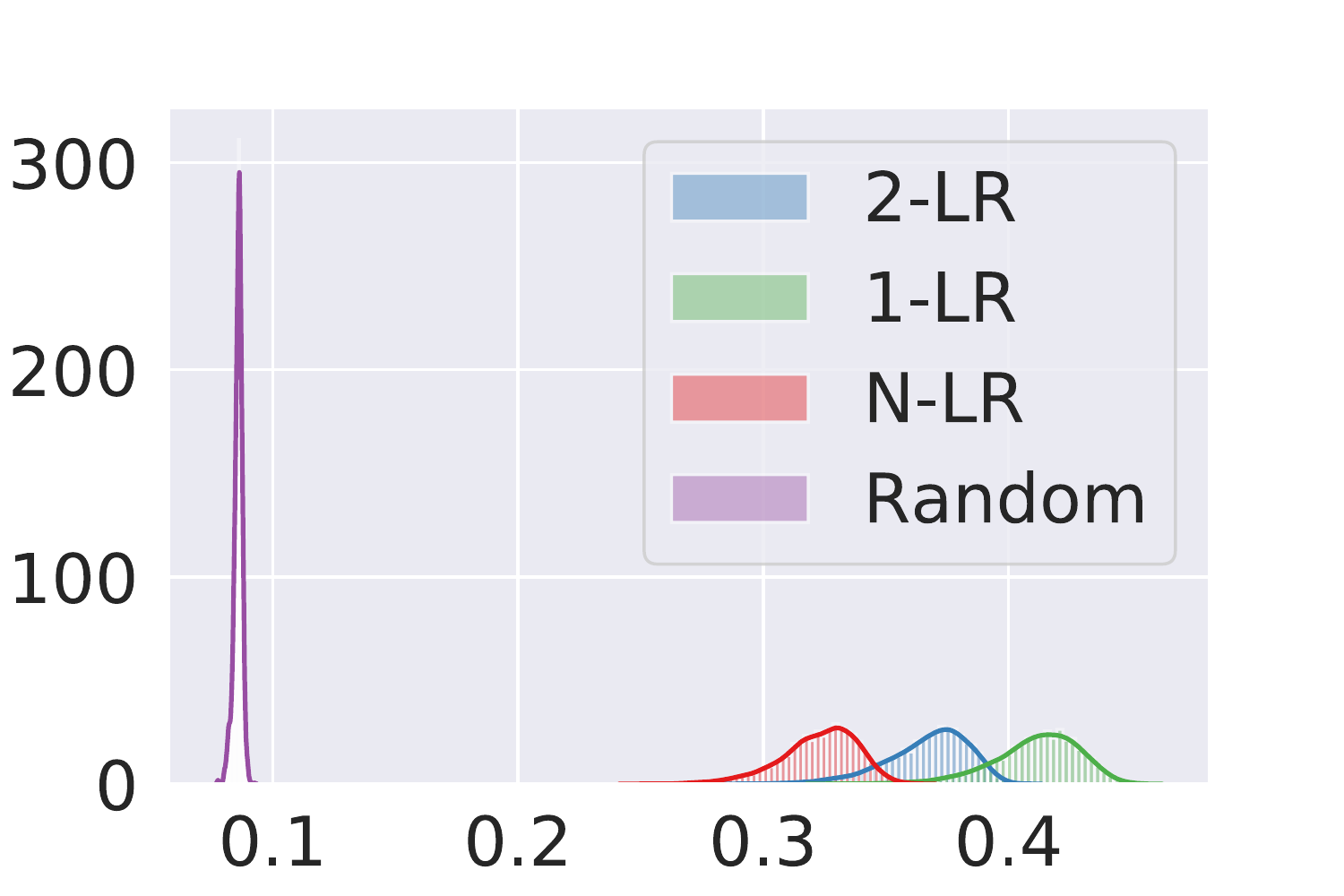
  \end{subfigure}\hspace{15pt}
  \begin{subfigure}[c]{0.22\linewidth} \centering
\def\svgwidth{0.99\columnwidth}
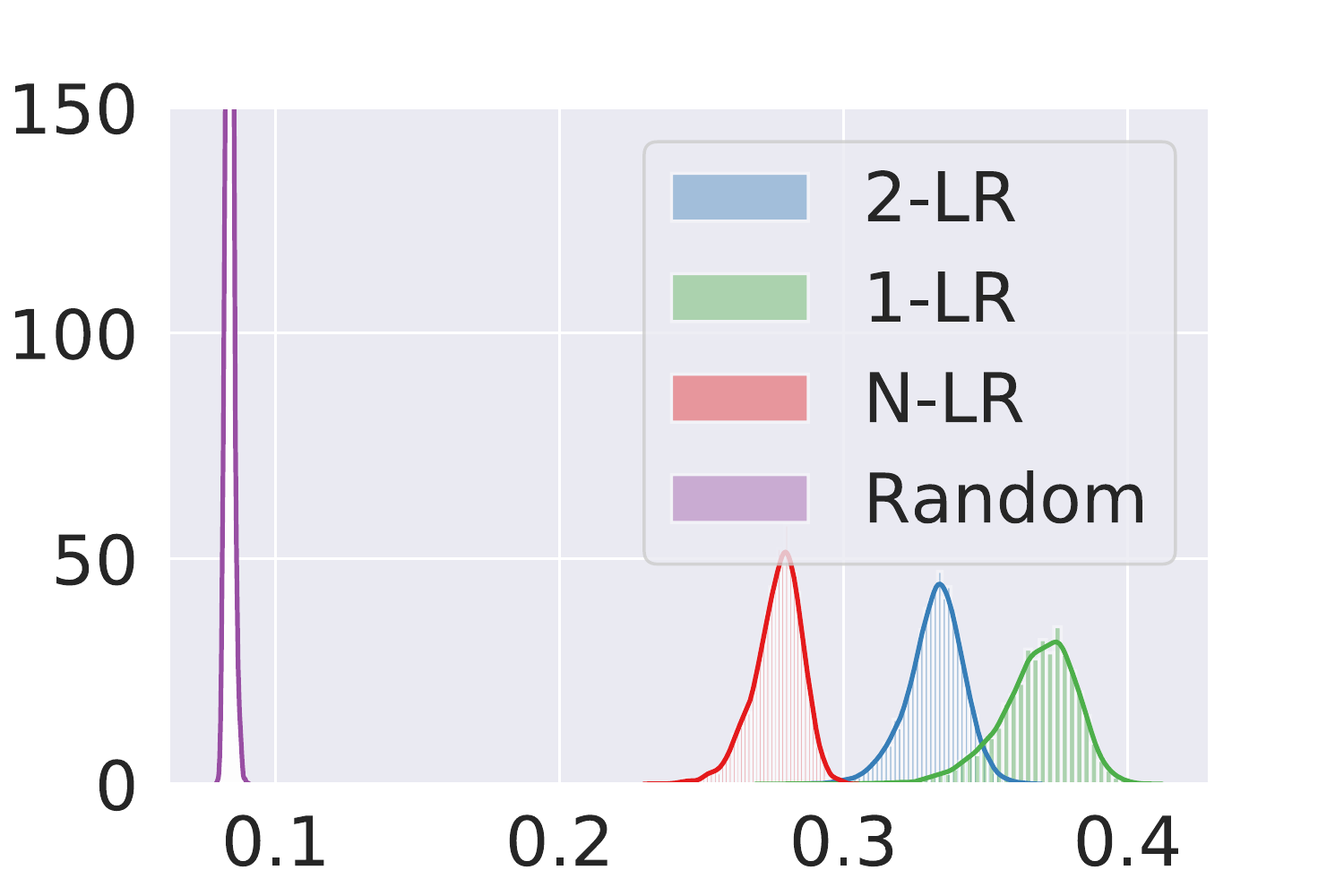
  \end{subfigure}
  \caption{Cushion of Layer 2}
  \label{fig:int_svhn_spec_lyr2_cush}
\end{figure}
\end{center}
\vspace{-35pt}
\begin{center}
  \begin{figure}[h!]
  \begin{subfigure}[c]{0.22\linewidth} \centering
\def\svgwidth{0.99\columnwidth}
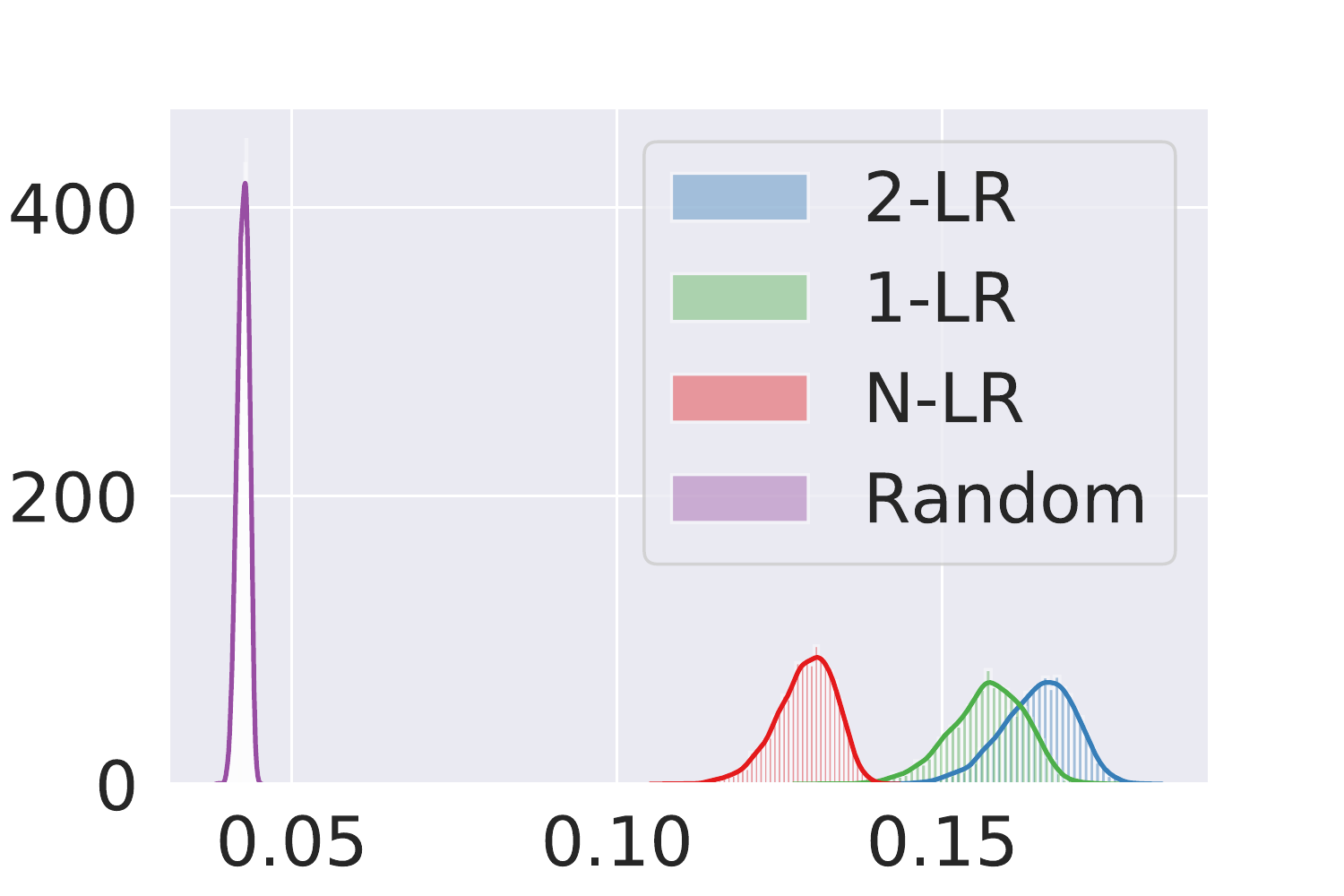
  \end{subfigure}\hspace{15pt}
  \begin{subfigure}[c]{0.22\linewidth} \centering
\def\svgwidth{0.99\columnwidth}
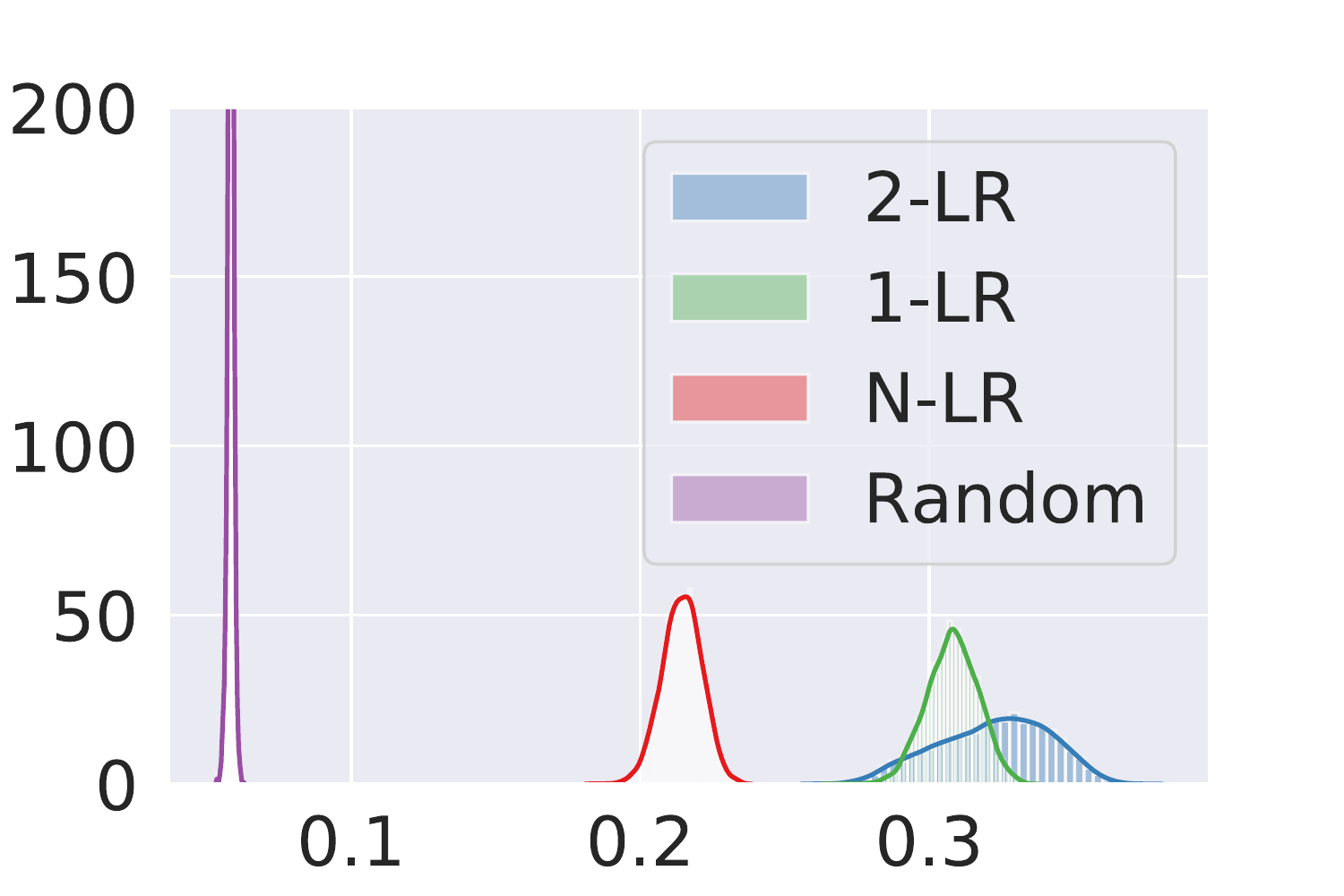
  \end{subfigure}\hspace{15pt}
  \begin{subfigure}[c]{0.22\linewidth} \centering
\def\svgwidth{0.99\columnwidth}
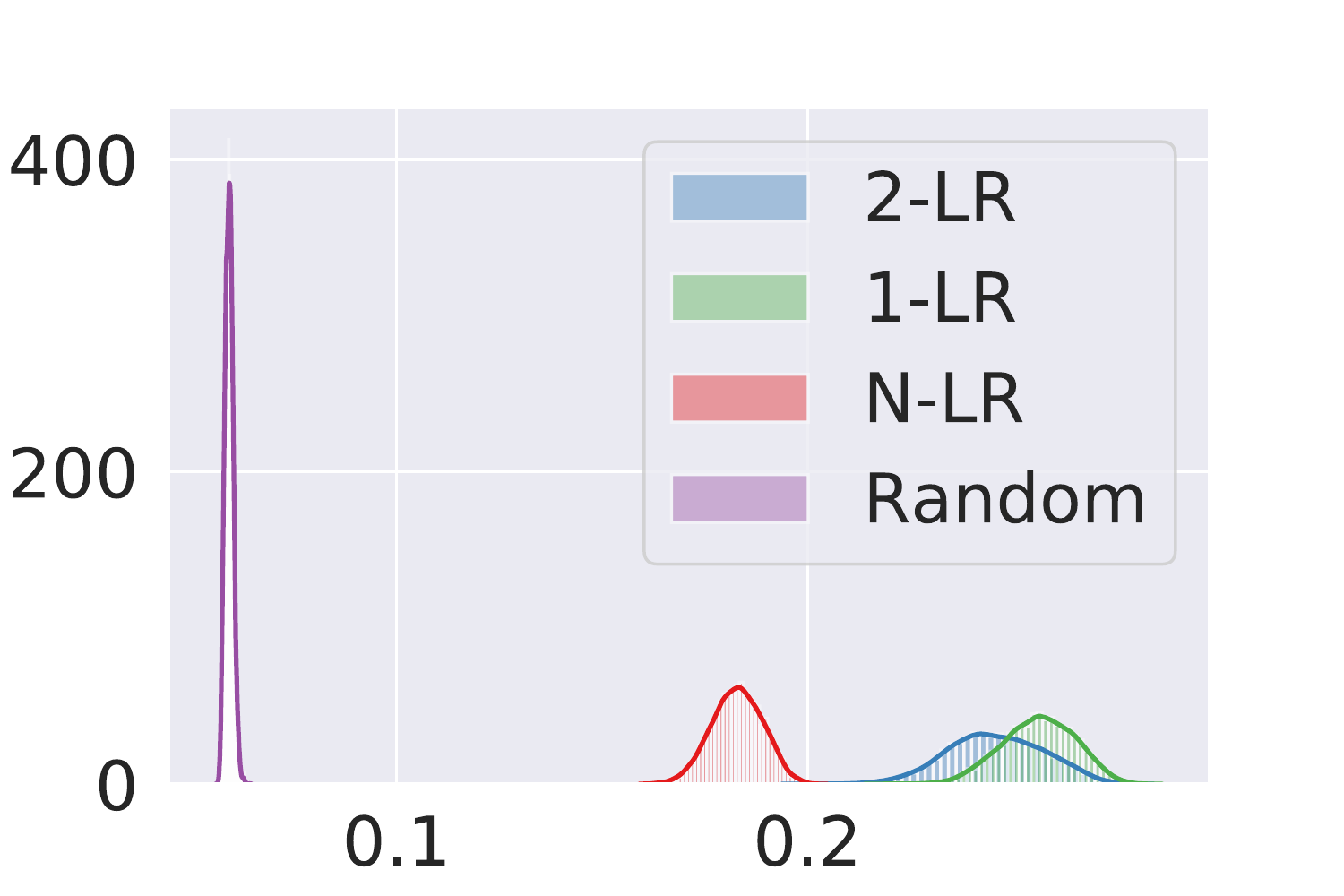
  \end{subfigure}\hspace{15pt}
  \begin{subfigure}[c]{0.22\linewidth} \centering
\def\svgwidth{0.99\columnwidth}
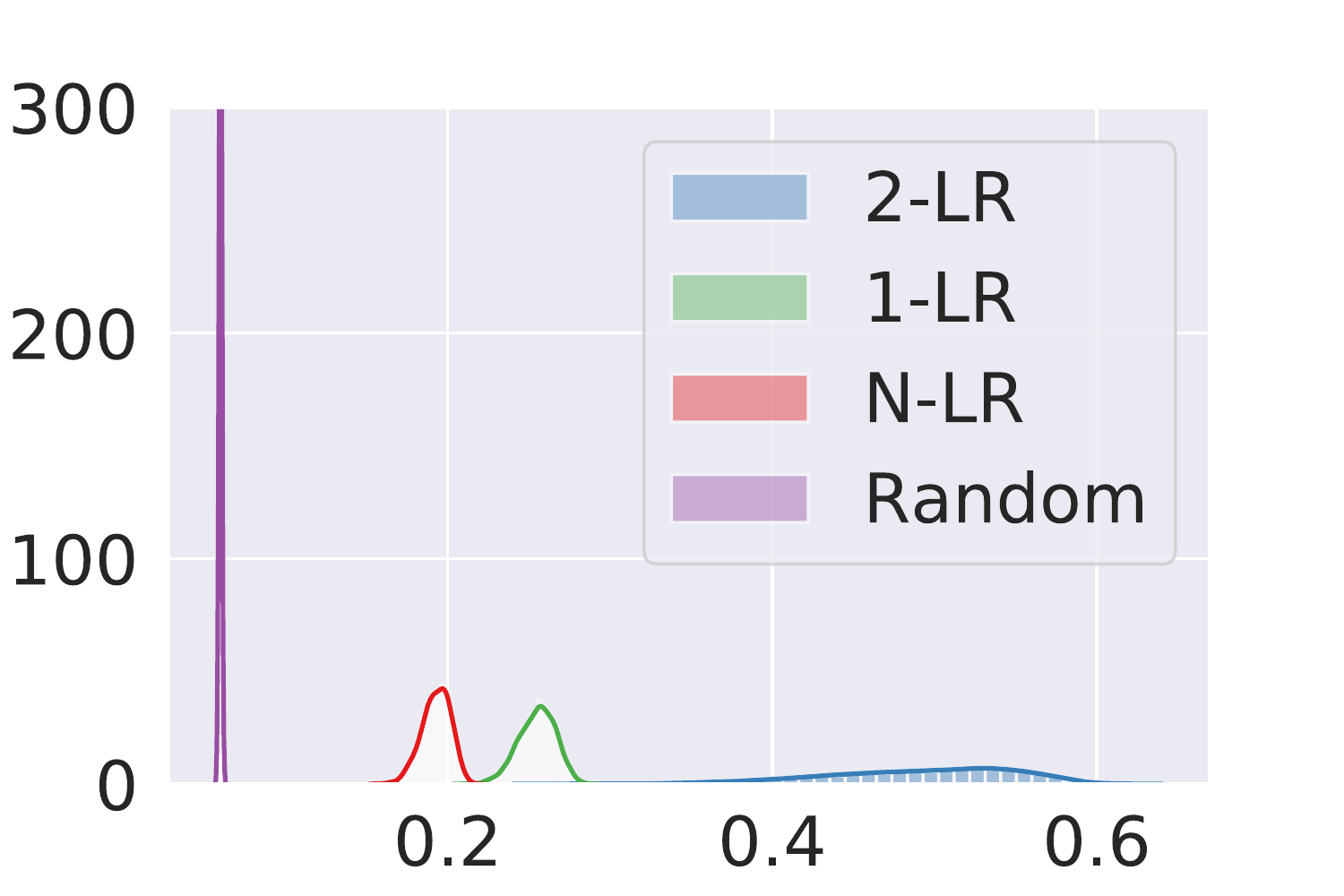
  \end{subfigure}
  \caption{Cushion of Layer 3}
  \label{fig:int_svhn_spec_lyr3_cush}
\end{figure}
\end{center}\vspace{-35pt}
  \begin{figure}[h!]
  \begin{subfigure}[c]{0.22\linewidth} \centering
\def\svgwidth{0.99\columnwidth}
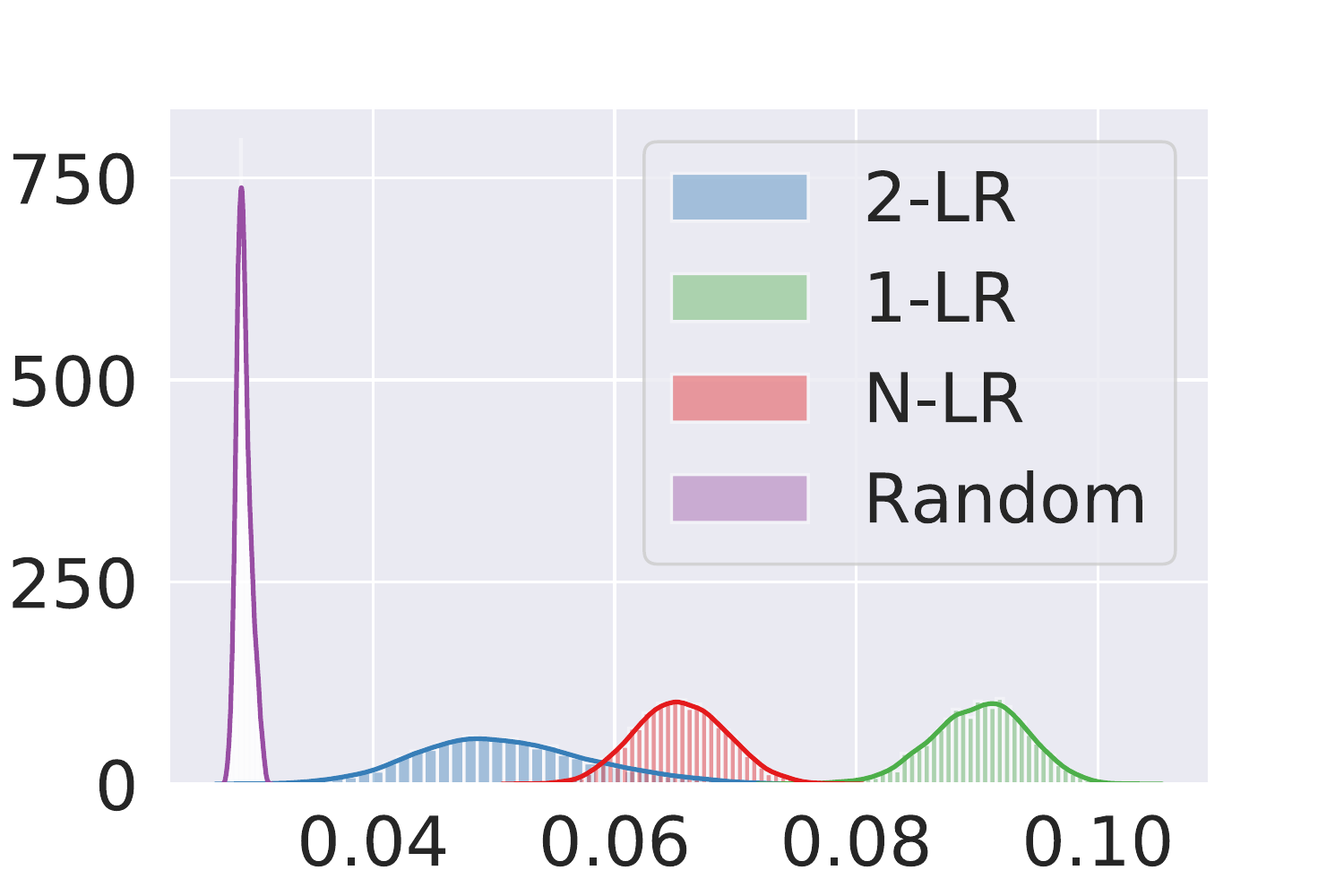
  \end{subfigure}\hspace{15pt}
  \begin{subfigure}[c]{0.22\linewidth} \centering
\def\svgwidth{0.99\columnwidth}
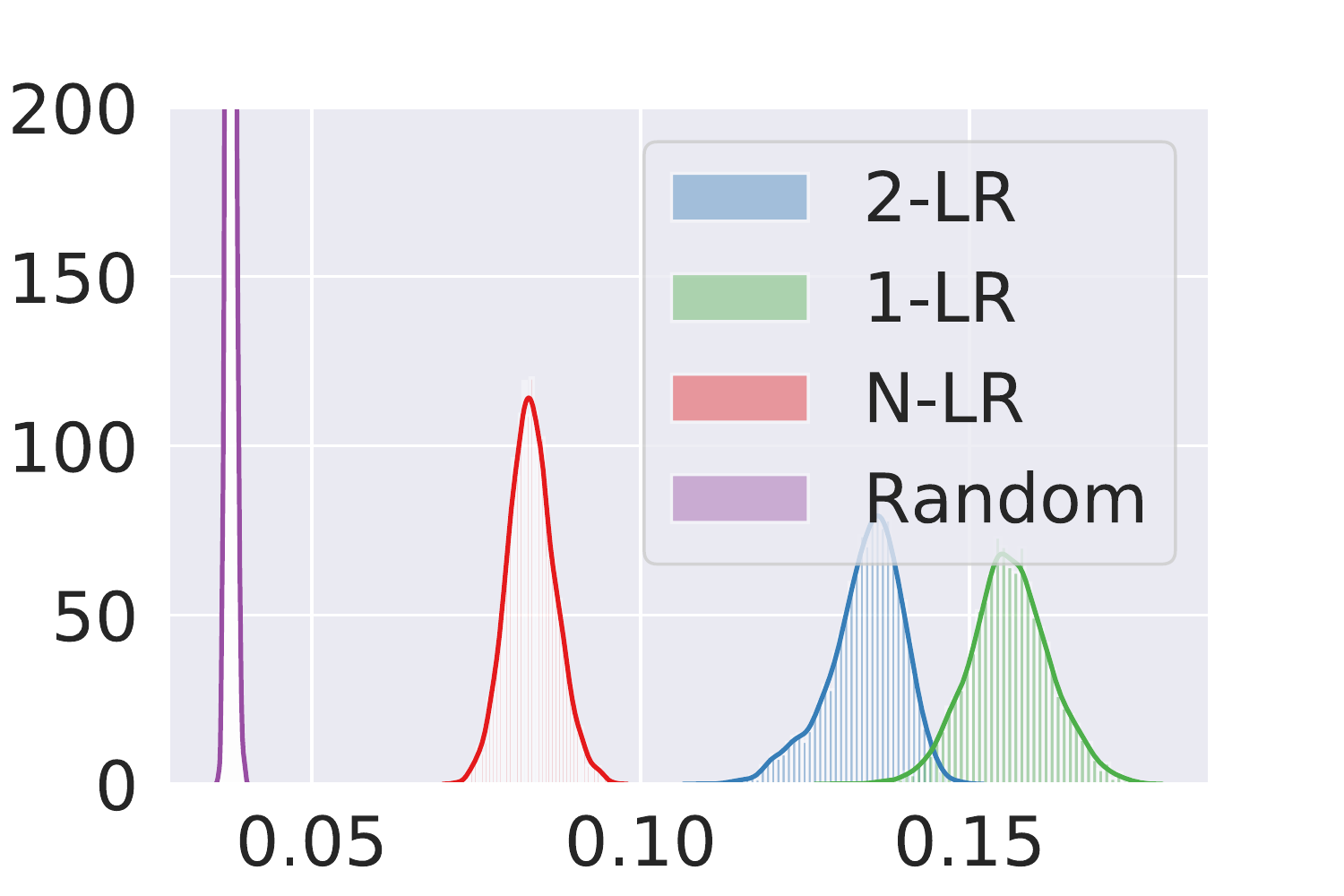
  \end{subfigure}\hspace{15pt}
  \begin{subfigure}[c]{0.22\linewidth} \centering
\def\svgwidth{0.99\columnwidth}
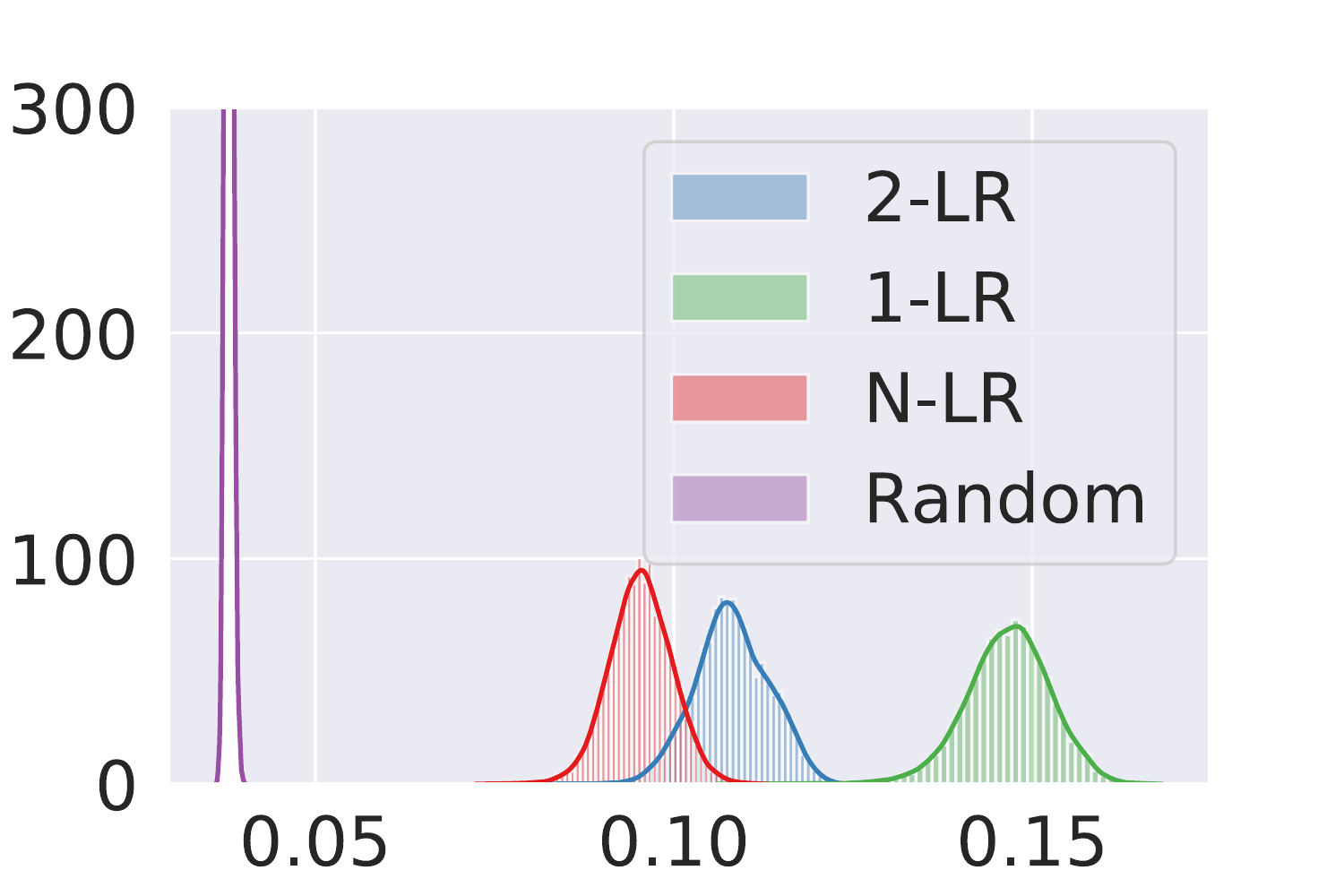
  \end{subfigure}\hspace{15pt}
  \begin{subfigure}[c]{0.22\linewidth} \centering
\def\svgwidth{0.99\columnwidth}
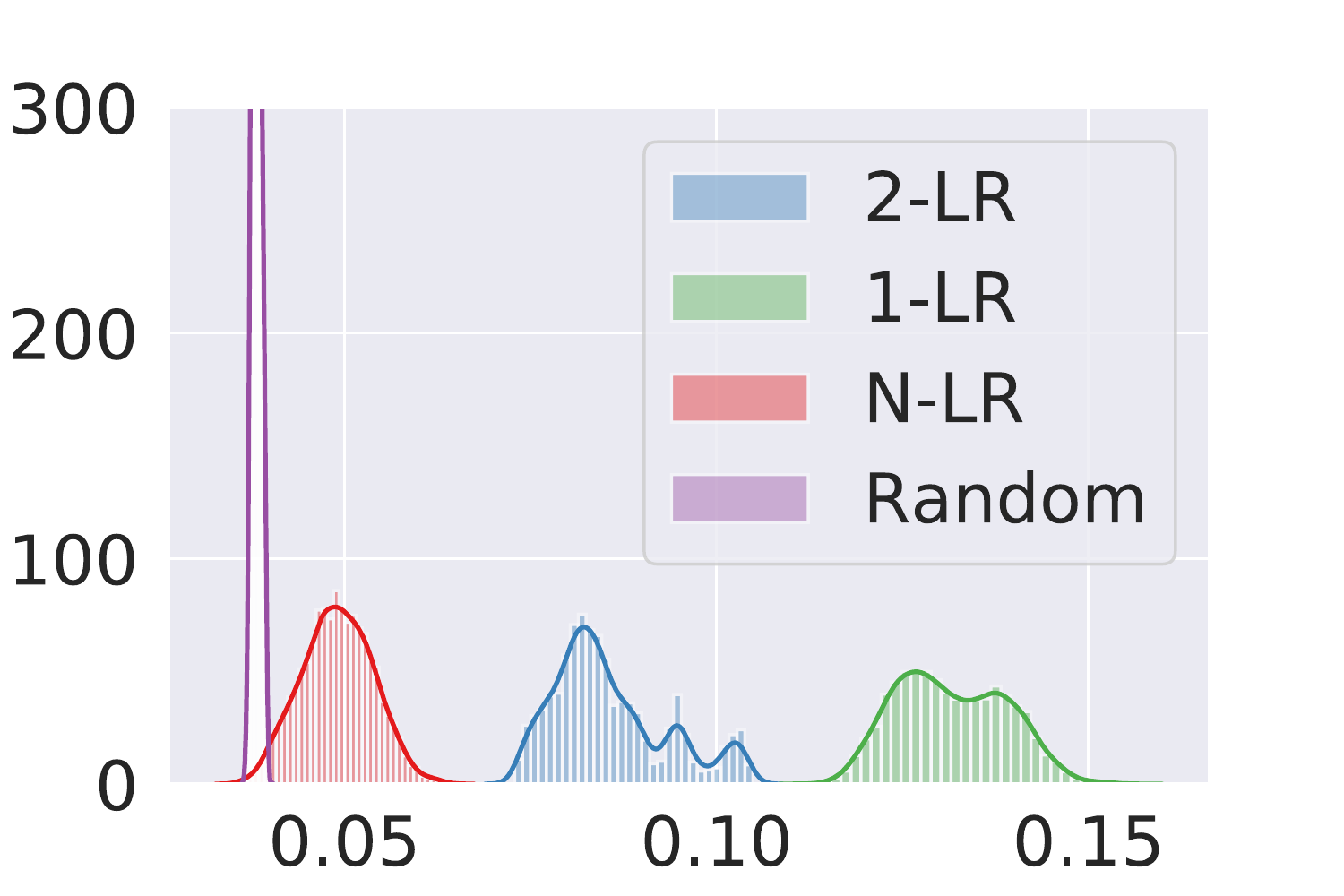
  \end{subfigure}
  \caption{Cushion of Layer 4}
  \label{fig:int_svhn_spec_lyr4_cush}
\end{figure}

\end{document}